\documentclass{article}
\usepackage[T1]{fontenc}
\usepackage[utf8]{inputenc}
\usepackage{lmodern}
\usepackage[a4paper,margin=25mm,footskip=10mm]{geometry}
\usepackage{microtype}
\usepackage{graphicx}
\usepackage{tikz}
\usetikzlibrary{arrows.meta,positioning,calc}
\usepackage{enumitem} 
\usepackage{float,subcaption} 
\usepackage{booktabs,multirow,makecell} 
\usepackage{algorithm,algpseudocode,listings}  
\usepackage{authblk} 
\usepackage{tcolorbox}
\usepackage{upgreek} 

\usepackage[tbtags]{amsmath}
\usepackage{amssymb,amsthm}
\usepackage{mathtools}
\usepackage{mathrsfs}
\usepackage{physics2}
\usephysicsmodule{ab} 

\usepackage{sota-commands}

\usepackage[round]{natbib} 
\usepackage{bm} 
\usepackage[colorlinks,linktoc=page,linkcolor=purple,citecolor=teal]{hyperref}
\usepackage[noabbrev,capitalize]{cleveref}

\allowdisplaybreaks

\newtheorem{theorem}{Theorem}
\newtheorem*{theorem*}{Theorem}
\newtheorem{result}[theorem]{Result}
\newtheorem{corollary}[theorem]{Corollary}

\newtheorem{lemma}{Lemma}[section]
\newtheorem{proposition}[lemma]{Proposition}

\theoremstyle{definition}
\newtheorem{assumption}[lemma]{Assumption}
\newtheorem{definition}[lemma]{Definition}

\crefname{assumption}{Assumption}{Assumptions}

\newcommand{\MP}{\mathrm{MP}}
\DeclareMathOperator{\ST}{ST}
\DeclareMathOperator{\dirac}{\updelta}

\allowdisplaybreaks[1]

\begin{document}

\title{Precise Dynamics of Diagonal Linear Networks:\\A Unifying Analysis by Dynamical Mean-Field Theory}

\author[1,2]{Sota Nishiyama\thanks{\texttt{snishiyama@g.ecc.u-tokyo.ac.jp}}}
\author[1,2,3]{Masaaki Imaizumi}

\affil[1]{The University of Tokyo}
\affil[2]{RIKEN Center for Advanced Intelligence Project}
\affil[3]{Kyoto University}

\maketitle

\begin{abstract}
    Diagonal linear networks (DLNs) are a tractable model that captures several nontrivial behaviors in neural network training, such as initialization-dependent solutions and incremental learning. These phenomena are typically studied in isolation, leaving the overall dynamics insufficiently understood. In this work, we present a unified analysis of various phenomena in the gradient flow dynamics of DLNs. Using Dynamical Mean-Field Theory (DMFT), we derive a low-dimensional effective process that captures the asymptotic gradient flow dynamics in high dimensions. Analyzing this effective process yields new insights into DLN dynamics, including loss convergence rates and their trade-off with generalization, and systematically reproduces many of the previously observed phenomena. These findings deepen our understanding of DLNs and demonstrate the effectiveness of the DMFT approach in analyzing high-dimensional learning dynamics of neural networks.
\end{abstract}

\tableofcontents

\section{INTRODUCTION}

The training dynamics of neural networks have attracted significant attention in deep learning theory. It has been suggested that the dynamics induced by training algorithms strongly influence the generalization performance of neural networks. This effect is captured in the idea of \emph{implicit bias} \citep{neyshabur2015search}, in which the algorithm selects a certain solution among many induced by nonconvexity of the loss and overparameterization of networks. Accordingly, many recent works have studied the interplay between models and optimizers, aiming to characterize the resulting implicit biases  \citep{neyshabur2017implicit,soudry2018implicit,arora2019implicit,bartlett2021deep}.
Moreover, understanding the convergence speed and timescales of the training dynamics contributes to efficient training of high-performance models in practice, especially in the context of modern large-scale neural networks in which the training is stopped at a compute-optimal point \citep{kaplan2020scaling}.

For a refined understanding of the dynamics, \emph{diagonal linear networks} (DLNs) have emerged as a tractable theoretical model that captures several nontrivial behaviors of learning algorithms, making them a valuable tool for studying neural network dynamics. Recent studies have uncovered various phenomena, such as the dependence of solutions on algorithmic parameters \citep{woodworth2020kernel,nacson2022implicit,pesme2021implicit,even2023sgd} and incremental learning dynamics \citep{berthier2023incremental,pesme2023saddletosaddle}.

One of the challenges in the study of DLNs is that relationships among these phenomena remain unclear. This is because existing analyses often rely on case-specific techniques. In addition, there are aspects of the dynamics that have not yet been investigated, such as convergence speed to long-term behaviors, and these unexplored elements hinder a comprehensive understanding of the overall dynamics.
Specifically, we raise the following questions:
\begin{enumerate}
    \item Which dynamical regimes and timescales arise under different initializations, and how does performance evolve in each?
    \item What solution do trained DLNs converge to, and at what rate?
\end{enumerate}

\paragraph{Contributions.}
In this work, we develop a unified framework to describe DLN dynamics and conduct a comprehensive analysis of diverse phenomena. Specifically, by leveraging Dynamical Mean-Field Theory (DMFT), which provides a precise characterization in high-dimensional limits, we derive a system of equations that characterizes the gradient flow training dynamics of DLNs in sparse regression.
By analyzing the derived equations, we elucidate the long-time behavior and timescale structure of the learning process, thereby deriving insights into the dynamics and implicit bias of DLNs.

Our main contributions are summarized as follows.
\begin{itemize}
    \item We identify distinct dynamical regimes which depend on training time and initialization scales. For large initialization, we observe a sharp transition from memorizing solutions (fit all data but generalize poorly) to generalizing solutions; for small initialization, an early \emph{search} plateau and \emph{incremental learning} that follows it.
    \item We characterize the fixed point of the gradient flow and its dependence on initialization, showing that \emph{a smaller initialization leads to better generalization}, close to minimum $\ell_1$ norm solutions. This provides an alternative derivation of the result from \citet{woodworth2020kernel}.
    \item We derive convergence rates of losses in time and show that \emph{a smaller initialization leads to slower convergence}.
    \item Together with the fixed point and convergence rate result, we establish \emph{a trade-off between optimization speed and generalization performance}.
\end{itemize}
Overall, our findings deepen the theoretical understanding of DLNs and highlight the utility of DMFT as a powerful tool for probing the dynamics of high-dimensional, nonlinear learning systems.

\begin{figure*}[t]
    \begin{center}
        \begin{subfigure}{0.45\textwidth}
            \includegraphics[width=\textwidth]{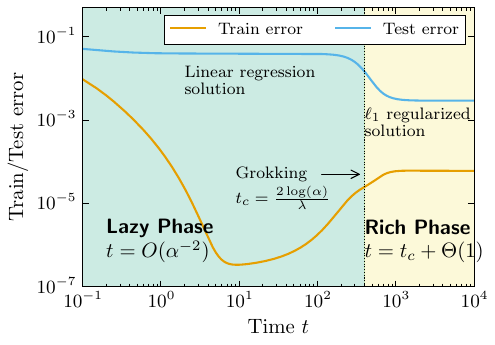}
            \caption{Large initialization ($\alpha \gg 1$).}
        \end{subfigure}
        \begin{subfigure}{0.45\textwidth}
            \includegraphics[width=\textwidth]{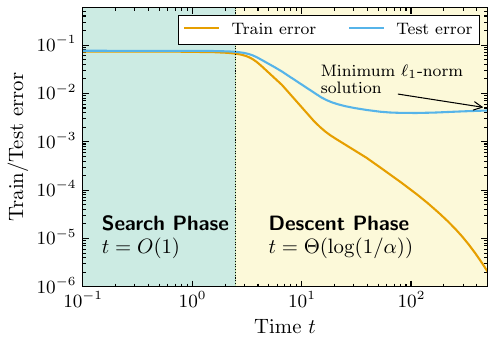}
            \caption{Small initialization ($\alpha \ll 1$).}
        \end{subfigure}
        \caption{Schematic illustrations of the timescale structures of gradient flow dynamics in DLNs.}
        \label{fig:timescales}
    \end{center}
\end{figure*}

\subsection{Related Work}
\label{sec:relatedwork}

\paragraph{Diagonal Linear Networks.}
DLNs were studied in \citet{gunasekar2018implicit} as a simple model that captures the rich implicit bias of neural networks. \citet{vaskevicius2019implicit} showed that DLNs trained with gradient descent and small initialization can implicitly perform sparse recovery. \citet{woodworth2020kernel} studied the implicit bias of gradient flow training for DLNs and uncovered a transition between the \emph{kernel regime} (large initialization) and the \emph{rich regime} (small initialization), showing that smaller initialization leads to a sparser, richer bias. \citet{moroshko2020implicit} studied similar phenomena in classification settings.
DLNs have since become a testbed to gain insight into the implicit bias of various optimization algorithms and their hyperparameter choices, including the relative scale of layers \citep{azulay2021implicit}, gradient noise in stochastic gradient descent (SGD) \citep{haochen2021shape,pesme2021implicit,even2023sgd}, step size \citep{nacson2022implicit}, early stopping time \citep{li2021implicit}, and other optimizers \citep{papazov2024leveraging,clara2025training}.

Beyond implicit bias, several works investigated the training dynamics of DLNs. \citet{berthier2023incremental} and \citet{pesme2021implicit} identified an \emph{incremental learning} or \emph{saddle-to-saddle dynamics} in DLNs with small initialization, where the parameter coordinates are sequentially activated to learn the true solution.

\paragraph{Dynamical Mean-Field Theory.}
Dynamical mean-field theory (DMFT) is a technique to reduce high-dimensional random dynamics into a low-dimensional effective process characterized by a system of integro-differential equations. Originally developed in statistical physics to analyze the Langevin dynamics of spin glasses \citep{sompolinsky1981dynamic,sompolinsky1982relaxational,crisanti1993spherical,cugliandolo1993analytical}, DMFT has been applied to a wide range of problems involving many degrees of freedom with random interactions; see \citet{cugliandolo2024recent} for a recent survey. Over the last decade, DMFT has been applied to several high-dimensional optimization and estimation problems \citep{agoritsas2018outofequilibrium,saraomannelli2020marvels,mignacco2020dynamical,bordelon2022selfconsistent,montanari2025dynamical}, with rigorous derivations established in certain settings \citep{celentano2021highdimensional,gerbelot2024rigorous,fan2025dynamical}.

The most closely related to ours is that of \citet{montanari2025dynamical}, who applied DMFT to wide two-layer networks and uncovered a timescale separation for generalization and overfitting.
Our work differs from theirs in several aspects.
Regarding the model, while two-layer DLNs analyzed in this work can be interpreted as a special case of a general two-layer neural network, they consider a narrow (compared to the input dimension) two-layer neural network with fully-connected first layer, while the DLN we consider has a diagonal first layer with width equal to the input dimension. We also consider a weight decay term not considered in their work.
Regarding the analytical focus, we analyze timescale structures of gradient flow training in \cref{sec:timescale} in a similar spirit to \citet{montanari2025dynamical}; however, the result is qualitatively different.
In addition, we go beyond the timescale analysis and analyze long-time behaviors of the dynamics in \cref{sec:longtime}.

\section{PRELIMINARIES}

\subsection{Notation}

For vectors $\bx = (x_1,...,x_d)^\transpose,\by = (y_1,...,y_d)^\transpose \in \reals^d$, $\bx \odot \by$ denotes entry-wise multiplication, i.e., $\bx \odot \by = (x_1 y_1, \dots, x_d y_d)^\transpose \in \reals^d$. For $\bx\in\reals^d$ and $L\in\naturals$, $\bx^L$ denotes entry-wise power. $\idmat_d \in \reals^{d\times d}$ denotes the $d \times d$ identity matrix. $\bone_d \in \reals^{d}$ denotes the all-ones vector $\bone_d = (1,\dots,1)^\transpose$. $\GP(0,Q)$ denotes a centered Gaussian process with covariance kernel $Q$. We denote by $\ST$ the soft thresholding function $\ST(x;\tau) \coloneqq \sign(x) \max\{\abs{x} - \tau,0\}$.

\subsection{Setup}
\label{sec:setup}

\paragraph{Data Model.}
We consider $n$ i.i.d.\ samples $(\bx_\mu,y_\mu) \in \reals^d \times \reals$ indexed by $\mu=1,\dots,n$. The input vectors $\bx_\mu$ are sampled independently from the isotropic Gaussian distribution $\normal(0,\idmat_d/d)$. Let $\bX \in \reals^{n \times d}$ be a data matrix whose $\mu$-th row is $\bx_\mu$.
The labels $y_\mu$ follow a linear model $y_\mu = \bw^{*\transpose}\bx_\mu + \xi_\mu$, where $\xi_\mu \sim \normal(0,\sigma^2)$ is some independent noise with mean $0$ and variance $\sigma^2$, and $\bw^* \in \reals^d$ is a target vector. The empirical distribution of the entries of $\bw^*$ converges to $P_*$ as $d \to \infty$. We define a label vector $\by \coloneqq (y_1,\dots,y_n)^\transpose \in\reals^n$ and a scale term $\rho^2 \coloneqq \norm{\bw^*}_2^2 / d$.

\paragraph{Diagonal Linear Network.}
We consider a two-layer diagonal linear network:
\begin{align}
f(\bx;\bu,\bv) = \bw^\transpose \bx \,, \quad \bw = \frac{1}{2}(\bu^2-\bv^2) \,, \label{eq:dln}
\end{align}
for $\bu,\bv \in \reals^d$ and $\bu^2,\bv^2$ denote entry-wise squares. This can be considered as a two-layer linear neural network with diagonal first layers $\diag(\bu),\diag(\bv)$ and second layers $\bu,\bv$. Although it represents only linear functions in $\bx$, the nonlinear reparameterization of $\bw$ induces nontrivial dynamics.

\paragraph{Training Algorithm.}
We train the DLN by minimizing the following regularized quadratic loss:
\begin{align}
L(\bu,\bv) = \frac{1}{2 n} \norm{\by - \bX \bw}_2^2 + \frac{\lambda}{2 d}(\norm{\bu}_2^2 + \norm{\bv}_2^2), \label{eq:loss}
\end{align}
where $\lambda \geq 0$ is a regularization parameter.
We consider full-batch gradient flow (continuous-time gradient descent) for time $t \geq 0$ to minimize the loss:
\begin{align}
\diff{}{t} (\bu(t), \bv(t)) & = -\frac{d}{2} (\nabla_{\bu},\nabla_{\bv}) L(\bu(t),\bv(t)) \,, \label{eq:df_dln}
\end{align}
with initial values $\bu(0)=\bv(0)=\alpha \bone_d$ for $\alpha > 0$. We denote the loss at time $t$ by $L(t) \coloneqq L(\bu(t),\bv(t))$.

\paragraph{Proportional Asymptotics.}
We analyze the proportional asymptotic regime where $n,d\to\infty$ with $n/d\to \delta \in (0,\infty)$.
Analyses in this regime have yielded significant insights into high-dimensional learning systems through exact predictions of asymptotic performance via tools from statistical physics and random matrix theory \citep{zdeborova2016statistical,mei2022generalization}.

\subsection{Overview of Main Findings}
\label{sec:overview}

\begin{figure*}[t]
    \begin{center}
        \begin{subfigure}[t]{0.32\textwidth}
            \includegraphics[width=\textwidth]{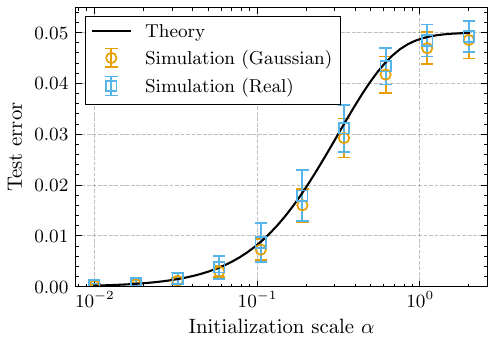}
            \caption{Test errors at fixed points.}
            \label{fig:fixedpoint}
        \end{subfigure}
        \begin{subfigure}[t]{0.32\textwidth}
            \includegraphics[width=\textwidth]{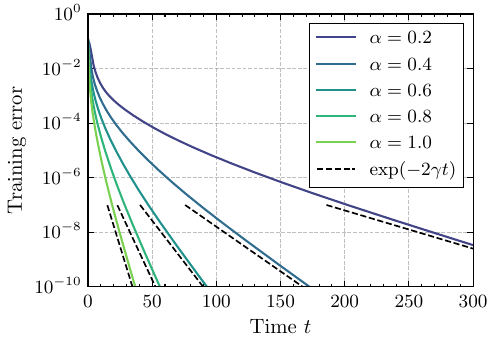}
            \caption{Convergence rates of training errors for Gaussian data.}
            \label{fig:convergencerate}
        \end{subfigure}
        \begin{subfigure}[t]{0.32\textwidth}
            \includegraphics[width=\textwidth]{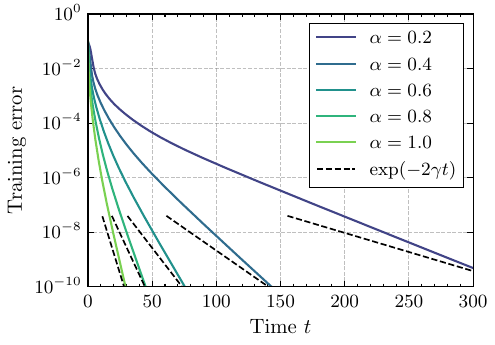}
            \caption{Convergence rates of training errors for real data.}
            \label{fig:convergencerate_real}
        \end{subfigure}
        \caption{Long-time behaviors of DLNs for $\lambda = 0$ and $\delta = 0.5$. (a): Smaller initialization $\alpha$ leads to better generalization at the fixed point (\cref{res:fixedpoint} Case (iii)). Simulations are run $10$ times on independent data, and the error bars indicate one standard deviation. (b), (c): Smaller initialization $\alpha$ leads to slower convergence (\cref{res:convergence_rates}), thus showing a trade-off with the generalization performance. The slopes of the dashed lines represent theoretical predictions for the convergence rate. Experimental details are discussed in \cref{sec:numerical}.}
        \label{fig:longtime}
    \end{center}
\end{figure*}

First, we analyze timescale structures of the dynamics to show that they exhibit \emph{qualitatively different behaviors depending on the initialization scale $\alpha$}, as depicted in \cref{fig:timescales}. We discuss the details in \cref{sec:timescale}.

\begin{description}
    \item[Large Initialization ($\alpha \gg 1$).] DLNs initially behave as approximately linear models (\emph{lazy regime}). In this phase, the loss rapidly decreases, but the model generalizes poorly. When the model is regularized ($\lambda > 0$), it then transitions to a sparse, generalizing solution (\emph{rich regime}) around the time $t \approx 2 \log(\alpha)/\lambda$ (\emph{grokking}).

    \item[Small Initialization ($\alpha \ll 1$).] We observe an early plateau with negligible change in the loss (\emph{search phase}). The loss then decreases on a timescale of $\Theta(\log(1/\alpha))$ (\emph{descent phase}) via \emph{incremental learning}, in which the model learns target coordinates one by one.
\end{description}

Second, we analyze long-time behaviors of the dynamics to identify their fixed points (long-time limit) and convergence rates (speed of convergence in time to the fixed point). Here, we focus on the unregularized ($\lambda = 0$) and overparameterized ($\delta < 1$) case, as it exhibits the most distinctive behaviors, as illustrated in \cref{fig:longtime}. Details are discussed in \cref{sec:longtime}.

\paragraph{Smaller Initialization Improves Generalization.}
We show that the fixed point of the gradient flow matches the solution of a minimum norm interpolation problem, i.e., minimizes a certain norm of $\bw$ while fitting all data. The norm depends on the initialization scale $\alpha$, with a smaller $\alpha$ enforcing a stronger bias towards sparse solutions. This implies that \emph{smaller initialization leads to better generalization} in sparse regression settings (see \cref{fig:fixedpoint}).

\paragraph{Trade-off Between Generalization and Convergence.}
We show that the loss converges exponentially as $L(t) \sim \napier^{-2\gamma t}$, with the exponent $\gamma$ monotonically increasing with initialization $\alpha$. Thus, \emph{smaller initialization leads to slower convergence} (see \cref{fig:convergencerate}) and, combined with the fixed point characterization, this reveals a \emph{trade-off between the generalization performance and the convergence speed}.

\section{DMFT ANALYSIS}

We apply Dynamical Mean-Field Theory (DMFT) to the gradient flow \eqref{eq:df_dln} with randomness coming from samples $\bx_1,\dots,\bx_n$. DMFT is a technique from statistical physics that provides a low-dimensional effective description by averaging out microscopic fluctuations, thereby capturing the macroscopic behavior of high-dimensional systems in a tractable manner.
In particular, the DMFT equation consists of stochastic process that characterize the high-dimensional dynamics of the model parameters and deterministic functions called \emph{correlation} and \emph{response} functions which encode the evolution of the macroscopic properties of the system.

In our setting, the DMFT formalism for the gradient flow \eqref{eq:df_dln} yields the following system, involving correlation and response functions $C_w,C_f,R_w,R_f\colon \reals_{\geq 0}^2 \to \reals$ and stochastic processes $w,g\colon \reals_{\geq 0} \to \reals$:
\begin{subequations} \label{eq:dmft_dln}
\begin{align}
    C_w(t,t')       & = \E[(w(t) - w^*) (w(t') - w^*)] \,, \label{eq:dmft_Cw} \\
    R_w(t,t')       & = -\E\ab[\diffp{w(t)}{z(t')}] \,, \label{eq:dmft_Rw} \\
    C_f(t,t')      & = C_w(t,t') + \sigma^2 \notag \\
    & \qquad -\int_0^{t'} R_f(t',s) (C_w(t,s) + \sigma^2) \de s \notag \\
    & \qquad - \int_0^t R_w(t,s) C_f(t',s) \de s \,, \label{eq:dmft_Cf} \\
    R_f(t,t')      & = R_w(t,t') - \int_{t'}^t R_w(t,s) R_f(s,t') \de s \,, \label{eq:dmft_Rf} \\
    g(t)           & = \frac{z(t)}{\delta} + w(t) - w^* \notag \\
    & \qquad - \int_0^t R_f(t,s) (w(s) - w^*) \de s \,, \label{eq:dmft_g} \\
    \diff{}{t}w(t) & = - \sqrt{w(t)^2 + \alpha^4 \napier^{-2\lambda t}} g(t) - \lambda w(t) \,,\label{eq:dmft_w}
\end{align}
\end{subequations}
where $z \sim \GP(0,\delta C_f)$ and $w^* \sim P_*$.

As $d \to \infty$, the empirical distribution of entries of $\bw(t)$ converges to the law of the process $w(t)$. In this way, the high-dimensional dynamics \eqref{eq:df_dln} is reduced to a scalar-valued stochastic process \eqref{eq:dmft_w}, which is more tractable and amenable to theoretical analysis.

Macroscopic quantities, such as training and test errors, can be asymptotically computed from the solution $(C_w,C_f,R_w,R_f)$ of the DMFT equation \eqref{eq:dmft_dln}:
\begin{gather}
    E_\mathrm{train}(t) \coloneqq \frac{1}{n} \sum_{\mu=1}^n (y_\mu-\bw(t)^\transpose \bx_\mu)^2 \to C_f(t,t) \,, \\
    E_\mathrm{test}(t) \coloneqq \E_{\bx,y}[(y-\bw(t)^\transpose \bx)^2] \to C_w(t,t) + \sigma^2 \,.
\end{gather}

We provide a heuristic derivation of the DMFT equation \eqref{eq:dmft_dln} based on statistical physics in \cref{app:heuristic_dmft} and present a rigorous justification in \cref{sec:rigorous}. It is also validated against numerical simulations in \cref{sec:numerical}.

\section{LEARNING TIMESCALES}
\label{sec:timescale}

\begin{figure*}
    \begin{center}
        \begin{subfigure}{0.32\textwidth}
            \includegraphics[width=\textwidth]{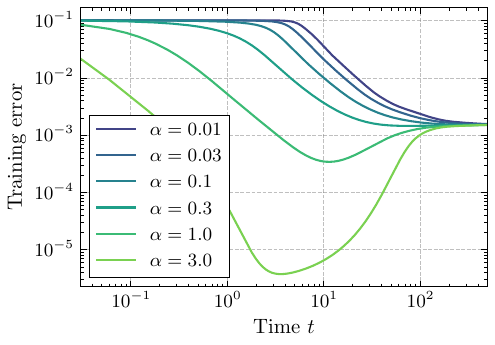}
            \caption{Training error dynamics.}
        \end{subfigure}
        \begin{subfigure}{0.32\textwidth}
            \includegraphics[width=\textwidth]{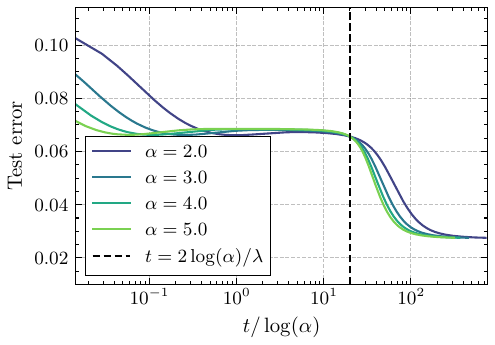}
            \caption{Grokking timescales for large $\alpha$.}
            \label{fig:timescale_grok}
        \end{subfigure}
        \begin{subfigure}{0.32\textwidth}
            \includegraphics[width=\textwidth]{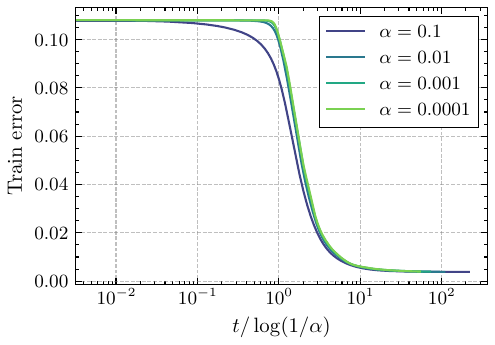}
            \caption{Descent timescales for small $\alpha$.}
            \label{fig:timescale_descent}
        \end{subfigure}
        \caption{(a): Training error dynamics for various initialization scales $\alpha$. Plots are simulations of DLNs with $d = 200$. We observe qualitatively different dynamics depicted in \cref{fig:timescales}. The monotonicity of the training error changes at around $\alpha \approx 0.3$. (b): Test error dynamics for large $\alpha$. Once the time is rescaled by $\log(\alpha)$, the transition times to the second dynamical regime collapse, showing that it is the correct scaling for the transition time. (c): Training error dynamics for small $\alpha$. Once the time is rescaled by $\log(1/\alpha)$, descent phases start and proceed on the same timescales.}
        \label{fig:timescales_detail}
    \end{center}
\end{figure*}

In this section, we analyze the timescale structure of the DMFT equation \eqref{eq:dmft_dln} for the gradient flow \eqref{eq:df_dln}. We identify qualitatively distinct behaviors unfolding across different timescales, depending on the initialization scale $\alpha$, as depicted in \cref{fig:timescales,fig:timescales_detail}.

To illustrate our technique, we analyze the simplified case of the infinite-data limit ($\delta \to \infty$), where the DMFT equation \eqref{eq:dmft_dln} reduces to the following scalar ordinary differential equation (ODE).
\begin{align}
    \diff{}{t}w(t) = - \sqrt{w(t)^2 + \alpha^4 \napier^{-2\lambda t}} (w(t) - w^*) - \lambda w(t) \,. \label{eq:dmft_inf_delta}
\end{align}
The full analysis for general $\delta$ appears in \cref{app:timescale}.

\subsection{Technique: Singular Perturbation Theory}

We analyze the dynamics \eqref{eq:dmft_inf_delta} in $\alpha \to \infty$ and $\alpha \to 0$ limits using \emph{singular perturbation theory} \citep{bender1999advanced}. It is a useful technique in the study of dynamical systems, which allows us to separate the behaviors of dynamical systems into different timescales. Following \citet{montanari2025dynamical}, we proceed heuristically: posit an ansatz on a given timescale and check consistency with the DMFT equation. We also validate against numerical simulations. Although widely used, its rigorous treatment is challenging and is beyond the scope of this paper.

\subsection{Large Initialization Limit $\alpha \to \infty$}

In this case, the dynamics exhibit two distinct dynamical regimes: \emph{lazy phase} for $t = O(\alpha^{-2})$ and \emph{rich phase} for $t = 2\log(\alpha)/\lambda + \Theta(1)$. In each phase, we first analyze the dynamics in the $\delta \to \infty$ limit, and then discuss the behavior for general $\delta$. Note that the discussion for general $\delta$ is based on the analysis in \cref{app:timescale}.

\paragraph{Lazy Phase: $t = O(\alpha^{-2})$.}
In this timescale, the factor $\sqrt{w(t)^2 + \alpha^4 \napier^{-2\lambda t}}$ in \cref{eq:dmft_inf_delta} can be approximated by $\alpha^2$ (This is because $w(t)^2 \ll \alpha^4$ as $w(0) = 0$ and $\napier^{-2\lambda t} \approx 1$ as $\lambda t \ll 1$ for small $t$). Thus, we have the following approximate ODE (with random $w^* \sim P_*$):
\begin{align}
    \diff{}{t}w(t)  & = -\alpha^2 (w(t) - w^*) \,,
\end{align}
with an explicit solution $w(t) = w^* (1 - \napier^{-\alpha^2 t})$, showing that $w(t)$ converges exponentially to the target $w^*$.

Thus, DLNs essentially behave as unregularized linear models on this timescale. This phenomenon corresponds to the \emph{lazy training} in which models with large weights behave as linearized models around their initializations \citep{jacot2018neural,chizat2019lazy}. In \cref{app:timescale_largeinit}, we show for general $\delta$ that DLNs behave as linear models and converge to linear regression solutions in time $O(\alpha^{-2})$.

\paragraph{Rich Phase: $t = 2\log(\alpha)/\lambda + \Theta(1)$.}
When $\lambda > 0$, as $t$ grows, the $\napier^{-2\lambda t}$ factor in \cref{eq:dmft_inf_delta} becomes small and eventually breaks the approximation $\sqrt{w(t)^2 + \alpha^4 \napier^{-2\lambda t}} \approx \alpha^2$. This occurs when the two terms inside the square root become of the same order, which occurs at time $t \approx t_c \coloneqq 2 \log(\alpha) / \lambda$.

After the approximation breaks down, the dynamics transitions to the next dynamical regime governed by a different equation. Shifting the time as $\tau = t - t_c$ and dropping the $\alpha^2 \napier^{-2\lambda t} = \napier^{-2\lambda \tau}$ term for $\tau \gg 1$, the dynamics \eqref{eq:dmft_inf_delta} is expressed as
\begin{align}
    \diff{}{\tau}w(\tau) & \approx - \abs{w(\tau)} (w(\tau) - w^*) - \lambda w(\tau) \,. \label{eq:rich_ode}
\end{align}
This is a logistic equation that can be solved explicitly. Setting $\Delta \coloneqq \abs{w^*} - \lambda$, the solution is given by $w(\tau) = \sign(w^*) \Delta / (1 + C \napier^{-\Delta \tau})$, where $C$ is a constant. The fixed point as $\tau \to \infty$ can be expressed using the soft-thresholding function as $w(\infty) = \max\{\Delta,0\} = \ST(w^*;\lambda)$, reminiscent of $\ell_1$ regularization. Furthermore, the convergence rate is $\abs{w(\tau) - w(\infty)} \sim \napier^{-\abs{\Delta} \tau}$, showing slower convergence for paths with target $\abs{w^*}$ closer to the threshold $\lambda$.

In summary, the dynamics transitions at time $t \approx 2\log(\alpha)/\lambda$ to the second dynamical regime, the \emph{rich} phase, which exhibits nonlinear dynamics and sparsity bias in contrast to the lazy phase. The transition timescale is validated against simulations in \cref{fig:timescale_grok}.

\emph{Connection to grokking:}
The transition to the rich phase is sharp in the sense that the dynamics after the transition converge in a time of $\Theta(1)$, faster than the transition time of $\Theta(\log \alpha)$. This sharp transition is related to \emph{grokking} \citep{power2022grokking}, a phenomenon where a model quickly transitions to a generalizing solution long after it interpolates data with poor generalization. Indeed, in the overparameterized case ($\delta < 1$), DLNs first interpolate the data with bad generalization in the lazy phase (due to the $\ell_2$ implicit bias induced by the linear dynamics \citep{bartlett2021deep}, which favors dense solutions) and then transition to a sparse, generalizing solution in the rich phase. This sudden transition is caused by different implicit biases in the lazy and rich phases, an explanation for grokking given by \citet{lyu2024dichotomy} and \citet{kumar2024grokking}.

\subsection{Small Initialization Limit $\alpha \to 0$}

As with the large initialization case, the dynamics exhibit two dynamical regimes: \emph{search phase} for $t=O(1)$ and \emph{descent phase} for $t=\Theta(\log(1/\alpha))$. These names are adopted from \citet{arous2021online}, which established analogous two-stage dynamics for online SGD learning in high-dimensional inference.

\paragraph{Search Phase: $t=O(1)$.}
We introduce the rescaled parameter $W(t) \coloneqq w(t)/\alpha^2$. Assuming that $\abs{w(t)} \ll \abs{w^*}$, the dynamics \eqref{eq:dmft_inf_delta} is approximated as
\begin{align}
    \diff{}{t}W(t) & = w^* \sqrt{W(t)^2 + \napier^{-2\lambda t}} - \lambda W(t) \,,
\end{align}
with an explicit solution
\begin{align}
    W(t) & = \frac{\sign(w^*)}{2} (1 - \napier^{-2\abs{w^*} t}) \napier^{(\abs{w^*} - \lambda)t} \,.
\end{align}
The behavior of this solution depends on the relative scales of $w^*$ and $\lambda$. Again, let $\Delta \coloneqq \abs{w^*} - \lambda$. For large $t$, the solution behaves as $\abs{W(t)} \approx (1/2) \napier^{\Delta t}$. When $\Delta < 0$, $W(t)$ converges to zero; when $\Delta > 0$, $\abs{W(t)}$ grows exponentially.

Since changes in $w(t)$ are small (of $O(\alpha^2)$), the loss does not change appreciably, and hence we observe a plateau at the beginning of training.

In this dynamical regime, the algorithm searches and identifies entries of $\bw$ to be activated and suppresses others. Specifically, entries with $\abs{w^*}$ smaller than the threshold $\lambda$ are suppressed, and those above the threshold grow. Similar dynamics hold for general $\delta$, but with a different definition of $\Delta$; see \cref{app:timescale_smallinit}.

\paragraph{Descent Phase: $t=\Theta(\log(1/\alpha))$.}
Paths with $\Delta > 0$ exhibit a transition to the second dynamical regime. This occurs when $\abs{w(t)}$ and $\abs{w^*}$ become of the same order. Equating $\abs{w(t)} = \abs{\alpha^2 W(t)} \approx (\alpha^2/2) \napier^{\Delta t}$ to $\abs{w^*}$, we obtain $t \approx t_c \coloneqq 2 \log(1/\alpha) / \Delta$ as the transition time.

Setting $\tau \coloneqq t-t_c$, the dynamics after the transition are given as follows.
\begin{align}
    \diff{}{\tau}w(\tau) & = - \abs{w(\tau)} (w(\tau) - w^*) - \lambda w(\tau) \,. \label{eq:rich_ode_2}
\end{align}
This is the same equation as \cref{eq:rich_ode} and thus behaves similarly.

An important difference from the rich phase for a large initialization is that the shifted time $\tau$ is defined differently. In the large initialization case, the transition time $t_c = 2\log(\alpha)/\lambda$ is common for all paths, and the dynamics \eqref{eq:rich_ode} proceed on the same timescale for all paths. In contrast, in the small initialization case, the transition time $t_c = 2\log(1/\alpha)/\Delta$ is different for each path. Since the dynamics \eqref{eq:rich_ode_2} proceed in time $O(1)$, which is much faster than the transition timescale of $\Theta(\log(1/\alpha))$, activated paths (paths that have transitioned to the descent phase) converge quickly to their fixed points. Thus, training in this descent phase proceeds via \emph{incremental learning}, successive activations of entries of $\bw(t)$ \citep{berthier2023incremental,pesme2023saddletosaddle}. The timescale of $\Theta(\log(1/\alpha))$ in this regime is checked against numerical simulations in \cref{fig:timescale_descent}.

\section{LONG-TIME BEHAVIOR}
\label{sec:longtime}

We analyze long-time behaviors of the DMFT equation \eqref{eq:dmft_dln} and establish a trade-off between generalization performance and optimization speed. Our results are summarized in \cref{tab:longtime}.

\begin{table}[h]
\caption{Summary of long-time behaviors under different regularization $\lambda$ and the aspect ratio $\delta$. When $\lambda = 0$ and $\delta < 1$, decreasing $\alpha$ results in a trade-off: the generalization performance of the fixed point improves (lower test error), but convergence to the fixed point slows down.}
\label{tab:longtime}
\begin{tabular}{cc|cc}
\toprule
$\lambda$                         & $\delta$ & \textbf{\makecell{Fixed\\Point}}                                & \textbf{\makecell{Convergence\\Rate}}  \\ \midrule
$>0$                              & Any      & $\ell_1$-regularized                                            & Sub-Exponential                        \\ \midrule
\multirowcell{2}{\makecell{$=0$}} & $>1$     & Ridgeless                                                       & \multirowcell{2}{\makecell{Exponential\\(\textit{slower} in $\alpha \searrow$)}}  \\ \cmidrule{2-3}
                                  & $<1$     & \makecell{Minimum Norm\\(\textit{better} in $\alpha \searrow$)} &   \\
\bottomrule
\end{tabular}
\end{table}

\subsection{Fixed Points and the Benefit of Small Initialization}

We analyze the solution obtained by the gradient flow \eqref{eq:df_dln} through a fixed point analysis of the DMFT equation \eqref{eq:dmft_dln}. In the following result, we show that the solution can be characterized as a minimizer of a certain estimation problem. Note that our result is stated as a \emph{result} and not as a \emph{theorem}, due to the non-rigorous derivation of the DMFT equation \eqref{eq:dmft_dln} and derivation of the fixed point.

\begin{result}[Fixed point of gradient flow] \label{res:fixedpoint}
    Let $\bw(\infty) \in \reals^d$ be the fixed point of the gradient flow. Let $\hat\bw \in \reals^d$ be the solution of a minimization problem as follows.
    \begin{description}
        \item[Case (i): $\lambda > 0$.] $\ell_1$-regularized linear regression:
        \begin{align}
            \hat\bw = \argmin_{\bw \in \reals^d} \frac{1}{2n} \norm{\by - \bX\bw}_2^2 + \frac{\lambda}{d} \norm{\bw}_1 \,.
        \end{align}

        \item[Case (ii): $\lambda = 0$, $\delta > 1$.] Ridgeless linear regression:
        \begin{align}
            \hat \bw = \argmin_{\bw \in \reals^d} \frac{1}{2n} \norm{\by - \bX\bw}_2^2 \,.
        \end{align}

        \item[Case (iii): $\lambda = 0$, $\delta < 1$.] Minimum norm interpolation:
        \begin{align}
            \hat \bw = \argmin_{\bw \in \reals^d} J_\alpha(\bw) \; \text{subject to} \; \by = \bX \bw \,,
        \end{align}
        with a norm $J_\alpha(\bw) = \alpha^2 \sum_{i=1}^d J(w_i/\alpha^2)$ with $J(x) = x\sinh^{-1}(x) - \sqrt{1 + x^2} + 1$.
    \end{description}
    As $d \to \infty$, the joint empirical distributions of the entries of $(\bw(\infty), \bw^*)$ and of $(\hat\bw,\bw^*)$ approach the same limiting distribution characterized by the fixed point of the DMFT equation \eqref{eq:dmft_dln} shown in \cref{app:fixed_point}.
\end{result}

Case (i) is intuitive: $\ell_2$ regularization on $(\bu,\bv)$ translates into $\ell_1$ regularization on $\bw = (\bu^2 - \bv^2)/2$. Case (ii) is also natural, since in the underparameterized case ($\delta > 1$), there exists a unique minimizer of the loss \eqref{eq:loss} almost surely as $d \to \infty$ with the minimizer given by the ridgeless solution.

In Case (iii), there are multiple minimizers of the loss because of overparameterization ($\delta < 1$), and the implicit bias of the algorithm plays a role in selecting a solution among them. \Cref{res:fixedpoint} indicates that the gradient flow selects the solution that minimizes a norm $J_\alpha$ dependent on the initialization $\alpha$. Properties of this norm are discussed in detail in \citet{woodworth2020kernel}. As $\alpha \to \infty$, $J_\alpha$ approximately behaves as the $\ell_2$ norm, resulting in the same implicit bias as linear models. As $\alpha \to 0$, $J_\alpha$ is approximately proportional to the $\ell_1$ norm, which exhibits a stronger bias toward sparse solutions. Thus, in the case of a sparse target, smaller initialization yields better final performance, as illustrated in \cref{fig:fixedpoint}.

Our result for Case (iii) is derived under a more restricted setting than \citet[Theorem 1]{woodworth2020kernel}, which holds for any dimension $d$ and any data distributions, yet provides several advantages. First, our result allows for a precise prediction of performances in high dimensions as a solution to a system of equations. Second, our alternative derivation based on DMFT enhances our toolkit for studying implicit biases and has the potential to tackle problems that their method does not apply to.

\paragraph{Sketch of Derivation.}
Our derivation of \cref{res:fixedpoint} proceeds as follows; see \cref{app:fixed_point} for details.
\begin{enumerate}
    \item We obtain a system of equations that the fixed point of the DMFT equation \eqref{eq:dmft_dln} satisfies to derive the limiting distribution of the entries of $\bw(\infty)$.
    \item We obtain a characterization of the solutions $\hat \bw$ of the minimization problems given in \cref{res:fixedpoint} in the high-dimensional limit using \emph{approximate message passing} (AMP) \citep{donoho2009messagepassing,feng2022unifying}, and show that these two characterizations match in each case.
\end{enumerate}

\subsection{Convergence Rates and Their Trade-off with Generalization}

Next, we analyze convergence rates of the loss $L(t)$.

\begin{result}[Average-case convergence rate of gradient flow] \label{res:convergence_rates}
    The paths $w(t)$ converge exponentially with different rates for each path.

\begin{description}
    \item[Regularized ($\lambda > 0$).] There are paths with arbitrarily slow rates, and the convergence of the loss $L(t)$ is subexponential.
    \item[Unregularized ($\lambda = 0$).] The loss $L(t)$ converges exponentially as $\abs{L(t) - L(\infty)} = \exp(-2 \gamma t + o(1))$, where the exponent $\gamma > 0$ depends on $\alpha$, $\delta$, $\sigma^2$, and $P_*$ and can be computed as a solution of a nonlinear equation \eqref{eq:convergence_rate_eqn}. Furthermore, $\gamma$ is monotonically increasing with respect to $\alpha$.
\end{description}
\end{result}

To the best of our knowledge, this provides the first theoretical characterization of the average-case convergence rate of gradient flow for DLNs in high dimensions and its monotonicity with the initialization scale. \cref{res:convergence_rates} indicates that the convergence is slower for a smaller initialization $\alpha$, as shown in \cref{fig:convergencerate}. Together with \cref{res:fixedpoint} Case (iii), it establishes a \emph{trade-off between generalization performance and the convergence speed}.
Note that the trade-off is only meaningful in overparameterized settings where multiple solutions exist.
Although previous works already discuss that small initialization implies initialization near a saddle point of the loss and hence leads to slow escape from the initial saddle \citep{woodworth2020kernel}, our result is concerned with the long-time behavior and shows that the slow dynamics persists in the entire dynamics with a quantitative characterization of the convergence rate.

We note that a similar phenomenon is observed in a different setting by \citet{pesme2021implicit}, who studied SGD dynamics of DLNs and found that slower training leads to sparser solutions. This suggests a general principle in DLNs: \emph{better solutions are harder to find}.

With non-zero regularization $\lambda > 0$, the convergence is subexponential, and we do not show a monotonicity result with respect to the initialization scale $\alpha$.
However, for small but nonzero regularization $\lambda > 0$, the transient dynamics still resemble the unregularized case $\lambda = 0$ for a significant period of time.
Since regularization affects the dynamics only after time $t \sim \lambda^{-1}$, for $t \ll \lambda^{-1}$, the dynamics behave similar to the unregularized model, and hence the smaller initialization still leads to a slower dynamics (until time $t \sim \lambda^{-1}$).

\paragraph{Sketch of Derivation.}
We derive \cref{res:convergence_rates} as follows; see \cref{app:convergence_rate} for details.
\begin{enumerate}
    \item We linearize the DMFT equation \eqref{eq:dmft_dln} around the fixed point.
    \item We employ the \emph{Laplace transform} to analyze the linearized dynamics and find singularities of the Laplace transforms to derive the convergence rate.
\end{enumerate}

\section{RIGOROUS THEORY}
\label{sec:rigorous}

While the DMFT equation \eqref{eq:dmft_dln} is derived heuristically, we can rigorously justify it for a closely related model: \emph{truncated diagonal linear networks}. We define truncated (two-layer) DLNs as follows:
\begin{align}
 f(\bx;\bu,\bv) = \bw^\transpose \bx \,, \; \bw = \frac{1}{2}(\eta_M(\bu^2)-\eta_M(\bv^2)) \,,
\end{align}
where $\bu,\bv\in\reals^d$ and $\eta_M\colon \reals \to \reals$ is a smooth Lipschitz function with $\eta_M(x) = x$ for $\abs{x} \leq M$ and $\eta_M(x) = 0$ for $\abs{x} \geq M+1$, for $M > 0$. This entry-wise truncation ensures that $f_M$ and its gradients are Lipschitz continuous in $\bu$ and $\bv$, which makes the model more amenable to rigorous treatment. When the entries of $\bu$ and $\bv$ remain within $[-M,M]$, $\eta_M$ acts as the identity; hence, for large enough $M$, the truncated DLN closely approximates the original model \eqref{eq:dln}.

We consider gradient flow training of truncated DLNs as described in \cref{sec:setup}. To characterize its behavior, we extend the theory of \citet{celentano2021highdimensional}, which rigorously establishes DMFT characterization of gradient flow for a class of models that includes generalized linear models and narrow two-layer neural networks, but not truncated DLNs. We show that the empirical distribution of the entries of $\bw(t)$ for truncated DLNs is asymptotically equivalent to the distribution of the unique solution of a DMFT equation.
\begin{theorem*}[Informal version of \cref{cor:dmft_dln}]
    Assume that the entries $\bX=(x_{ij})_{i\in[n],j\in[d]}$ are independent and satisfy $\E x_{ij}=0,\E x_{ij}^2=1/d,\norm{x_{ij}}_{\psi_2}\leq C/\sqrt{d}$, where $\norm{\cdot}_{\psi_2}$ is the sub-Gaussian norm and $C>0$ is a constant. For any $T > 0$, there exists a unique solution $(w(t))_{t=0}^T$ of the DMFT equation \eqref{eq:dmft_truncated_dln}, and we have
    \begin{align}
        \frac{1}{d}\sum_{i=1}^d \dirac_{(w_i(t))_{t=0}^T,w^*_i} \xrightarrow{W_2} \sfP((w(t))_{t=0}^T,w^*) \,,
    \end{align}
    almost surely as $n,d \to \infty$, where $\xrightarrow{W_2}$ denotes convergence in the Wasserstein-$2$ distance and $\sfP((w(t))_{t=0}^T,w^*)$ denotes the joint law of the process $(w(t))_{t=0}^T$ and the random variable $w^*$.
\end{theorem*}

For full statements and proofs, see \cref{app:rigorous}.

Note that the distribution of entries of $\bX$ is not restricted to the Gaussian distribution, and our result is \emph{universal} with respect to the input distribution.

\section{NUMERICAL EXPERIMENTS}
\label{sec:numerical}

The code to reproduce the numerical experiments is available at \url{https://github.com/sotanishy/dmft-dln}.

\paragraph{Simulations with Gaussian Data.}
To validate our theoretical results, we have conducted numerical simulations with Gaussian data, with $\bx_\mu$ sampled from the Gaussian distribution $\normal(0,\idmat_d/d)$ with $d=500$. We take the target distribution $P_*$ to be a Bernoulli distribution with $\prob\{w^*=1\}=0.1$. Gradient flow is discretized with step size $\eta=0.1$. To obtain the fixed point, we run the dynamics up to $t=1000$. Results in \cref{fig:longtime,fig:timescales_detail} show remarkable agreement with theoretical predictions.

\paragraph{Simulations with Real-World Data.}
To check the universality of our results, we have tested on real-world data. We use a gene expression dataset \citep{ellrott2013tcga,fiorini2016gene} and take a random subset of $n=100$ samples and $d = 200$ features ($\delta=0.5$). Results for fixed points are shown in \cref{fig:fixedpoint} and show excellent agreement with the theoretical prediction. Results on convergence rates are shown in \cref{fig:convergencerate_real}. Although they deviate from the theoretical prediction (see discussions in \cref{app:numerical}), qualitative aspects, in particular the monotonicity of the convergence rates with respect to initialization, are well captured by our theory.

Further experiments and discussions are available in \cref{app:numerical}.

\section{DISCUSSION}

It would be interesting to explore whether the trade-off between generalization and optimization, specifically the idea that \emph{`better solutions are harder to find'}, which we discussed in \cref{sec:longtime}, can be established and extended to general neural networks.

In addition, extending our DMFT analysis to other architectures (such as deep linear networks, nonlinear networks, and transformers) and algorithms (such as SGD) would be a promising future direction. DMFT can handle a wide range of complex architectures in a common formalism \citep{bordelon2022selfconsistent,bordelon2024infinite,bordelon2025deep,montanari2025dynamical}, and transferring the insights developed in this work to these architectures would be a fruitful avenue. Furthermore, prior studies suggest that different optimizers induce distinct implicit biases compared to gradient flow (see \cref{sec:relatedwork}), and a deeper theoretical understanding is of great interest.

\subsubsection*{Acknowledgements}
Sota Nishiyama was supported by WINGS-FMSP at the University of Tokyo. Masaaki Imaizumi was supported by JSPS KAKENHI (24K02904), JST FOREST (JPMJFR216I), and JST BOOST (JPMJBY24A9).

\bibliography{dmft-dln}

\newpage
\appendix
\renewcommand{\thesection}{\Alph{section}}
\numberwithin{equation}{section}

\section{HEURISTIC DERIVATION OF THE DMFT EQUATION}
\label{app:heuristic_dmft}

\subsection{Derivation of the DMFT Equation Using Path Integrals}

We heuristically derive the DMFT equation \eqref{eq:dmft_dln} using the \emph{path integral} approach in statistical physics. Specifically, we base our derivation on the Martin--Siggia--Rose--De Dominicis--Janssen (MSRDJ) formalism \citep{martin1973statistical}. The derivation proceeds by expressing the dynamics in a path integral form and using the saddle-point method in the $d\to\infty$ limit to obtain self-consistent equations. Similar computations can be found, for example, in \citet{agoritsas2018outofequilibrium,saraomannelli2020marvels,mignacco2020dynamical,bordelon2022selfconsistent,montanari2025dynamical}.

The gradient flow dynamics \eqref{eq:df_dln} for our setup is
\begin{align}
    \diff{}{t} \bu(t) & = -\frac{1}{2} \ab(\bu(t) \odot \frac{1}{\delta} \bX^\transpose (\bX(\bw(t) - \bw^*) - \bxi) + \lambda \bu(t)) \,,   \\
    \diff{}{t} \bv(t) & = -\frac{1}{2} \ab(- \bv(t) \odot \frac{1}{\delta} \bX^\transpose (\bX(\bw(t) - \bw^*) - \bxi) + \lambda \bv(t)) \,,
\end{align}
where $\bxi\coloneqq(\xi_1,\dots,\xi_n)^\transpose \in\reals^n$ is a noise vector.

Defining fields $\bff(t) \in \reals^n$ and $\bg(t) \in \reals^d$ for $t \geq 0$ as
\begin{align}
    \bff(t) \coloneqq \bX (\bw(t) - \bw^*) \,, \quad \bg(t) \coloneqq \frac{1}{\delta} \bX^\transpose (\bff(t) - \bxi) \,,
\end{align}
the dynamics can be expressed as
\begin{align}
    \diff{}{t} \bu(t) = - \frac{1}{2} (\bu(t) \odot\bg(t) + \lambda \bu(t)) \,, \quad \diff{}{t} \bv(t) = - \frac{1}{2} (-\bv(t) \odot \bg(t) + \lambda \bv(t)) \,. \label{eq:gf_fields}
\end{align}

We define a dynamical partition function $Z$ as
\begin{align}
    Z & \coloneqq \int \De \bff \De \bg \De \bu \De \bv \De \bw \, \dirac(\bff(t) - \bX (\bw(t) - \bw^*)) \dirac(\delta \bg(t) - \bX^\transpose (\bff(t) - \bxi)) \,,
\end{align}
where $\dirac$ is the Dirac delta function (an upright $\dirac$ is used to distinguish it from $\delta$), and the path measures $\De\bff,\De\bg,\De\bu,\De\bv,\De\bw$ are implicitly defined with constraints \eqref{eq:gf_fields} and $\bw(t) = (\bu(t)^2-\bv(t)^2)/2$.

Using the Fourier representation of the delta function $\dirac(x) = \int_{-\im\infty}^{\im\infty} \de \hat x/ (2\pi \im) \, \napier^{\hat x x}$ (here and in the following, $\hat{\cdot}$ denotes conjugate variables), we compute the dataset-averaged partition function $\E Z$ as
\begin{align}
    \E Z & \propto \E \int \De \bff \De\hat \bff \De \bg \De\hat \bg \De\bu \De\bv \De\bw \, \exp\ab(\int \de t \, \hat \bff(t)^\transpose (\bff(t) - \bX(\bw(t) - \bw^*)) + \int \de t \, \hat \bg(t)^\transpose (\delta \bg(t) - \bX^\transpose (\bff(t) - \bxi))) \notag \\
         & = \E_{\bxi} \int \De \bff \De\hat \bff \De \bg \De\hat \bg \De\bu \De\bv \De\bw \, \exp\ab(\int \de t \, \hat \bff(t)^\transpose \bff(t) + \delta \int \de t \, \hat \bg(t)^\transpose \bg(t)) \E_{\bX} \napier^A \,, \label{eq:EZ}
\end{align}
where $\E_{\bX},\E_{\bxi}$ denote expectations over $\bX$, $\bxi$, respectively, and
\begin{align}
    A & \coloneqq - \int \de t \, \hat \bff(t)^\transpose \bX (\bw(t) - \bw^*) - \int \de t \, (\bff(t) - \bxi)^\transpose \bX \hat \bg(t) \,.
\end{align}

We calculate the expectation $\E_{\bX} \napier^A$ using a Gaussian integration over $\bX$ as follows.
\begin{align}
    \log \E_{\bX} \napier^A & = \log \int \de \bX \, \exp\ab(-\frac{d}{2} \tr(\bX^\transpose \bX) - \int \de t \, \hat \bff(t)^\transpose \bX (\bw(t) - \bw^*) - \int \de t \, (\bff(t) - \bxi)^\transpose \bX \hat \bg(t)) + \const \notag                                         \\
                            & = \frac{1}{2d} \int \de t \de t' \, \hat \bff(t)^\transpose \hat \bff(t') (\bw(t) - \bw^*)^\transpose (\bw(t') - \bw^*) + \frac{1}{2 d} \int \de t \de t' \, \hat \bg(t)^\transpose \hat \bg(t') (\bff(t) - \bxi)^\transpose (\bff(t') - \bxi) \notag \\
                            & \qquad + \frac{1}{d} \int \de t \de t' \, (\bff(t) - \bxi)^\transpose \hat \bff(t') \hat \bg(t)^\transpose (\bw(t') - \bw^*) + \const \,.
\end{align}
We introduce order parameters as follows.
\begin{alignat}{2}
    C_w(t,t') & \coloneqq \frac{1}{d} (\bw(t) - \bw^*)^\transpose (\bw(t') - \bw^*) \,,        & \quad R_w(t,t') & \coloneqq \frac{1}{d} (\bw(t) - \bw^*)^\transpose \hat \bg(t') \,, \\
    C_f(t,t') & \coloneqq \frac{1}{\delta d} (\bff(t) - \bxi)^\transpose (\bff(t') - \bxi) \,, & \quad
    R_f(t,t') & \coloneqq \frac{1}{\delta d} (\bff(t) -  \bxi)^\transpose \hat \bff(t') \,.
\end{alignat}
Using these order parameters, we have
\begin{align}
    \log \E_{\bX} \napier^A & = \frac{1}{2} \int \de t \de t' \, C_w(t,t') \hat \bff(t)^\transpose \hat \bff(t') + \frac{\delta}{2} \int \de t \de t' \, C_f(t,t') \hat \bg(t)^\transpose \hat \bg(t') + \delta d \int \de t \de t' \, R_f(t,t')R_w(t',t) + \const \,.
\end{align}
Inserting the definitions of the order parameters into \cref{eq:EZ} using the delta function as $\dirac(d C_w(t,t') - (\bw(t) - \bw^*)^\transpose (\bw(t') - \bw^*))$ and using the Fourier representation of the delta function, we get
\begin{align}
    \E Z \propto \int \De C_w \De\hat C_w \De C_f \De\hat C_f \De R_w \De\hat R_w \De R_f \De\hat R_f \, \napier^{d\Phi} \,, \label{eq:EZ_averaged}
\end{align}
where the action $\Phi$ is defined as
\begin{align}
    \Phi   & \coloneqq - \int \de t \de t' \, \ab(\hat C_w(t,t') C_w(t,t') + \delta \hat C_f(t,t') C_f(t,t') + \hat R_w(t,t') R_w(t,t') + \delta \hat R_f(t,t') R_f(t,t') - \delta R_f(t,t') R_w(t',t)) \notag \\
           & \qquad + \log Z_w + \delta \log Z_f \,,                                                                                                                                                           \\
    Z_w    & \coloneqq \int \De g \De \hat g \De u \De v \De w \, \napier^{\Phi_w} \,,                                                                                                                         \\
    Z_f    & \coloneqq \int \De f \De \hat f \, \napier^{\Phi_f} \,,                                                                                                                                           \\
    \Phi_w & \coloneqq \delta \int \de t \, \hat g(t) g(t) + \frac{\delta}{2} \int \de t \de t' \, C_f(t,t') \hat g(t) \hat g(t') \notag                                                                       \\
           & \qquad + \int \de t \de t' \, \ab(\hat C_w(t,t') (w(t) - w^*) (w(t') - w^*) + \hat R_w(t,t') (w(t) - w^*) \hat g(t')) \,,                                                                         \\
    \Phi_f & \coloneqq \int \de t \, \hat f(t) f(t) + \frac{1}{2} \int \de t \de t' \, C_w(t,t') \hat f(t) \hat f(t') \notag                                                                                   \\
           & \qquad + \int \de t \de t' \, \ab(\hat C_f(t,t') (f(t)-\xi)(f(t')-\xi) + \hat R_f(t,t') (f(t) - \xi) \hat f(t')) \,.
\end{align}
Here, the paths $f,g,u,v,w$ are now one-dimensional, and the dimensionality of the problem has been effectively reduced from $d$ to one.

In the $d \to \infty$ limit, we evaluate the integral \eqref{eq:EZ_averaged} using the saddle-point method. In the following, $\E$ denotes an expectation over measures $\napier^{\Phi_w} / Z_w$ and $\napier^{\Phi_f} / Z_f$. Taking derivatives of the action $\Phi$ with respect to the order parameters and setting them to zero, we get
\begin{alignat}{2}
    \diffp{\Phi}{\hat C_w(t,t')} & = -C_w(t,t') + \E[(w(t) - w^*) (w(t') - w^*)] = 0 \,,              & \quad \diffp{\Phi}{\hat R_w(t,t')} & = -R_w(t,t') + \E[(w(t) - w^*) \hat g(t')] = 0 \,,              \\
    \diffp{\Phi}{\hat C_f(t,t')} & = -\delta C_f(t,t') + \delta \E[(f(t) - \xi)(f(t') - \xi)] = 0 \,, & \quad \diffp{\Phi}{\hat R_f(t,t')} & = -\delta R_f(t,t') + \delta \E[(f(t) - \xi)\hat f(t')] = 0 \,,
\end{alignat}
and
\begin{alignat}{2}
    \diffp{\Phi}{C_w(t,t')} & = -\hat C_w(t,t') + \frac{\delta}{2} \E[ \hat f(t) \hat f(t')] = 0 \,,        & \quad \diffp{\Phi}{R_w(t,t')} & = -\hat R_w(t,t') + \delta R_f(t',t) = 0 \,,        \\
    \diffp{\Phi}{C_f(t,t')} & = -\delta \hat C_f(t,t') + \frac{\delta}{2} \E[ \hat g(t) \hat g(t')] = 0 \,, & \quad \diffp{\Phi}{R_f(t,t')} & = -\delta \hat R_f(t,t') + \delta R_w(t',t) = 0 \,.
\end{alignat}
We obtain
\begin{align}
    \hat R_w(t,t') = \delta R_f(t',t) \,, \quad \hat R_f(t,t') = R_w(t',t) \,.
\end{align}
Moreover, we can show the causality of the response functions, i.e.,
\begin{align}
    R_w(t',t) = R_f(t',t) = 0
\end{align}
for $t' > t$.

To obtain effective dynamics, we rewrite the effective path measures $\napier^{\Phi_w}$ and $\napier^{\Phi_f}$. We have
\begin{align}
    \int \De \hat g \, \napier^{\Phi_w} & = \int \De \hat g \, \exp\ab(\delta \int \de t \, \hat g(t) g(t) + \frac{\delta}{2} \int \de t \de t' \, C_f(t,t') \hat g(t) \hat g(t') + \delta \int \de t \de t' \, R_f(t,t') \hat g(t) (w(t') - w^*)) \notag \\
                                        & \qquad \times \exp\ab(\int \de t \de t' \, \hat C_w(t,t') (w(t) - w^*) (w(t') - w^*)) \notag                                                                                                                    \\
                                        & \propto \E_{z_g \sim \GP(0,\delta C_f)} \int \De \hat g \, \exp\ab(\delta \int \de t \, \hat g(t) \ab(g(t) - \frac{z_g(t)}{\delta} + \int \de t' \, R_f(t,t') (w(t') - w^*))) \notag                            \\
                                        & \qquad \times \exp\ab(\int \de t \de t' \, \hat C_w(t,t') (w(t) - w^*) (w(t') - w^*)) \notag                                                                                                                    \\
                                        & \propto \E_{z_g \sim \GP(0,\delta C_f)} \dirac\ab(g(t) - \frac{z_g(t)}{\delta} + \int_0^t \de t' \, R_f(t,t') (w(t') - w^*)) \notag                                                                             \\
                                        & \qquad \times \exp\ab(\int \de t \de t' \, \hat C_w(t,t') (w(t) - w^*) (w(t') - w^*)) \,,
\end{align}
where in the second line we used the \emph{Hubbard--Stratonovich transformation}
\begin{align}
    \exp\ab(\frac{1}{2} \int \de t \de t' A(t,t') x(t) x(t')) & \propto \int \De z \,  \exp\ab(-\frac{1}{2} \int \de t \de t' A^{-1}(t,t') z(t) z(t') - \int \de t \, x(t) z(t)) \\
                                                              & \propto \E_{z \sim \GP(0,A)} \exp\ab(- \int \de t \, x(t) z(t)) \,.
\end{align}
This result indicates that the effective path $g(t)$ satisfies the following equation.
\begin{align}
    g(t) = \frac{z_g(t)}{\delta} - \int_0^t \de t' \, R_f(t,t') (w(t') - w^*) \,. \label{eq:g_noGamma}
\end{align}

Similarly, for $\napier^{\Phi_f}$, we get
\begin{align}
    \int \De \hat f \, \napier^{\Phi_f} & \propto \E_{z_f \sim \GP(0,C_w)} \dirac\ab(f(t) - z_f(t) + \int_0^t \de t' \, R_w(t,t') (f(t') - \xi)) \notag \\
                                        & \qquad \times \exp\ab(\int \de t \de t' \, \hat C_f(t,t') (f(t)-\xi)(f(t')-\xi)) \,,
\end{align}
and thus the effective path $f(t)$ satisfies
\begin{align}
    f(t) = z_f(t) - \int_0^t \de t' \, R_w(t,t') (f(t') - \xi) \,.
\end{align}

Finally, we compute $R_w(t,t')$. We insert a source term into our effective action as
\begin{align}
    Z_w[J] \coloneqq \int \De g \De \hat g \De u \De v \De w \, \napier^{\Phi_w[J]} \,, \quad \Phi_w[J] \coloneqq \Phi_w + \int_0^t \de t \, J(t) \hat g(t) \,,
\end{align}
and compute $R_w(t,t')$ as
\begin{align}
    R_w(t,t') = \E[(w(t) - w^*) \hat g(t')] = \lim_{J \to 0} \diffp{}{J(t')} \E_J\ab[w(t) - w^*] \,,
\end{align}
where $\E_J$ denotes an expectation over the measure $\napier^{\Phi_w[J]}/Z_w[J]$. Then, the effective path corresponding to this measure is
\begin{align}
    g(t) = \frac{z_g(t) - J(t)}{\delta} - \int_0^t \de t' \, R_f(t,t') (w(t') - w^*) \,.
\end{align}
Therefore, we have
\begin{align}
    R_w(t,t') = \lim_{J \to 0} \E_J\ab[\diffp{(w(t) - w^*)}{J(t')}] = -\E\ab[\diffp{w(t)}{z_g(t')}] \,.
\end{align}

Similarly, we have
\begin{align}
    R_f(t,t') = \E[(f(t) - \xi)\hat f(t')] = -\E\ab[\diffp{f(t)}{z_f(t')}] \,.
\end{align}
$R_f(t,t')$ consists of a continuous bulk $R_f(t,t')$ for $t > t'$ and a delta spike at $t = t'$ with value $-1$.
Separating these two contributions, \cref{eq:g_noGamma} is written as
\begin{align}
    g(t) = \frac{z_g(t)}{\delta} + w(t) - w^* - \int_0^t \de t' \, R_f(t,t') (w(t') - w^*) \,,
\end{align}
where we abused the notation and used $R_f(t,t')$ for only the bulk contribution.

Collecting the above expressions, we obtain the following.
\begin{subequations}
    \begin{alignat}{2}
        C_w(t,t') & = \E[(w(t) - w^*) (w(t') - w^*)] \,, & \quad C_f(t,t') & = \E[(f(t) - \xi)(f(t') - \xi)]  \,, \\
        R_w(t,t') & = -\E\ab[\diffp{w(t)}{z_g(t')}] \,,  & \quad R_f(t,t') & = -\E\ab[\diffp{f(t)}{z_f(t')}] \,,
    \end{alignat}
    \vspace{-\baselineskip}
    \begin{align}
        f(t) & = z_f(t) - \int_0^t R_w(t,s)(f(s) - \xi) \de s \,,                              &  & z_f \sim \GP\ab(0,C_w) \,, \label{eq:f}     \\
        g(t) & = \frac{z_g(t)}{\delta} + w(t) - w^* - \int_0^t R_f(t,s) (w(s) - w^*) \de s \,, &  & z_g \sim \GP(0,\delta C_f) \,, \label{eq:g}
    \end{align}
    \vspace{-\baselineskip}
    \begin{gather}
        \diff{}{t}u(t) = - \frac{1}{2} u(t) (g(t) + \lambda) \,, \quad \diff{}{t}v(t) = - \frac{1}{2} v(t) (-g(t) + \lambda) \,, \quad w(t) = \frac{1}{2}(u(t)^2 - v(t)^2) \,.
    \end{gather} \label{eq:dmft_dln_heuristic}
\end{subequations}

\subsection{Simplifying the DMFT Equation}
\label{app:simplify_dmft}

The stochastic process $f(t)$ can be eliminated. Differentiating both sides of \cref{eq:f} by $z_f(t')$ ($t > t'$) and averaging, we get
\begin{align}
    R_f(t,t') & = R_w(t,t') - \int_{t'}^t R_w(t,s) R_f(s,t') \,\de s \,. \label{eq:Rf_simple}
\end{align}
Multiplying both sides of \cref{eq:f} by $(f(t')-\xi)$ and averaging, we get
\begin{align}
    C_f(t,t') & = \E[(z_f(t)-\xi)(f(t')-\xi)] - \int_0^t R_w(t,s) \E[(f(s)-\xi)(f(t') - \xi)] \de s \,.
\end{align}
By Stein's lemma (Gaussian integration-by-parts formula), we have
\begin{align}
    \E[z_f(t)(f(t')-\xi)] & = \int_0^{t'} \Cov(z_f(t), z_f(s)) \E\ab[\diffp{(f(t')-\xi)}{z_f(s)}] \de s \notag \\
                          & = C_w(t,t') - \int_0^{t'} C_w(t,s) R_f(t',s) \de s \,.
\end{align}
By \cref{eq:f}, we have
\begin{align}
    \E[\xi f(t)] & = \sigma^2 \int_0^t R_w(t,s) \de s - \int_0^t R_w(t,s) \E[\xi f(s)] \de s \,,
\end{align}
where we used $\E[\xi^2] = \sigma^2$. Comparing this with \cref{eq:Rf_simple}, we have
\begin{align}
    \E[\xi f(t)] & = \sigma^2 \int_0^t R_f(t,s) \de s \,,
\end{align}
and thus
\begin{align}
    C_f(t,t') & = C_w(t,t') + \sigma^2 - \int_0^{t'} R_f(t',s) (C_w(t,s) + \sigma^2) \de s - \int_0^t R_w(t,s) C_f(s,t') \de s \,.
\end{align}

We can eliminate $u(t)$ and $v(t)$ and express the dynamics in terms of $w(t)$.
The product $u(t)v(t)$ obeys a solvable dynamics:
\begin{align}
    \diff{}{t}(u(t)v(t)) & = v(t) \diff{}{t}u(t) + u(t) \diff{}{t}v(t)                                        \notag \\
                         & = - \frac{1}{2} u(t)v(t) (g(t) + \lambda) - \frac{1}{2} u(t)v(t) (-g(t) + \lambda) \notag \\
                         & = - \lambda u(t)v(t) \,,
\end{align}
from which we get
\begin{align}
    u(t)v(t) & = u(0)v(0) \napier^{-\lambda t} = \alpha^2 \napier^{-\lambda t} \,.
\end{align}
The dynamics of $w(t)$ is thus
\begin{align}
    \diff{}{t}w(t) & = \frac{1}{2} \diff{}{t}(u(t)^2 - v(t)^2) \notag                                             \\
                   & = u(t) \diff{}{t}u(t) - v(t) \diff{}{t}v(t) \notag                                           \\
                   & = - \frac{1}{2} u(t)^2 (g(t) + \lambda) + \frac{1}{2} v(t)^2 (-g(t) + \lambda) \notag        \\
                   & = - \frac{1}{2} (u(t)^2 + v(t)^2) g(t) - \lambda w(t) \notag                                 \\
                   & = - \sqrt{w(t)^2 + \alpha^4 \napier^{-2\lambda t}} g(t) - \lambda w(t) \,. \label{eq:dwdt_1}
\end{align}
In the last line, we used that
\begin{align}
    u(t)^2 + v(t)^2 & = \sqrt{(u(t)^2 - v(t)^2)^2 + 4u(t)^2v(t)^2} = 2\sqrt{w(t)^2 + \alpha^4 \napier^{-2\lambda t}} \,.
\end{align}

Combining these results, we obtain the simplified DMFT equation \eqref{eq:dmft_dln} for DLNs.

\section{DERIVATION OF THE LEARNING TIMESCALES}
\label{app:timescale}

In this section, we analyze timescale structures of the gradient flow dynamics in $\alpha \to \infty$ and $\alpha \to 0$ limits for general $\delta$. We follow the approach outlined in \cref{sec:timescale}, utilizing singular perturbation theory.

\subsection{Large Initialization: $\alpha \gg 1$}
\label{app:timescale_largeinit}

We show that the dynamics for large $\alpha$ (simulation shown in \cref{fig:timescale_large}) consists of two phases: \emph{lazy phase} and \emph{rich phase}.

\begin{figure}
    \begin{center}
        \includegraphics[width=\textwidth]{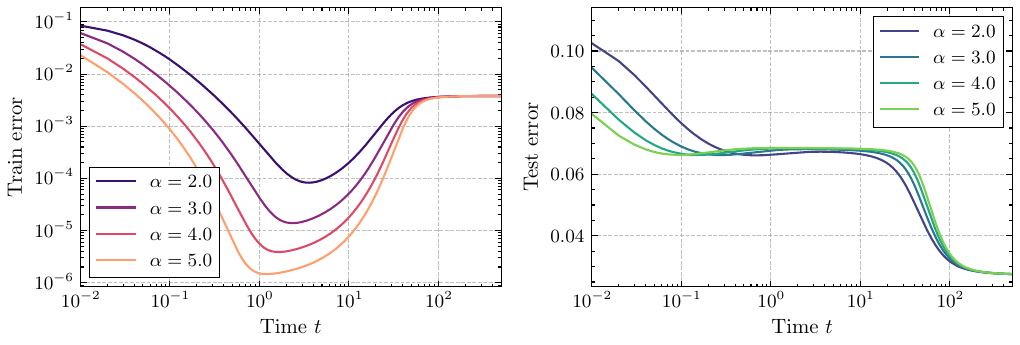}
        \caption{Training and test error dynamics for large $\alpha$ simulated with $d=200$.}
        \label{fig:timescale_large}
    \end{center}
\end{figure}

\paragraph{Lazy Phase: $t=O(\alpha^{-2})$.}

\newcommand{\lazy}{\mathrm{lazy}}

In this short timescale, the factor $\sqrt{w(t)^2 + \alpha^4 \napier^{-2\lambda t}}$ in \cref{eq:dmft_dln} is approximated as $\alpha^2 + o(1)$. This motivates a time rescaling $\bar t \coloneqq \alpha^2 t$. We then make the following ansatz for the DMFT solution on this timescale.
\begin{subequations} \label{eq:lazy_ansatz}
    \begin{gather}
        C_w(\bar t/\alpha^2,\bar t'/\alpha^2) = C_w^{\lazy}(\bar t,\bar t') + o(1) \,, \quad C_f(\bar t/\alpha^2,\bar t'/\alpha^2) = C_f^{\lazy}(\bar t,\bar t') + o(1) \,,               \\
        R_w(\bar t/\alpha^2,\bar t'/\alpha^2) = \alpha^2 R_w^{\lazy}(\bar t,\bar t') + o(\alpha^2) \,, \quad R_f(\bar t/\alpha^2,\bar t'/\alpha^2) = \alpha^2 R_f^{\lazy}(\bar t,\bar t') + o(\alpha^2) \,, \\
        \quad C_w^\lazy(\bar t,\bar t') = \E[(w^\lazy(\bar t) - w^*) (w^\lazy(\bar t') - w^*)] \,, \quad R_w^\lazy(\bar t,\bar t') = -\E\ab[\diffp{w^\lazy(\bar t)}{z^\lazy(\bar t')}] \,, \\
        w(\bar t/\alpha^2) = w^\lazy(\bar t) + o(1) \,, \quad z^\lazy \sim \GP(0,\delta C_f^\lazy) \,,
    \end{gather}
\end{subequations}
where $C_f^\lazy,R_f^\lazy$ are functions independent of $\alpha$ and $w^\lazy$ is a stochastic process independent of $\alpha$.

Up to the leading order, the dynamics of $w(t)$ are written as follows.
\begin{align}
    \diff{}{\bar t}w(\bar t/\alpha^2) & = - g(\bar t/\alpha^2) + O(\alpha^{-2}) \notag                                                                                                                   \\
                                      & = - \frac{z(\bar t/\alpha^2)}{\delta} - (w(\bar t/\alpha^2) - w^*) + \int_0^{\bar t} R_f^{\lazy}(\bar t,\bar s) (w(\bar s/\alpha^2) - w^*) \de \bar s + o(1) \,.
\end{align}
Thus, $w^\lazy(\bar t)$ satisfies a linear integro-differential equation
\begin{align}
    \diff{}{\bar t}w^\lazy(\bar t) = - \frac{z^\lazy(\bar t)}{\delta} - (w^\lazy(\bar t) - w^*) + \int_0^{\bar t} R_f^{\lazy}(\bar t,\bar s) (w^\lazy(\bar s) - w^*) \de \bar s \,, \quad z^\lazy \sim \GP(0,\delta C_f^\lazy) \,.
\end{align}

The stochastic process $w^\lazy(\bar t)$ can be eliminated along the same lines as in \cref{app:simplify_dmft} using the linearity of its dynamics to yield a closed system for correlation and response functions.
\begin{subequations}
    \begin{align}
        \diffp{}{\bar t}C_w^{\lazy}(\bar t,\bar t')  & = - C_w^{\lazy}(\bar t,\bar t') + \int_0^{\bar t'} R_w^{\lazy}(\bar t',\bar s) C_f^{\lazy}(\bar t,\bar s) \de \bar s + \int_0^{\bar t} R_f^{\lazy}(\bar t,\bar s) C_w^{\lazy}(\bar t',\bar s) \de \bar s \,,                        \\
        C_f^{\lazy}(\bar t,\bar t')                  & = C_w^{\lazy}(\bar t,\bar t') + \sigma^2 - \int_0^{\bar t'} R_f^{\lazy}(\bar t',\bar s) (C_w^{\lazy}(\bar t,\bar s) + \sigma^2) \de \bar s - \int_0^{\bar t} R_w^{\lazy}(\bar t,\bar s) C_f^{\lazy}(\bar t',\bar s) \de \bar s  \,, \\
        \diffp{}{\bar t} R_w^{\lazy}(\bar t,\bar t') & = - R_w^{\lazy}(\bar t,\bar t') + \int_{\bar t'}^{\bar t} R_f^{\lazy}(\bar t,\bar s) R_w^{\lazy}(\bar s,\bar t') \de \bar s  \,,                                                                                                    \\
        R_f^{\lazy}(\bar t,\bar t')                  & = R_w^{\lazy}(\bar t,\bar t') - \int_{\bar t'}^{\bar t} R_w^{\lazy}(\bar t,\bar s) R_f^{\lazy}(\bar s,\bar t') \de \bar s \,,
    \end{align} \label{eq:dmft_lazy}
\end{subequations}
with boundary conditions $C_w^\lazy(\bar t,0) = C_w^\lazy(0,\bar t) = 0$ and $R_w^\lazy(\bar t,\bar t)=1/\delta$ for $\bar t \geq 0$.

These equations are equivalent to the ones for (ridgeless) linear regression derived in \citet{fan2025dynamical} and \citet{bordelon2024dynamical}. Thus, in this dynamical regime, DLNs behave as linear models. \Cref{eq:dmft_lazy} can be solved explicitly as follows.
\begin{subequations}
    \begin{align}
        R_w^{\lazy}(\bar t,\bar t') & = \frac{1}{\delta} \int \napier^{-x (\bar t - \bar t')} \de \mu_\MP(x) \,,                                                                                                 \\
        R_f^{\lazy}(\bar t,\bar t') & = \frac{1}{\delta} \int x \napier^{-x(\bar t - \bar t')} \de \mu_\MP(x) \,,                                                                                                \\
        C_w^{\lazy}(\bar t,\bar t') & = \rho^2 \int \napier^{-x(\bar t+\bar t')} \de\mu_\MP(x) + \frac{\sigma^2}{\delta} \int \frac{1}{x} (1 - \napier^{-x\bar t})(1 - \napier^{-x\bar t'}) \de\mu_\MP(x) \,,    \\
        C_f^{\lazy}(\bar t,\bar t') & = \rho^2 \int x \napier^{-x(\bar t+\bar t')} \de\mu_\MP(x) + \frac{\sigma^2}{\delta} \int \napier^{-x(\bar t+\bar t')} \de\mu_\MP(x) + \frac{\delta-1}{\delta}\sigma^2 \,,
    \end{align}
\end{subequations}
where $\mu_\MP$ is the \emph{Marchenko--Pastur law}, the limiting eigenvalue spectrum of a random matrix $\delta^{-1} \bX^\transpose \bX$, whose density is given explicitly as follows.
\begin{align}
    \de \mu_\MP(x) & = \frac{\delta \sqrt{(\lambda_+ - x)(x - \lambda_-)}}{2\pi x} + (1-\delta)\dirac(x) \bone_{\delta < 1}, \quad \text{where} \; \lambda_{\pm} = \ab(1 \pm \frac{1}{\sqrt{\delta}})^2 \,, \label{eq:mp_law}
\end{align}
where $\dirac(x)$ is the Dirac delta function. Training and test errors are given by
\begin{align}
    E_\text{train}^\lazy(\bar t) = C_f^{\lazy}(\bar t,\bar t) \,, \quad E_\text{test}^\lazy(\bar t) = C_w^\lazy(\bar t,\bar t) + \sigma^2 \,. \label{eq:lazy_sol}
\end{align}
These solutions are checked against simulations in \cref{fig:timescale_large_lazy} and show good agreement.

\begin{figure}
    \begin{center}
        \includegraphics[width=\textwidth]{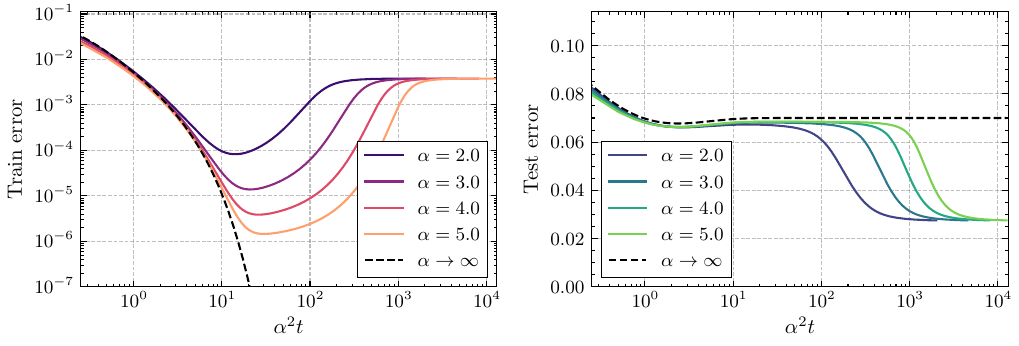}
        \caption{Training and test error dynamics for large $\alpha$, with time rescaled by $\alpha^2$. The initial descent of training and test errors collapses onto the limiting solution \eqref{eq:lazy_sol}}
        \label{fig:timescale_large_lazy}
    \end{center}
\end{figure}

\paragraph{Rich Phase: $t = 2\log(\alpha)/\lambda$.}
When $\lambda > 0$, the ansatz \eqref{eq:lazy_ansatz} breaks down when $w(t)$ and $\alpha^4 \napier^{-2\lambda t}$ are of the same order, which occurs at $t \approx t_c \coloneqq 2 \log \alpha / \lambda$.
Introducing a new time variable as $\bar t = t - t_c$, the parameter $w(\bar t)$ obeys the following new equation.
\begin{align}
    \diff{}{\bar t}w(\bar t) & = - \sqrt{w(\bar t)^2 + \napier^{-2\lambda \bar t}} g(\bar t) - \lambda w(\bar t) \,. \label{eq:dmft_rich}
\end{align}
A fixed point analysis in \cref{app:fixed_point} reveals that this equation converges to the $\ell_1$-regularized solution in time $O(1)$. These results are checked against simulations in \cref{fig:timescale_large_grok,fig:timescale_large_rich}.

\begin{figure}
    \begin{center}
        \includegraphics[width=\textwidth]{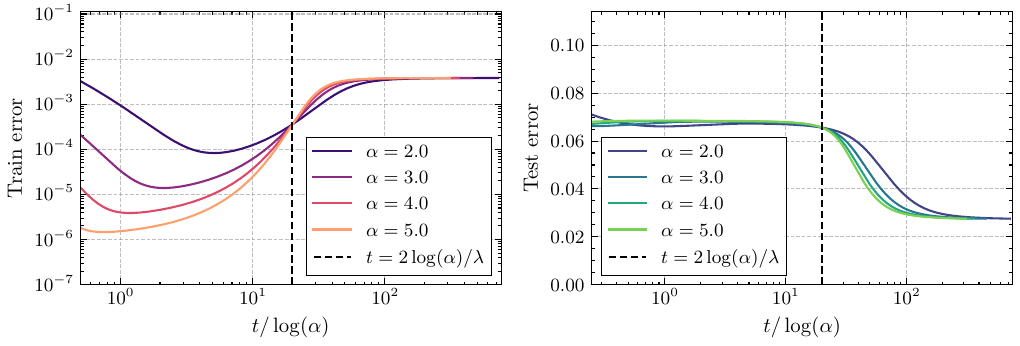}
        \caption{Training and test error dynamics for large $\alpha$, with time rescaled by $\log(\alpha)$. Transition times to the rich phase collapse to the same value of $t/\log(\alpha) = 2/\lambda$.}
        \label{fig:timescale_large_grok}
    \end{center}
\end{figure}

\begin{figure}
    \begin{center}
        \includegraphics[width=\textwidth]{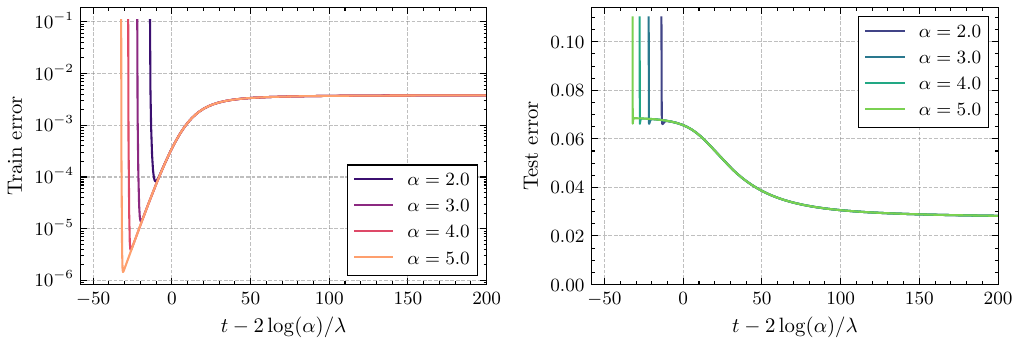}
        \caption{Training and test error dynamics for large $\alpha$, with time shifted by the transition time $2\log(\alpha)/\lambda$. These curves collapse, indicating that the dynamics after the transition proceed in time $O(1)$.}
        \label{fig:timescale_large_rich}
    \end{center}
\end{figure}

\subsection{Small Initialization: $\alpha \ll 1$}
\label{app:timescale_smallinit}

We show that the dynamics for small $\alpha$ (simulation shown in \cref{fig:timescale_small}) consists of two phases: \emph{search phase} and \emph{descent phase}.

\begin{figure}
    \begin{center}
        \includegraphics[width=\textwidth]{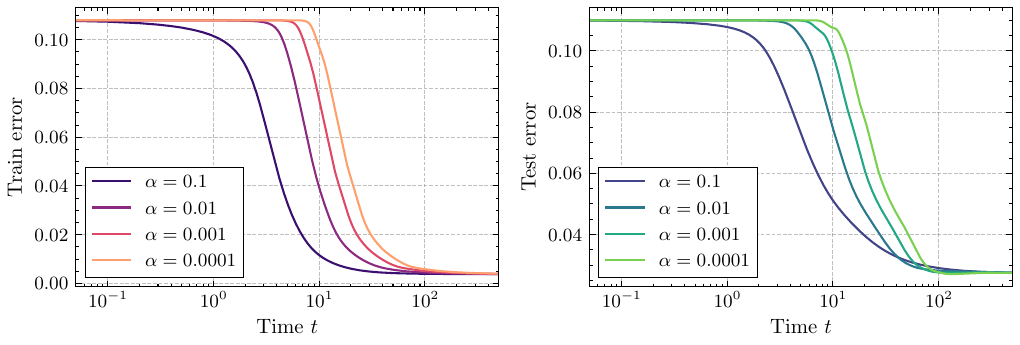}
        \caption{Training and test error dynamics for small $\alpha$ simulated with $d=200$.}
        \label{fig:timescale_small}
    \end{center}
\end{figure}

\paragraph{Search Phase: $t=O(1)$.}
Let $W(t) = w(t) / \alpha^2$. Assuming that $W(t) = O(1)$, the DMFT equation is approximated up to the leading order as follows.
\begin{alignat}{2}
    C_w(t,t') & = \rho^2 + O(\alpha^2) \,, & \quad C_f(t,t') & = \rho^2 + \sigma^2 + O(\alpha^2)  \,, \\
    R_w(t,t') & = O(\alpha^2) \,,          & \quad R_f(t,t') & = O(\alpha^2) \,,
\end{alignat}
and the dynamics of $W(t)$ is given by
\begin{gather}
    \diff{}{t}W(t) = (w^* - z(0)/\delta) \sqrt{W(t)^2 + \napier^{-2\lambda t}}  - \lambda W(t) + o(1) \,, \quad z(0) \sim \normal(0,\delta(\rho^2+\sigma^2)) \,.
\end{gather}
This equation can be solved explicitly as
\begin{align}
    W(t) = \frac{\sign(w^* - z(0)/\delta)}{2} (1 - \napier^{-2\abs{w^* - z(0)/\delta} t}) \napier^{(\abs{w^* - z(0)/\delta} - \lambda) t} \,.
\end{align}
Let $\Delta = \abs{w^* - z(0)/\delta} - \lambda$. For large $t$, a sample path $W(t)$ behaves as $\abs{W(t)} \approx (1/2) \napier^{\Delta t}$. When $\Delta < 0$, $W(t)$ converges exponentially to zero; when $\Delta > 0$, $\abs{W(t)}$ grows exponentially.

The noise term $z(0)/\delta$ captures a finite sample effect, which essentially acts as a noise that obscures the ground truth $w^*$. It vanishes as $\delta \to \infty$.

\paragraph{Descent Phase: $t=\Theta(\log(1/\alpha))$.}

\newcommand{\descent}{\mathrm{desc}}

As sample paths with $\Delta > 0$ grow, the assumption that $W(t) = O(1)$ breaks down. A transition to the second dynamical regime occurs when $w(t)=\alpha^2 W(t)$ becomes of $O(1)$, which happens at a timescale of $\Theta(\log(1/\alpha))$. This motivates the following rescaling of dynamical variables with $\bar t = t / \log(1/\alpha)$.
\begin{subequations}
    \begin{gather}
        C_w(\log(1/\alpha)\bar t,\log(1/\alpha)\bar t') = C_w^{\descent}(\bar t,\bar t')\,, \quad C_f(\log(1/\alpha) \bar t,\log(1/\alpha)\bar t') = C_f^{\descent}(\bar t,\bar t') \,, \\
        R_w(\log(1/\alpha)\bar t,\log(1/\alpha)\bar t') = \frac{1}{\log(1/\alpha)} R_w^{\descent}(\bar t,\bar t') \,, \quad R_f(\log(1/\alpha)\bar t,\log(1/\alpha)\bar t') = \frac{1}{\log(1/\alpha)} R_f^{\descent}(\bar t,\bar t') \,, \\
        w(\log(1/\alpha)\bar t) = w^{\descent}(\bar t) \,, \quad g(\log(1/\alpha)\bar t) = g^{\descent}(\bar t) \,.
    \end{gather}
\end{subequations}
The rescaled DMFT equation is
\begin{subequations}
    \begin{align}
        C_w^{\descent}(\bar t,\bar t')                              & = \E[(w^{\descent}(\bar t) - w^*) (w^{\descent}(\bar t') - w^*)] \,,                                                                                                                                                                                \\
        R_w^{\descent}(\bar t,\bar t')                              & = -\log(1/\alpha)\E\ab[\diffp{w^{\descent}(\bar t)}{z^{\descent}(\bar t')} ]  \,,                                                                                                                                                                   \\
        C_f^{\descent}(\bar t,\bar t')                              & = C_w^{\descent}(\bar t,\bar t') + \sigma^2 - \int_0^{\bar t'} R_f^{\descent}(\bar t',\bar s) (C_w^{\descent}(\bar t,\bar s) + \sigma^2)  \de \bar s - \int_0^{\bar t} R_w^{\descent}(\bar t,\bar s) C_f^{\descent}(\bar t',\bar s) \de \bar s  \,, \\
        R_f^{\descent}(\bar t,\bar t')                              & = R_w^{\descent}(\bar t,\bar t') - \int_{\bar t'}^{\bar t} R_w^{\descent}(\bar t,\bar s) R_f^{\descent}(\bar s,\bar t') \de \bar s \,,                                                                                                              \\
        g^{\descent}(\bar t)                                        & = \frac{z^{\descent}(\bar t)}{\delta} + w^{\descent}(\bar t) - w^* - \int_0^{\bar t} R_f^{\descent}(\bar t,\bar s) (w^{\descent}(\bar s) - w^*) \de \bar s \,, \quad z^{\descent} \sim \GP(0,\delta C_f^{\descent}) \,,                             \\
        \frac{1}{\log(1/\alpha)}\diff{}{\bar t}w^{\descent}(\bar t) & = - \sqrt{w^{\descent}(\bar t)^2 + \alpha^{4 + 2\lambda \bar t}} g^{\descent}(\bar t) - \lambda w^{\descent}(\bar t) \,.
    \end{align}
\end{subequations}

The time it takes for each path to become active (become of $\Theta(1)$) can be derived as follows. Let $W(\bar t) = w^{\descent}(\bar t) / \alpha^2$. Assuming that $1 \ll W(\bar t) \ll \alpha^{-2}$, the dynamics of $W(\bar t)$ is approximated up to the leading order as
\begin{align}
    \frac{1}{\log(1/\alpha)}\diff{}{\bar t}W(\bar t) & \approx - \abs{W(\bar t)} \ab(\frac{z^{\descent}(\bar t)}{\delta} - w^* + \int_0^{\bar t} R_f^{\descent}(\bar t,\bar s) w^* \de \bar s)  - \lambda W(\bar t) \,.
\end{align}
It can be solved as
\begin{align}
    W(\bar t) & \propto \exp\ab(\log(1/\alpha) \int_0^{\bar t} \abs*{\frac{z^{\descent}(\bar t')}{\delta} - w^* + \int_0^{\bar t'} R_f^{\descent}(\bar t',\bar s) w^* \de \bar s} \de \bar t' - \log(1/\alpha) \lambda \bar t) \,.
\end{align}
Thus, the time $\bar t_c$ at which $w(\bar t) = \alpha^2 W(\bar t)$ becomes of $\Theta(1)$ is given implicitly by
\begin{align}
    \int_0^{\bar t_c} \abs*{\frac{z^{\descent}(\bar t')}{\delta} - w^* + \int_0^{\bar t'} R_f^{\descent}(\bar t',\bar s) w^* \de \bar s} \de \bar t' - \lambda \bar t_c = 2 \,.
\end{align}
When $\delta \to \infty$, the left-hand side of the above equation reduces to $(\abs{w^*} - \lambda) \bar t_c$, and we thus have $t_c = \log(1/\alpha) \bar t_c = 2\log(1/\alpha) / (\abs{w^*} - \lambda)$ for the transition time, as derived in \cref{sec:timescale}.
As already discussed in the main text, the transition times $\bar t_c$ are different for each path, as opposed to the large initialization ($\alpha \gg 1$) case where the transition time $t_c = 2\log(\alpha)/\lambda$ is the same for all paths. We therefore observe \emph{incremental learning} with successive activation of paths.

After the transition, defining a new time variable $\bar t = t-\log(1/\alpha) \bar t_c$, the new dynamics is
\begin{align}
    \diff{}{\bar t}w(\bar t) & \approx - \abs{w(\bar t)} g(\bar t) - \lambda w(\bar t) \,,
\end{align}
which behaves similarly to \cref{eq:dmft_rich} for large $\bar t$.
$w(\bar t)$ converges in time $O(1)$ to the $\ell_1$ regularized solution for $\lambda > 0$ and minimum $\ell_1$ norm interpolator for $\lambda = 0$, as described in \cref{sec:longtime}.

The timescale is checked against numerical simulations in \cref{fig:timescale_small_transition}, showing that $\log(1/\alpha)$ is indeed the correct time scaling.

\begin{figure}
    \begin{center}
        \includegraphics[width=\textwidth]{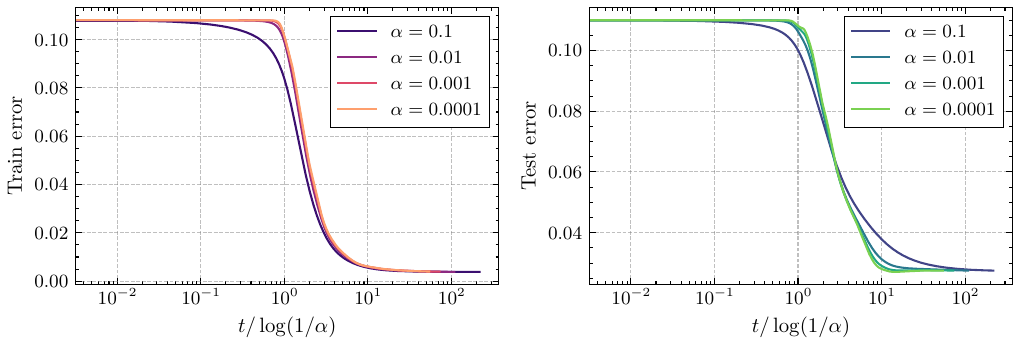}
        \caption{Training and test error dynamics for small $\alpha$, with time rescaled by $\log(1/\alpha)$. Learning curves collapse, indicating that the descent phase proceeds on the timescale $\Theta(\log(1/\alpha))$.}
        \label{fig:timescale_small_transition}
    \end{center}
\end{figure}

\section{DERIVATION OF THE LONG-TIME BEHAVIOR}
\label{app:longtime}

In this section, we derive \cref{res:fixedpoint,res:convergence_rates} by analyzing long-time behaviors of the DMFT equation \eqref{eq:dmft_dln}.

\subsection{Preliminary: Laplace Transform}

Throughout this section, we make extensive use of the \emph{Laplace transform}, which is a useful technique for analyzing linear differential equations. Given a function $f \colon \reals_{\geq 0} \to \reals$, its Laplace transform $\mathcal{L}[f] = \bar f$ is defined as
\begin{align}
    \bar f(p) \coloneqq \int_0^\infty f(t) \napier^{-p t} \de t \,,
\end{align}
for $p \in \complex$ with sufficiently large real part for the integral to be convergent.

We state several of its basic properties.
\begin{itemize}
    \item \emph{Linearity}. For $f \colon \reals_{\geq 0} \to \reals$ and $a,b \in \reals$, we have
          \begin{align}
              \mathcal{L}[af + b] = a\mathcal{L}[f] + b \,.
          \end{align}
    \item \emph{Laplace transforms of derivatives, integrals, and convolutions}. For $f,g \colon \reals_{\geq 0} \to \reals$, we have
          \begin{align}
              \mathcal{L}\ab[f'(t)](p)                      & = p \bar f(p) - f(0) \,,  \\
              \mathcal{L}\ab[\int_0^t f(s) \de s](p)        & = \frac{\bar f(p)}{p} \,, \\
              \mathcal{L}\ab[\int_0^t f(t-s) g(s) \de s](p) & = \bar f(p) \bar g(p) \,.
          \end{align}
    \item \emph{The final value theorem}. For $f \colon \reals_{\geq 0} \to \reals$, we have
          \begin{align}
              \lim_{t \to \infty} f(t) = \lim_{p \to 0} p \bar f(p) \,,
          \end{align}
          if all singularities of $\bar f$ lie on the left half-plane.
    \item \emph{Convergence rate}. For $f \colon \reals_{\geq 0} \to \reals$, let $p_c$ be the singularity of $\bar f$ with the largest real part. Then, as $t \to \infty$, $f(t) = \exp(- (\Real p_c) t + o(1))$.
\end{itemize}

\subsection{Fixed Point}
\label{app:fixed_point}

\subsubsection{Deriving the Fixed-Point Equations}

First, we derive equations for the fixed point of the DMFT equation \eqref{eq:dmft_dln}. The fixed point can be obtained by taking $t \to \infty$ and setting the time derivative $\difs{w(t)}{t}$ to zero. We make a \emph{time-translation invariance} (TTI) ansatz to handle integrals with response functions.

\paragraph{Case (i): $\lambda > 0$.} We assume that a long-time limit $t \to \infty$ exists with TTI response functions.
\begin{itemize}
    \item $w(t) \to w$ and $g(t) \to g$ (random, constant for each path),
    \item $C_w(t,t) \to C_w$ and $C_f(t,t) \to C_f$ (constants),
    \item $R_w(t,t') \approx R_w(t-t')$ and $R_f(t,t') \approx R_f(t-t')$ with $R_w(t),R_f(t) \to 0$ and $R_w(t),R_f(t)$ are both integrable.
\end{itemize}

Denote the integrated responses (\emph{susceptibilities}) as
\begin{align}
    \chi_w \coloneqq \int_0^\infty R_w(t) \de t \,, \quad \chi_f \coloneqq \int_0^\infty R_f(t) \de t \,.
\end{align}

Using TTI, the DMFT equation for $R_f$ becomes
\begin{align}
    R_f(t) = R_w(t) - \int_{0}^t R_w(t-s) R_f(s) \de s \,.
\end{align}
Taking the Laplace transform of both sides, we obtain
\begin{align}
    \bar R_f(p) = \bar R_w(p) - \bar R_w(p) \bar R_f(p) \,,
\end{align}
from which we get
\begin{align}
    \bar R_f(p) = \frac{\bar R_w(p)}{1 + \bar R_w(p)} \,. \label{eq:laplace_R}
\end{align}
Since $\bar R_w(0) = \chi_w$ and $\bar R_f(0) = \chi_f$, we get
\begin{align}
    \chi_f = \frac{\chi_w}{1 + \chi_w} \,. \label{eq:fix_chi_f}
\end{align}
Similar manipulations for $C_f$ yield
\begin{align}
    C_f = \frac{C_w + \sigma^2}{(1 + \chi_w)^2} \,.
\end{align}

Next, we derive pathwise fixed points of $w$ and $g$. As $t \to \infty$, the factor $\napier^{-2\lambda t}$ vanishes, and setting $\difs{w(t)}{t} = 0$ yields
\begin{align}
    0 = -\abs{w} g - \lambda w \,, \quad g = \frac{z}{\delta} + w - w^* - (w - w^*) \chi_f \,, \quad z \sim \normal(0,\delta C_f) \,.
\end{align}
From the first equation, we get $w = 0$ or $g = -\lambda \sign(w)$. In the case of $g = -\lambda \sign(w)$ we get from the second equation that
\begin{align}
    w + (1 + \chi_w)\lambda \sign(w) = w^* - \frac{1 + \chi_w}{\delta} z \,.
\end{align}
This equation has a solution if and only if $\abs{w} \geq (1 + \chi_w)\lambda$, otherwise we have $w = 0$. These solutions can be expressed using the soft thresholding function as
\begin{align}
    w = \ST\ab(w^* - \frac{1 + \chi_w}{\delta}z; (1 + \chi_w) \lambda) \,.
\end{align}

Finally, the response of $w$ to a \emph{constant} input $z$ gives the integrated response $\chi_w$.
\begin{align}
    \chi_w = -\E\ab[\diffp{w}{z}] = \frac{1 + \chi_w}{\delta} \partial_x \ST\ab(w^* - \frac{1 + \chi_w}{\delta}z; (1+\chi_w)\lambda) \,.
\end{align}

Collecting these results, we obtain the following system of equations for the fixed point.
\begin{equation}
    \boxed{\begin{gathered}
            C_w = \E[(w - w^*)^2] \,, \quad \chi_w = \frac{1+\chi_w}{\delta} \E\ab[\partial_x \ST\ab(w^* - \frac{1+\chi_w}{\delta}z; (1+\chi_w) \lambda)] \,, \quad C_f = \frac{C_w + \sigma^2}{(1 + \chi_w)^2} \,, \\
            w = \ST\ab(w^* - \frac{1 + \chi_w}{\delta} z; (1 + \chi_w) \lambda) \,, \quad z \sim \normal(0,\delta C_f) \,.
        \end{gathered}} \label{eq:fix_case1}
\end{equation}
This result is validated numerically as shown in \cref{fig:fixedpoint_case1}.

We note that the susceptibility $\chi_w$ is related to the \emph{train-test gap} because $(1 + \chi_w)^2$ is equal to the ratio between fixed points of training and test errors, $E_\text{train} = C_f$ and $E_\text{test} = C_w + \sigma^2$.

\begin{figure}
    \begin{center}
        \includegraphics[width=\textwidth]{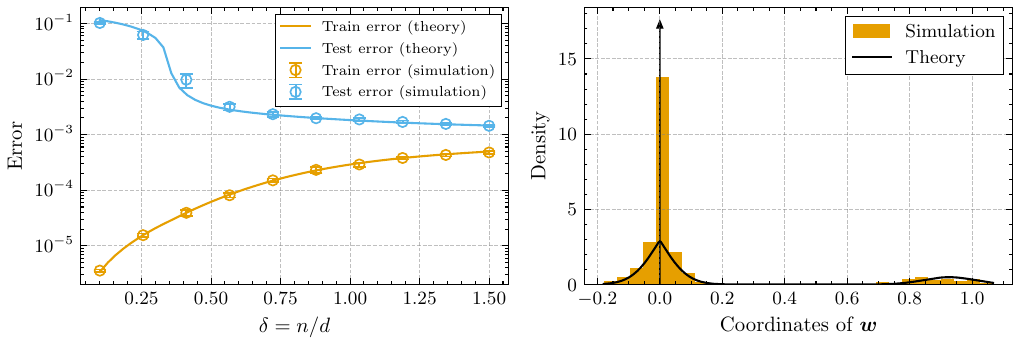}
        \caption{Fixed points for case (i) ($\lambda > 0$). (Left): Fixed point of train and test errors for different $\delta$. Simulations are run 10 times on independent data with $d=500$, and fixed points are obtained at $t=1000$. Error bars indicate one standard deviation. (Right): Distribution of coordinates of $\bw(\infty)$. The arrow indicates the delta spike at $w=0$ with its height scaled by the bin size.}
        \label{fig:fixedpoint_case1}
    \end{center}
\end{figure}

\paragraph{Case (ii): $\lambda = 0$, $\delta > 1$.} Calculation proceeds along the same lines as Case (i). We have the same fixed point equations for $\chi_f$ and $C_f$:
\begin{align}
    \chi_f = \frac{\chi_w}{1 + \chi_w} \,, \quad C_f = \frac{C_w + \sigma^2}{(1 + \chi_w)^2} \,. \label{eq:fix_fs}
\end{align}

Fixed point conditions for $w$ and $g$ are
\begin{align}
    0 = -\sqrt{w^2 + \alpha^4} g \,, \quad g = \frac{z}{\delta} + w - w^* - (w - w^*)\chi_f \,, \quad z \sim \normal(0,\delta C_f) \,.
\end{align}
From the first equation, we have $g = 0$. From the second equation, we get
\begin{align}
    w = w^* - \frac{1 + \chi_w}{\delta} z \,.
\end{align}

The susceptibility $\chi_w$ satisfies
\begin{align}
    \chi_w = -\E[\partial_z w] = \frac{1 + \chi_w}{\delta} \,, \quad \text{therefore} \; \chi_w = \frac{1}{\delta - 1} \,.
\end{align}

We can simplify the fixed point equation using the explicit form of $\chi_w$ to arrive at the following result.
\begin{equation}
    \boxed{\begin{gathered}
            C_w = \frac{\sigma^2}{\delta - 1} \,, \quad \chi_w = \frac{1}{\delta - 1} \,, \quad C_f = \frac{\delta - 1}{\delta} \sigma^2 \,, \\
            w = w^* - \frac{z}{\delta - 1} \,, \quad z \sim \normal(0,\delta C_f) \,.
        \end{gathered}} \label{eq:fix_case2}
\end{equation}
This result is validated numerically as shown in \cref{fig:fixedpoint_case2}.

If $\sigma^2 > 0$, the test error $C_w + \sigma^2$ diverges as $\delta \to 1$. This is the well-known \emph{double descent} peak \citep{belkin2019reconciling,hastie2022surprises}.

\begin{figure}
    \begin{center}
        \includegraphics[width=\textwidth]{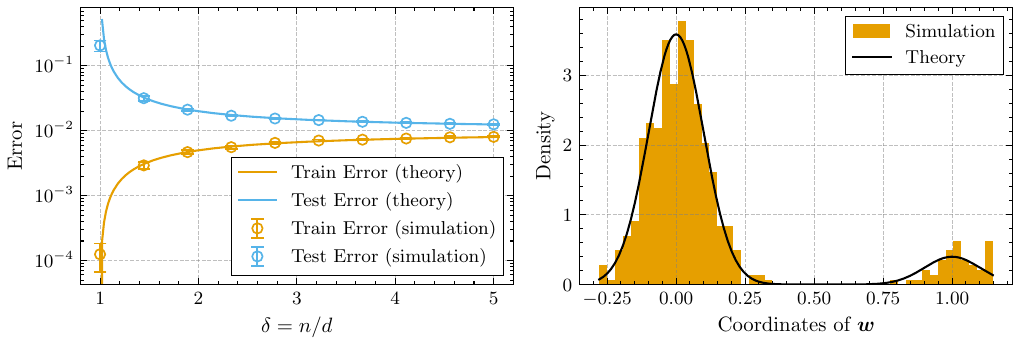}
        \caption{Fixed points for case (ii) ($\lambda =0$, $\delta > 1$). (Left): Fixed point of train and test errors for different $\delta$. Simulations are run 10 times on independent data with $d=500$, and fixed points are obtained at $t=1000$. Error bars indicate one standard deviation. (Right): Distribution of coordinates of $\bw(\infty)$.}
        \label{fig:fixedpoint_case2}
    \end{center}
\end{figure}

\paragraph{Case (iii): $\lambda = 0$, $\delta < 1$.} This case requires a more careful argument than the previous cases, since in this case the response function $R_w(\tau)$ does not vanish as $\tau \to \infty$. This corresponds to the fact that the minimum of the loss function is degenerate and that a perturbation to the system will permanently shift the solution.

Assuming that $R_w$ converges to a nonzero constant, its integral $\chi_w$ diverges to infinity. By \eqref{eq:fix_fs}, we have $\chi_f = 0$ and $C_f = 0$. To obtain fixed points for other variables, we need to know how fast they converge. Thus, we make the following ansatz
\begin{itemize}
    \item $w(t) \to w$ and $C_w(t,t) \to C_w$,
    \item $g(t) \to 0$ and integrable,
    \item $C_f(t,t') \to 0$ and integrable on $\reals_{\geq 0}^2$,
    \item $R_w(t,t') = R_w(t-t')$ and $R_w(t) \to R_w$,
    \item $R_f(t,t') = R_f(t-t')$ and $1 - \int_0^t R_f(s) \de s$ is integrable (with respect to $t$),
\end{itemize}
and define
\begin{align}
    \tilde \chi_f \coloneqq \int_0^\infty \ab(1 - \int_0^t R_f(s)\de s) \de t \,, \quad \tilde C_f \coloneqq \int_0^\infty \int_0^\infty C_f(t,t') \de t \de t' \,.
\end{align}

By \cref{eq:laplace_R}, we get
\begin{align}
    \frac{1 - \bar R_f(p)}{p} = \frac{1}{p(1 + \bar R_w(p))} \,.
\end{align}
Since $\lim_{p \to 0} (1 - \bar R_f(p)) / p = \tilde \chi_f$ and $\lim_{p \to 0} p \bar R_w(p) = R_w$, we have
\begin{align}
    \tilde \chi_f = \frac{1}{R_w} \,.
\end{align}
Similarly, for $C_f$, we get
\begin{align}
    \tilde C_f = \frac{C_w + \sigma^2}{R_w^2} \,.
\end{align}

Next, we derive the fixed point condition for $w$. The DMFT equation for $\difs{w(t)}{t}$ can be transformed as
\begin{align}
    \diff{}{t} \sinh^{-1} (w(t) / \alpha^2) & = -g(t) \,.
\end{align}
Integrating both sides, we have
\begin{align}
    \sinh^{-1} (w / \alpha^2) = -\int_0^\infty g(t) \de t \eqqcolon - \tilde g \,.
\end{align}

Next, we derive a fixed point condition for $g$. Integrating both sides of the DMFT equation for $g$, we get
\begin{align}
    \tilde g & = \frac{1}{\delta} \int_0^\infty z(t) \de t + (w - w^*) \int_0^\infty \ab(1 - \int_0^t R_f(t-s) \de s) \de t \notag \\
             & = \frac{1}{\delta} \int_0^\infty z(t) \de t + (w - w^*) \tilde \chi_f \,.
\end{align}
$\int_0^\infty z(t) \de t$ follows a Gaussian distribution with mean zero and variance
\begin{align}
    \E\ab[\ab(\int_0^\infty z(t) \de t)^2] = \int_0^\infty \int_0^\infty \E[z(t)z(t')] \de t \de t' = \int_0^\infty \int_0^\infty \delta C_f(t,t') \de t \de t' = \delta \tilde C_f \,.
\end{align}
Thus, we have
\begin{align}
    w + R_w \sinh^{-1}(w / \alpha^2) = w^* - \frac{R_w}{\delta} \tilde z \,, \quad \tilde z \sim \normal(0,\delta \tilde C_f) \,. \label{eq:fix_w_case3}
\end{align}

Finally, differentiating $w$ with respect to $\tilde z$ and taking the expectation gives the response $R_w$. Differentiating both sides of \eqref{eq:fix_w_case3} by $\tilde z$,
\begin{gather}
    \partial_{\tilde z} w + \frac{R_w \partial_{\tilde z} w}{\sqrt{w^2 + \alpha^4}} = -\frac{R_w}{\delta} \,. \notag \\
    R_w = -\E[\partial_{\tilde z} w] = \frac{R_w}{\delta} \E\ab[\frac{1}{1 + R_w / \sqrt{w^2 + \alpha^4}}] \,.
\end{gather}

Collecting these results, we have
\begin{equation}
    \boxed{\begin{gathered}
            C_w = \E[(w - w^*)^2] \,, \quad 1 = \frac{1}{\delta} \E\ab[\frac{1}{1 + R_w/\sqrt{w^2 + \alpha^4}}]  \,, \quad \tilde C_f = \frac{C_w + \sigma^2}{R_w^2} \,, \\
            w = f\ab(w^* - \frac{R_w}{\delta}\tilde z; \alpha^2,R_w) \,, \quad \tilde z \sim \normal(0,\delta \tilde C_f) \,. \\
        \end{gathered}} \label{eq:fix_case3}
\end{equation}
Here, $f(x;a,b)$ is the inverse (with respect to $x$) of a function
\begin{align}
    g(x;a,b) = x + b \sinh^{-1}(x/a) \,.
\end{align}
This result is validated numerically as shown in \cref{fig:fixedpoint_case3}.

\begin{figure}
    \begin{center}
        \includegraphics[width=\textwidth]{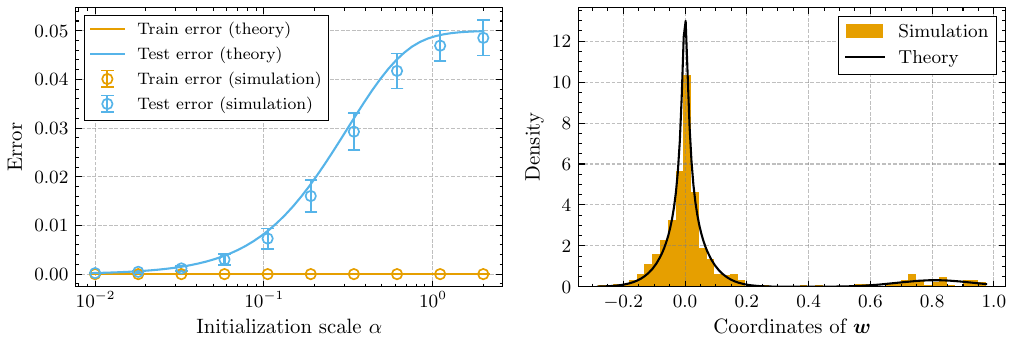}
        \caption{Fixed points for case (iii) ($\lambda =0$, $\delta < 1$). (Left): Fixed point of train and test errors for different $\alpha$. Simulations are run 10 times on independent data with $d=500$, and fixed points are obtained at $t=1000$. Error bars indicate one standard deviation. (Right): Distribution of coordinates of $\bw(\infty)$.}
        \label{fig:fixedpoint_case3}
    \end{center}
\end{figure}

\subsubsection{Analyzing the Minimization Problem}

We characterize the fixed point distributions \eqref{eq:fix_case1}, \eqref{eq:fix_case2}, \eqref{eq:fix_case3} as solutions of minimization problems. We consider the following minimization problem.
\begin{align}
    \hat \bw \coloneqq \argmin_{\bw \in \reals^d} \frac{1}{2n} \norm{\bX \bw - \by}_2^2 + \frac{\lambda}{d} \sum_{i=1}^d J(w_i) \,, \label{eq:min_problem}
\end{align}
for $\lambda > 0$ and a convex function $J \colon \reals \to \reals$. For $\lambda > 0$ and strictly convex $J$, the solution $\hat \bw$ is unique.

We proceed as follows.
\begin{enumerate}
    \item We characterize the empirical distribution of the entries of the minimizer $\hat \bw$ in the high-dimensional limit $n,d \to \infty$ using \emph{approximate message passing} (AMP). This gives a self-consistent equation for the limiting distribution.
    \item We show that, with a specific choice of the norm $J$, the self-consistent equation becomes equivalent to the fixed point equations for the DMFT equation.
\end{enumerate}

\paragraph{Characterizing the Empirical Distribution of the Minimizer via AMP.}
AMP is an iterative algorithm for solving high-dimensional statistical estimation tasks. A distinct feature of AMP is that its behavior can be rigorously tracked using a scalar recursion called \emph{state evolution} (SE). For a tutorial on AMP, see \citet{feng2022unifying}.

Let $\eta$ be the proximal operator defined as
\begin{align}
    \eta(u;t) = \prox_{t J}(u) \coloneqq \argmin_{x \in \reals} \ab\{ \frac{1}{2}(x-u)^2 + t J(x) \} \,.
\end{align}
We consider the following AMP iteration for $k=0,1,\dots$.
\begin{align}
    \hat \br^k = \by - \bX \hat\bw^k + b_k\hat \br^{k-1} \,, \quad \hat\bw^{k+1} = \eta\ab(\hat\bw^{k} + \frac{1}{\delta} \bX^\transpose \hat \br^k;t_{k+1}) \,, \label{eq:amp}
\end{align}
initialized with $\hat \bw^0 = \hat \br^{-1} = \bzero$, together with its state evolution
\begin{equation}
    \begin{gathered}
        \sigma_1^2 = \frac{\sigma^2 + \E[(W^*)^2]}{\delta} \,, \quad t_1 = \lambda (1 + b_0) \,, \quad b_k = \frac{1}{\delta} \E[\eta'(W^* + \sigma_k G_k; t_k)] \,, \\
        \sigma_{k+1}^2 = \frac{\sigma^2 + \E[(W^* - \eta(W^* + \sigma_k G_k; t_k))^2]}{\delta} \,, \quad t_{k+1} = \lambda + b_k t_k \,,
    \end{gathered}
\end{equation}
where $W^* \sim P_*$, $G_k \sim \normal(0,1)$ and $b_0 > 0$.

The \emph{master theorem} \citep[Theorem 4.2]{feng2022unifying} states that, under some regularity conditions, we have the following for any second-order pseudo-Lipschitz function $\psi\colon \reals \to \reals$, almost surely as $d \to \infty$.
\begin{align}
    \abs*{\frac{1}{d} \sum_{i=1}^d \psi(w^{k}_i,w^*_i) - \E[\psi(\eta(W^* + \sigma_k G),W^*)]} \to 0\,,
\end{align}
where $W^* \sim P_*$ and $G \sim \normal(0,1)$.
In other words, the joint empirical distribution of the entries of $(\bw^k,\bw^*)$ is asymptotically equivalent to the joint distribution of $(\eta(W^* + \sigma_k G),W^*)$.

Furthermore, it can be shown that the AMP iteration \eqref{eq:amp} converges to the minimizer $\hat \bw$ of the minimization problem \eqref{eq:min_problem} \citep[Theorem 1]{rangan2016fixed}. Thus, the empirical distribution of the entries of $\hat \bw$ can be characterized using the fixed point $(b_*,\sigma_*,t_*)$ of the SE recursion:
\begin{align}
    b_* = \frac{1}{\delta} \E[\eta'(W^* + \sigma_* G;t_*)] \,, \quad \sigma_*^2 = \frac{\sigma^2 + \E[(W^* - \eta(W^* + \sigma_* G; t_*))^2]}{\delta} \,, \quad t_* = \lambda + b_* t_* \,. \label{eq:fix_se}
\end{align}

Next, we map the SE fixed point \eqref{eq:fix_se} to each of the DMFT fixed points.

\paragraph{Case i: $\lambda > 0$.}
We show that by choosing $J(w) = \abs{w}$, the SE fixed point \eqref{eq:fix_se} becomes equivalent to the DMFT fixed point \eqref{eq:fix_case1}. When $J(w) = \abs{w}$, the proximal operator $\eta$ is the soft thresholding function $\eta(x;t) = \ST(x;t)$ and the SE fixed point corresponds to the DMFT fixed point with the following mapping.
\begin{align}
    b_* \to \frac{\chi_w}{1 + \chi_w} \,, \quad \sigma_*^2 \to \frac{(1+\chi_w)^2}{\delta^2} \cdot \delta C_f = \frac{C_w + \sigma^2}{\delta} \,, \quad t_* = \frac{\lambda}{1 - b_*} \to (1 + \chi_w) \lambda \,. \label{eq:map_se_dmft}
\end{align}
Thus, the fixed point of the gradient flow for DLNs is asymptotically equivalent to the $\ell_1$ regularized solution. This is natural since $\ell_2$ regularization on $\bu$ and $\bv$ are equivalent to $\ell_1$ regularization on $\bw = (\bu^2-\bv^2)/2$.

\paragraph{Case ii: $\lambda = 0$, $\delta > 1$.}
In this case, the penalty term vanishes since $\lambda = 0$. Then the proximal operator is the identity function $\eta(x;t) = x$ and the SE fixed point \eqref{eq:fix_se} is solved by
\begin{align}
    b_* = \frac{1}{\delta} \,, \quad \sigma_*^2 = \frac{\sigma^2}{\delta-1} \,, \quad t_* = \frac{\delta}{\delta - 1} \lambda \,.
\end{align}
Again, with the same mapping as \eqref{eq:map_se_dmft}, we recover the DMFT fixed point \eqref{eq:fix_case2}.

\paragraph{Case iii: $\lambda = 0$, $\delta < 1$.}
We take $J(w) = w \sinh^{-1}(w / \alpha^2) - \sqrt{w^2 + \alpha^4} + \alpha^2$ and send $\lambda \to 0$. This corresponds to the following constrained minimization problem.
\begin{align}
    \min_{\bw \in \reals^d} \frac{1}{d} \sum_{i=1}^d J(w_i) \quad \text{subject to} \quad \bX \bw = \by \,.
\end{align}
We assume that the $\lambda \to 0$ limit of the SE fixed point characterizes the solution of the above constrained minimization problem (this amounts to assuming that $\lambda \to 0$ and $d \to \infty$ limits commute).

The proximal operator is
\begin{align}
    \eta(u;t_*) = \argmin_{x \in \reals} \ab\{\frac{1}{2}(x-u)^2 + t_* J(x) \} \,,
\end{align}
and $x_* \coloneqq \eta(u;t_*)$ satisfies
\begin{align}
    0 = x_* - u + t_* J'(x_*) = x_* - u + t_* \sinh^{-1}(x_*/\alpha^2) \,.
\end{align}
Then $\eta'(u;t_*) = \partial_u x_*$ satisfies, by the implicit function theorem,
\begin{align}
    0 & = \partial_u x_* - 1 + \frac{t_*}{\sqrt{x_*^2 + \alpha^4}} \partial_u x_* \,.
\end{align}
Thus, the equation for $b_*$ is
\begin{align}
    1 - \frac{\lambda}{t_*} & = \frac{1}{\delta} \E\ab[\frac{1}{1 + t_* / \sqrt{\eta(W^* + \sigma_* G; t_*)^2 + \alpha^4}}] \,.
\end{align}
Taking $\lambda \to 0$, we recover the DMFT fixed point \eqref{eq:fix_case3} with the following mapping.
\begin{align}
    \sigma_*^2 \to \frac{R_w^2}{\delta^2} \cdot \delta \tilde C_f = \frac{C_w + \sigma^2}{\delta} \,, \quad t_* \to R_w \,.
\end{align}

\subsection{Convergence Rate}
\label{app:convergence_rate}

\subsubsection{Regularized Case $\lambda > 0$}
We linearize the dynamics around the fixed point. Let $w(t) = w + \Delta w(t)$ and $g(t) = g + \Delta g(t)$ where $w$ and $g$ are the fixed points \eqref{eq:fix_case1}. We make the following assumptions and approximations.
\begin{itemize}
    \item $\Delta w(t)$ and $\Delta g(t)$ are small, and terms of second or higher order can be ignored.
    \item $w(t)$ converge slower than $\napier^{-\lambda t}$ and hence the $\alpha^2 \napier^{-2 \lambda t}$ term can be ignored compared to $w(t)^2$.
    \item Response functions are TTI.
\end{itemize}

For paths with $w = 0$, we have
\begin{align}
    \diff{}{t} \Delta w(t) & \approx -\abs{\Delta w(t)} (g + \Delta g(t)) - \lambda \Delta w(t) \approx -(\lambda + \sign(\Delta w(t))g) \Delta w(t) = -(\lambda - \abs{g})\Delta w(t) \,,
\end{align}
where in the last equality, we used $\sign(\Delta w(t)) = -\sign(g)$. From the condition $w = 0$, we have
\begin{align}
    \abs{g} = \abs*{\frac{z}{\delta} - \frac{w^*}{1 + \chi_w}} \leq \lambda \,,
\end{align}
and thus $\lambda - \abs{g} \geq 0$. Therefore, paths with $w = 0$ converge to zero as $\abs{w(t)} \sim \napier^{-(\lambda-\abs{g}) t}$. This rate is consistent with the assumption that $w(t)$ converges slower than $\napier^{-\lambda t}$. This rate further implies that, the closer the observation $w^* - (1 + \chi_w) z / \delta$ is to the threshold $(1 + \chi_w)\lambda$, the slower the convergence. Since there are paths with $\lambda - \abs{g}$ arbitrarily small (because of the continuous nature of the noise $z$), the convergence of macroscopic observables (such as training and test errors) is subexponential.

For paths with $w \neq 0$, we have
\begin{align}
    \diff{}{t} \Delta w(t) & \approx -\abs{w + \Delta w(t)} (g + \Delta g(t)) - \lambda (w + \Delta w(t)) \notag                                                \\
                           & \approx -(\sign(w) g + \lambda) (w + \Delta w(t)) - \abs{w} \Delta g(t) \notag                                                     \\
                           & = -\abs{w} \Delta g(t) \notag                                                                                                      \\
                           & = -\abs{w} \ab(\frac{\Delta z(t)}{\delta} + \Delta w(t) - \int_0^t R_f(t-s) \Delta w(s) \de s) \,, \label{eq:dmft_dwdt_linearized}
\end{align}
where in the third line we used $g = -\lambda \sign(w)$. Taking the Laplace transform, we get
\begin{align}
    p \Delta \bar w(p) - w(0) = -\abs{w} \ab(\frac{\Delta \bar z(p)}{\delta} + \Delta \bar w(p) - \bar R_f(p) \Delta \bar w(p)) \,,
\end{align}
and we get
\begin{align}
    \Delta \bar w(p) = \frac{w(0) -\abs{w} \Delta \bar z(p)/\delta}{p + \abs{w} (1 - \bar R_f(p))} \,.
\end{align}
The long-time behavior of $\Delta w(t)$ is controlled by the singularity $p_c$ of $\Delta \bar w(p)$ with the largest real part. It is the point at which $p_c + \abs{w} (1 - \bar R_f(p_c)) = 0$. For small $\abs{w}$, we have $p \approx 0$, and we can approximate $\bar R_f(p) \approx \bar R_f(0) = \chi_f$. Thus, the asymptotic behavior is approximately $\abs{\Delta w(t)} \sim \exp(-\abs{w}(1 - \chi_f) t) = \exp(-\frac{\abs{w}}{1 + \chi_w} t)$. Again, when the magnitude of the observation $w^* - (1+\chi_w)z / \delta$ is closer to the threshold $\lambda$, $\abs{w}$ is small and thus the convergence is slow. There are paths with arbitrarily small $\abs{w}$ and hence the convergence of macroscopic observables is subexponential.

\subsubsection{Unregularized Case $\lambda = 0$}

The linearized dynamics around the fixed point are
\begin{align}
    \diff{}{t} \Delta w(t) & \approx -\sqrt{w^2 + \alpha^4} \Delta g(t) = -\sqrt{w^2 + \alpha^4} \ab(\frac{\Delta z(t)}{\delta} + \Delta w(t) - \int_0^t R_f(t-s) \Delta w(s) \de s) \,.
\end{align}
Taking the Laplace transform, we find that the singularity of $\Delta \bar w(p)$ satisfies $p_c + \sqrt{w^2 + \alpha^4} (1 - \bar R_f(p_c)) = 0$, implying slower convergence for smaller $w$. However, unlike the case with $\lambda > 0$, we can still have $p_c \neq 0$ even for $w = 0$, which implies exponential convergence. We can explicitly determine the rate, as we discuss below.

\paragraph{Convergence Rates of Response Functions.}
We derive the convergence rates of response functions $R_w$ and $R_f$. By TTI, we have
\begin{align}
    \diff{}{\tau} \hat R_w(\tau) & = \sqrt{w^2 + \alpha^4} \ab(-\hat R_w(\tau) + \int_{0}^\tau R_f(\tau-\sigma) \hat R_w(\sigma) \de \sigma) \,, \\
    R_f(\tau)                    & = -\int_{0}^\tau R_w(\tau-\sigma) R_f(\sigma) \de \sigma + R_w(\tau) \,,
\end{align}
where $\hat R_w(t,t') \coloneqq -\difsp{w(t)}{z(t')}$ (without expectations) and assumed TTI for $\hat R_w$ as well.
Taking the Laplace transform, we obtain the following equations.
\begin{align}
    \bar R_w(p) = \frac{1+\bar R_w(p)}{\delta} \E\ab[\frac{\sqrt{w^2 + \alpha^4}}{p(1+\bar R_w(p)) + \sqrt{w^2 + \alpha^4}}] \,, \quad \bar R_f(p) = \frac{\bar R_w(p)}{1 + \bar R_w(p)} \,. \label{eq:barRw}
\end{align}

First, we consider the convergence rate of $R_w$. To this end, we compute the rightmost singularity $p_c$ of $\bar R'_w(p) \coloneqq p \bar R_w(p)$. The factor $p$ is introduced to eliminate the pole at $p=0$, which exists when $\delta < 1$ due to the nonzero fixed point of $R_w(t)$. Note that this factor does not alter the location of other singularities, and the convergence rate is determined from the singularity with the largest negative real part.

By \cref{eq:barRw}, $\bar R'_w(p)$ satisfies
\begin{align}
    \bar R'_w(p) = \frac{p+\bar R'_w(p)}{\delta} \E\ab[\frac{\sqrt{w^2 + \alpha^4}}{p+\bar R'_w(p) + \sqrt{w^2 + \alpha^4}}] \,.
\end{align}
The rightmost singularity $p_c$ can be found by defining
\begin{align}
    h(R',p) & \coloneqq R' - \frac{p + R'}{\delta} \E\ab[\frac{\sqrt{w^2 + \alpha^4}}{p + R' + \sqrt{w^2 + \alpha^4}}],
\end{align}
and solving the following system of equations:
\begin{align}
    h(R',p_c)               & = R' - \frac{p_c + R'}{\delta} \E\ab[\frac{\sqrt{w^2 + \alpha^4}}{p_c+R' + \sqrt{w^2 + \alpha^4}}] = 0 \,, \\
    \partial_{R'} h(R',p_c) & = 1 - \frac{1}{\delta} \E\ab[\ab(\frac{\sqrt{w^2 + \alpha^4}}{p_c+R' + \sqrt{w^2 + \alpha^4}})^2] = 0 \,.
\end{align}
Let $u=R'+p_c$. By the second equation, $u$ satisfies the following equation:
\begin{align}
    1 - \frac{1}{\delta} \E\ab[\ab(\frac{\sqrt{w^2 + \alpha^4}}{u + \sqrt{w^2 + \alpha^4}})^2] & = 0 \,. \label{eq:u}
\end{align}
Let $u_*$ be the solution of the above equation. By solving the equation $h(R',p_c)=0$ for $R'$, we have
\begin{align}
    R' = \frac{u_* A}{\delta} \,, \quad A \coloneqq \E\ab[\frac{\sqrt{w^2 + \alpha^4}}{u_* + \sqrt{w^2 + \alpha^4}}] \,.
\end{align}
We thus have
\begin{align}
    p_c = u_* - R' = \frac{\delta - A}{\delta} u_* \,.
\end{align}
It follows that the convergence rate is $R_w(t) = \exp(-\gamma t + o(1))$ where $\gamma = -p_c$.

Next, we consider the convergence rate of $R_f$. By \cref{eq:barRw}, assuming that $1 + \bar R_w(p)$ is never zero for $p$ with $\Real p \geq \Real p_c$, the rightmost singularity of $\bar R_f$ is simply that of $\bar R_w$, and $R_f$ has the same convergence rate as $R_w$.

In summary, the convergence rates $\gamma$ of $R_w$ and $R_f$ can be obtained by solving the following system.
\begin{equation}
    \boxed{
        \E\ab[\ab(\frac{\sqrt{w^2 + \alpha^4}}{u + \sqrt{w^2 + \alpha^4}})^2] = \delta \,, \quad A = \E\ab[\frac{\sqrt{w^2 + \alpha^4}}{u + \sqrt{w^2 + \alpha^4}}] \,, \quad \gamma = \frac{A - \delta}{\delta} u \,.} \label{eq:convergence_rate_eqn}
\end{equation}

\paragraph{Convergence Rates of Correlation Functions.}
Next, we derive the rates for correlation functions $C_w$ and $C_f$. Since these are bivariate functions, we use the bivariate Laplace transform, defined for $f \colon \reals_{\geq 0}^2 \to \reals$ and $p,q \in \complex$ with sufficiently large real parts as
\begin{align}
    \bar f(p,q) = \int_0^\infty \int_0^\infty f(t,s) \napier^{-pt - qs} \de t \de s \,.
\end{align}
Taking the Laplace transform of the equations for $C_w$ and $C_f$, we obtain
\begin{align}
    \bar C_w(p,q) = \E[(\bar w(p) - w^*/p) (\bar w(q) - w^*/q)] \,, \quad \bar C_f(p,q) = \frac{\bar C_w(p,q) + \sigma^2 / (pq)}{(1 + \bar R_w(p))(1 + \bar R_w(q))} \,.
\end{align}
By the second equation, $pq \bar C_f(p,q)$ has singularities at $p = -\gamma$ and $q = -\gamma$. This implies that $C_f(t,s)$ behaves as $\abs{C_f(t,s) - C_f(\infty,\infty)} = \exp(-\gamma (t+s) + o(1))$ and thus $\abs{L(t)-L(\infty)} = \abs{C_f(t,t) - C_f(\infty,\infty)} = \exp(-2\gamma t + o(1))$. This result is validated against numerical simulations in \cref{fig:convergencerate_details}.

\begin{figure*}[t]
    \begin{center}
        \begin{subfigure}{0.49\textwidth}
            \includegraphics[width=\textwidth]{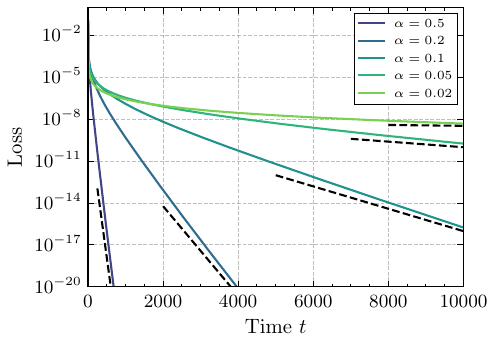}
            \caption{Convergence of the loss for $\delta = 2$.}
        \end{subfigure}
        \begin{subfigure}{0.49\textwidth}
            \includegraphics[width=\textwidth]{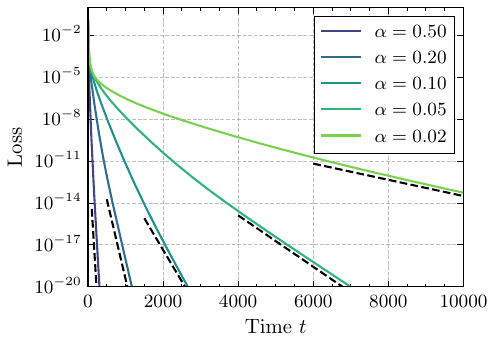}
            \caption{Convergence of the loss for $\delta = 0.5$.}
        \end{subfigure}
        \caption{Convergence of the loss for $\delta = 2 > 1$ and $\delta = 0.5 < 1$. As the initialization scale $\alpha$ decreases, the convergence rate $\gamma$ becomes slower. Simulations are run for $d=500$ and show good agreement with the theoretical rate.}
        \label{fig:convergencerate_details}
    \end{center}
\end{figure*}

\paragraph{Nonnegativity and Monotonicity of the Convergence Rate.}

We prove basic properties of the convergence rate $\gamma$.
\begin{proposition}
    Let $\gamma$ be the solution of \cref{eq:convergence_rate_eqn}. When $\delta \neq 1$, we have $\gamma > 0$ and $\difs{\gamma}{\alpha} > 0$.
\end{proposition}

\begin{proof}
    Let $x(\alpha) = \sqrt{w^2 + \alpha^4}$, $y(u,\alpha) = x / (u + x)$ and define $A, B$ as follows.
    \begin{align}
        A(u,\alpha) = \E[y] \,, \quad B(u,\alpha) = \E[y^2] \,.
    \end{align}
    Since $u=u(\alpha)$ satisfies $B(u,\alpha)=\delta$, we have
    \begin{align}
        \delta \gamma = u(A - \delta) = u(A - B) = u\E[y(1-y)] = u^2\E\ab[\frac{x}{(u+x)^2}]
    \end{align}
    and thus
    \begin{align}
        \gamma = \frac{u^2}{\delta} C \,, \quad C(u,\alpha) \coloneqq \E\ab[\frac{x}{(u+x)^2}] \,.
    \end{align}
    Since $x > 0$ and $u \neq 0$ when $\delta \neq 1$, it follows that $\gamma > 0$.

    Using $\partial_\alpha x = 2\alpha^3 / x$, we have
    \begin{align}
        \partial_u B = -2\E\ab[\frac{x^3}{(u+x)^3}] \,, \quad \partial_\alpha B = 4\alpha^3 u \E\ab[\frac{1}{(u+x)^3}] \,.
    \end{align}
    By the implicit function theorem, we get
    \begin{align}
        \diff{u}{\alpha} = -\frac{\partial_\alpha B}{\partial_u B} = \frac{2 \alpha^3 u \E[(u+x)^{-3}]}{\E[x^2(u+x)^{-3}]} \,.
    \end{align}

    For $C$, we have
    \begin{align}
        \partial_u C = -2\E\ab[\frac{x}{(u+x)^3}] \,, \quad \partial_\alpha C = 2\alpha^3 \E\ab[\frac{1}{x(u+x)^2} - \frac{2}{(u+x)^3}] \,.
    \end{align}

    Thus, we have
    \begin{align}
        \diff{\gamma}{\alpha} & = \frac{1}{\delta} \ab(2u \diff{u}{\alpha} C + u^2 \ab(\partial_u C \diff{u}{\alpha} + \partial_\alpha C)) \notag                                                                   \\
                              & = \frac{1}{\delta} \ab(\diff{u}{\alpha} (2uC + u^2\partial_u C) + u^2\partial_\alpha C) \notag                                                                                      \\
                              & = \frac{1}{\delta} \ab(\frac{2\alpha^3 u\E[(u+x)^{-3}]}{\E[x^2(u+x)^{-3}]} \cdot 2u\E\ab[\frac{x^2}{(u+x)^3}] + 2\alpha^3 u^2 \E\ab[\frac{1}{x(u+x)^2} - \frac{2}{(u+x)^3}]) \notag \\
                              & = \frac{2 \alpha^3 u^2}{\delta} \E\ab[\frac{1}{x(u+x)^2}] \notag                                                                                                                    \\
                              & > 0 \,.
    \end{align}
    In the last line, we used that $u \neq 0$ for $\delta \neq 1$.

\end{proof}

These properties are illustrated in \cref{fig:convergencerate_theory}.

\begin{figure}[t]
    \begin{center}
        \includegraphics[width=0.5\textwidth]{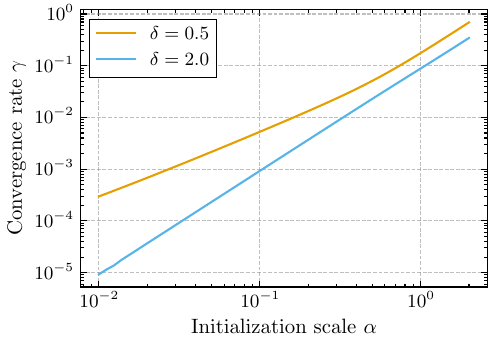}
        \caption{Theoretical convergence rates $\gamma$ for $\delta = 2 > 1$ and $\delta = 0.5 < 1$. $\gamma$ is monotonically increasing with respect to the initialization scale $\alpha$.}
        \label{fig:convergencerate_theory}
    \end{center}
\end{figure}

\paragraph{Limiting Behaviors of the Convergence Rate.}
As $\alpha \to \infty$, the equation for $u$ becomes, up to the leading order,
\begin{align}
    1 - \frac{1}{\delta} \frac{\alpha^4}{(u + \alpha^2)^2} & \approx 0 \,,
\end{align}
which leads to the solution $u_* \approx \alpha^2(\delta^{-1/2} - 1)$, $A \approx \delta^{1/2}$ and $\gamma = \alpha^2(1 - \delta^{-1/2})^2$. Notice that $(1 - \delta^{-1/2})^2$ is the lower end of the support of the Marchenko--Pastur law \eqref{eq:mp_law}, which is the asymptotic minimum eigenvalue of the sample covariance matrix $\delta^{-1} \bX^\transpose \bX$. This is consistent with the fact that the dynamics as $\alpha \to \infty$ is approximately linear (as described in \cref{sec:timescale}), and that the convergence rate of a linear dynamics is governed by the minimum eigenvalue of the coefficients, in this case the sample covariance matrix.

As $\alpha \to 0$, we have $\gamma \to 0$, which implies subexponential decay.
\section{DETAILS OF THE RIGOROUS THEORY}
\label{app:rigorous}

In this section, we develop a rigorous theory for truncated DLNs. We introduce a general class of flows that includes gradient flows on truncated DLNs, characterize its high-dimensional limit using DMFT, and finally specialize to truncated DLNs.

\subsection{General Setup}

\paragraph{General Flow.}
Let $\bX\in\reals^{n\times d}$ and $\bz\in\reals^d$.
Let $k \in \naturals$ and $\ell\colon\reals^k \times \reals \times \reals_{\geq 0} \to \reals^k$, $w\colon \reals^k \times \reals_{\geq 0} \to\reals^k$, $p\colon\reals^k\times \reals^k \times \reals_{\geq 0}\to\reals^k$ be Lipschitz functions.

Consider the following flow for $\btheta\in\reals^{d\times k}$, denoted as $\mathfrak{F}\coloneqq \mathfrak{F}(\btheta(0),\bz,\delta,\ell,w,p)$.
\begin{align}
    \diff{\btheta(t)}{t} = p_t(\bg(t),\btheta(t)) \,, \quad \bg(t) = \frac{1}{\delta} \bX^\transpose \ell_t(\bff(t);\bz) \,, \quad \bff(t) = \bX w_t(\btheta(t)) \,. \label{eq:flow_general}
\end{align}
Here, the functions $\ell_t,p_t,w_t$ are applied row-wise.

This flow is a generalization of the one defined in \citet{celentano2021highdimensional}. Our generalized flow allows row-wise reparameterization of the parameter $\btheta$ through the function $w_t$ and more general post-processing of the gradient $\bg$ through the function $p_t$.

\paragraph{DMFT Equation.}
Given random variables $\theta(0) \in \reals^k$ and $z \in \reals$, we consider the following DMFT equation $\mathfrak{S} \coloneqq \mathfrak{S}(\theta(0),z,\delta,\ell,w,p)$ corresponding to the flow $\mathfrak{F}$, for unknown deterministic functions $\Gamma \colon \reals_{\geq 0} \to \reals$,  $R_w,R_\ell,C_w,C_\ell \colon \reals_{\geq 0}^2 \to \reals$ and stochastic processes $\theta,f,g \colon \reals_{\geq 0} \to \reals$.
\begin{subequations} \label{eq:dmft_general}
    \begin{align}
        \diff{}{t}\theta(t) & = p_t(g(t),\theta(t)) \,,   \label{eq:gf_theta}                                                                                                           \\
        g(t)                & = \frac{u_g(t)}{\delta} + \Gamma(t) w_t(\theta(t)) + \int_0^t R_\ell(t,s) w_s(\theta(s)) \de s \,, &  & u_g \sim \GP(0,\delta C_\ell) \,, \label{eq:gf_g} \\
        f(t)                & = u_f(t) + \int_0^t R_w(t,s)\ell_s(f(s);z) \de s \,,                                               &  & u_f \sim \GP(0,C_w) \,, \label{eq:gf_f}           \\
        R_w(t,s)            & = \E\ab[\diffp{w_t(\theta(t))}{u_g(s)}] \,,                                                        &  & 0 \leq s \leq t\,,                                \\
        R_\ell(t,s)         & = \E\ab[\diffp{\ell_t(f(t);z)}{u_f(s)}] \,,                                                        &  & 0 \leq s \leq t\,,                                \\
        \Gamma(t)           & = \E\ab[\nabla_f\ell_t(f(t);z)] \,,                                                                                                                       \\
        C_w(t,s)            & = \E[w_t(\theta(t)) w_s(\theta(s))^\transpose] \,,                                                                                                        \\
        C_\ell(t,s)         & = \E[\ell_t(f(t),z)\ell_s(f(s),z)^\transpose] \,.
    \end{align}
\end{subequations}
We set $R_w(t,s) = R_\ell(t,s) = 0$ for $t < s$.
The quantities $\partial w_t(\theta(t))/\partial u_g(s)$ and $\partial \ell_t(f(t);z)/\partial u_f(s)$ are stochastic processes defined as follows.
\begin{subequations} \label{eq:dmft_general_aux}
    \begin{align}
        \diffp{w_t(\theta(t))}{u_g(s)}      & = \nabla_\theta w_t(\theta(t)) \diffp{\theta(t)}{u_g(s)} \,,                                                          \label{eq:gf_dw}     \\
        \diff{}{t}\diffp{\theta(t)}{u_g(s)} & = \nabla_g p_t(g(t),\theta(t)) \diffp{g(t)}{u_g(s)} + \nabla_\theta p_t(g(t),\theta(t)) \diffp{\theta(t)}{u_g(s)} \,, \label{eq:gf_dtheta} \\
        \diffp{g(t)}{u_g(s)}                & = \Gamma(t) \diffp{w_t(\theta(t))}{u_g(s)} + \int_s^t R_\ell(t,s') \diffp{w_{s'}(\theta(s'))}{u_g(s)} \de s' \,, \label{eq:gf_dg}          \\
        \diffp{\ell_t(f(t);z)}{u_f(s)}      & = \nabla_f \ell_t(f(t);z) \diffp{f(t)}{u_f(s)} \,, \label{eq:gf_dl}                                                                        \\
        \diffp{f(t)}{u_f(s)}                & = R_w(t,s) \nabla_f \ell_s(f(s);z) + \int_s^t R_w(t,s') \diffp{\ell_{s'}(f(s');z)}{u_f(s)} \de s' \,, \label{eq:gf_df}
    \end{align}
\end{subequations}
with the initial condition given by
\begin{align}
    \diffp{\theta(t)}{u_g(t)} = \frac{1}{\delta} \nabla_g p_t(g(t),\theta(t)) \,.
\end{align}

\subsection{Truncated DLNs}

\paragraph{Setup.}
Let $L \geq 2$ be an integer. We consider $L$-layer truncated diagonal linear networks defined as follows.
\begin{align}
    f_{L,M}(\bx;\bu,\bv) = \bw^\transpose \bx, \quad \bw = \frac{1}{L}(\eta_M(\bu^L)-\eta_M(\bv^L)), \quad \bu,\bv \in \reals^d \,,
\end{align}
where $\eta_M \colon \reals \to \reals$ is a truncation operator satisfying $\eta_M(x) = x$ for $\abs{x} \leq M$ and $\eta_M(x) = 0$ for $\abs{x} \geq M+1$, for $M > 0$. Such a function can be explicitly constructed by using the smooth step function
\begin{align}
    \eta(t) = \displaystyle \begin{dcases} 0 & (t \leq 0) \,, \\ \frac{\napier^{-1/t}}{\napier^{-1/t}+\napier^{-1/(1-t)}} & (0 < t < 1)\,, \\ 1 & (t \geq 1)\,,\end{dcases}
\end{align}
and setting $\eta_M(x) = (1 - \eta(\abs{x} - M)) x$.

We consider a regression task with truncated DLNs. Let $\bX \in \reals^{n \times d}$, $\bw^* \in \reals^d$, and $\bxi \in \reals^n$. We generate labels $\by \in \reals^n$ as $\by = \bX \bw^* + \bxi$. We consider the following loss:
\begin{align}
    L(\bu,\bv) = \frac{1}{2 n} \sum_{\mu=1}^n (y_\mu - f_{L,M}(\bx_\mu;\bu,\bv))^2 + \frac{\lambda}{2 d}(\norm{\bu}_2^2 + \norm{\bv}_2^2) \,,
\end{align}
where $\lambda \geq 0$ is a regularization parameter. The gradient flow for the loss $L$ is
\begin{subequations}
    \begin{align}
        \diff{}{t} \bu(t) & = -\frac{d}{2} \nabla_{\bu} L(\bu(t),\bv(t)) \notag                                                                                   \\
                          & = -\frac{1}{2} \ab(\lambda \bu(t) + \bu(t)^{L-1} \odot \frac{1}{\delta} \bX^\transpose (\bX\bw(t) - \by) \odot \eta'_M(\bu(t)^L)) \,, \\
        \diff{}{t} \bv(t) & = -\frac{d}{2} \nabla_{\bv} L(\bu(t),\bv(t)) \notag                                                                                   \\
                          & = -\frac{1}{2} \ab(\lambda \bv(t) - \bv(t)^{L-1} \odot \frac{1}{\delta} \bX^\transpose (\bX\bw(t) - \by) \odot \eta'_M(\bv(t)^L)) \,,
    \end{align} \label{eq:gf_truncated_dln}
\end{subequations}
for given initial values $\bu(0)$ and $\bv(0)$. If the entries of $\bu(t)$ and $\bv(t)$ stay inside $[-M,M]$, the truncation can be ignored.

\paragraph{DMFT Equation.}
Given random variables $u(0),v(0),w^*,\xi \in \reals$, we consider the following DMFT equation corresponding to the gradient flow \eqref{eq:gf_truncated_dln}.
\begin{subequations}
    \begin{align}
        \diff{}{t} u(t) & = p_u(g(t),u(t),v(t)) \,,                                                                                           \\
        \diff{}{t} v(t) & = p_v(g(t),u(t),v(t)) \,,                                                                                           \\
        w(t)            & = \frac{1}{L} (\eta_M(u(t)^L) - \eta_M(v(t)^L)) \,,                                                                 \\
        g(t)            & = \frac{z_g(t)}{\delta} + w(t) - w^* - \int_0^t R_f(t,s) (w(s) - w^*) \de s \,, &  & z_g \sim \GP(0,\delta C_f) \,, \\
        f(t)            & = z_f(t) - \int_0^t R_w(t,s) (f(s)-\xi) \de s \,,                               &  & z_f \sim \GP\ab(0,C_w) \,,     \\
        R_w(t,s)        & = -\E\ab[\diffp{w(t)}{z_g(s)}] \,,                                                                                  \\
        R_f(t,s)        & = -\E\ab[\diffp{f(t)}{z_f(s)}] \,,                                                                                  \\
        C_w(t,s)        & = \E[(w(t) - w^*)(w(s) - w^*)] \,,                                                                                  \\
        C_f(t,s)        & = \E[(f(t)-\xi)(f(s)-\xi)] \,,
    \end{align} \label{eq:dmft_truncated_dln}
\end{subequations}
where
\begin{align}
    p_u(g,u,v) \coloneqq -\frac{1}{2} (\lambda u + u^{L-1} g \eta'_M(u^L)) \,, \quad p_v(g,u,v) \coloneqq -\frac{1}{2} (- \lambda v + v^{L-1} g \eta'_M(v^L)) \,,
\end{align}
and with auxiliary processes
\begin{subequations}
    \begin{align}
        \diffp{w(t)}{z_g(s)}            & = u(t)^{L-1} \eta'_M(u(t)^L) \diffp{u(t)}{z_g(s)} - v(t)^{L-1} \eta'_M(v(t))^{L-1} \diffp{v(t)}{z_g(s)} \,,                                                             \\
        \diff{}{t} \diffp{u(t)}{z_g(s)} & = \partial_g p_u(g(t),u(t),v(t)) \diffp{g(t)}{z_g(s)} + \partial_u p_u(g(t),u(t),v(t)) \diffp{u(t)}{z_g(s)} + \partial_v p_u(g(t),u(t),v(t))) \diffp{v(t)}{z_g(s)}  \,, \\
        \diff{}{t} \diffp{u(t)}{z_g(s)} & = \partial_g p_v(g(t),u(t),v(t)) \diffp{g(t)}{z_g(s)} + \partial_u p_v(g(t),u(t),v(t)) \diffp{u(t)}{z_g(s)} + \partial_v p_v(g(t),u(t),v(t))) \diffp{v(t)}{z_g(s)}  \,, \\
        \diffp{g(t)}{z_g(s)}            & = \diffp{w(t)}{z_g(s)} - \int_0^t R_f(t,s') \diffp{w(s')}{z_g(s)} \de s' \,,                                                                                            \\
        \diffp{f(t)}{z_f(s)}            & = -\int_0^t R_w(t,s') \diffp{f(s')}{z_f(s)} \de s' + R_w(t,s) \,,
    \end{align}
\end{subequations}
with initial conditions given by
\begin{align}
    \diffp{u(t)}{z_g(t)} = \frac{1}{\delta} \partial_g p_u(g(t),u(t),v(t)) \,, \quad \diffp{v(t)}{z_g(t)} = \frac{1}{\delta} \partial_g p_v(g(t),u(t),v(t)) \,.
\end{align}

If we ignore the truncation $\eta_M$, this equation for $L=2$ matches the equation \eqref{eq:dmft_dln_heuristic} obtained heuristically. After further simplification along the lines of \cref{app:simplify_dmft}, we obtain the DMFT equation \eqref{eq:dmft_dln} presented in the main text.

\subsection{Statements of the Results}

\begin{assumption} \label{ass:main}
    \noindent
    \begin{itemize}
        \item The entries $\bX=(x_{ij})_{i\in[n],j\in[d]}$ are independent and satisfy $\E x_{ij}=0,\E x_{ij}^2=1/d,\norm{x_{ij}}_{\psi_2}\leq C/\sqrt{d}$, where $\norm{\cdot}_{\psi_2}$ is the sub-Gaussian norm.
        \item $n,d\to\infty,\;n/d\to\delta\in(0,\infty).$
        \item $\bz\in\reals^d,\btheta(0)\in\reals^{d\times k}$ is independent of $\bX$, and the empirical distributions $\hat\mu_{\theta(0)}\coloneqq d^{-1}\sum_{i=1}^d \dirac_{\theta_i(0)},\hat\mu_{z}\coloneqq n^{-1}\sum_{i=1}^n \dirac_{z_i}$ converge to $\sfP(\theta(0)),\sfP(z)$, respectively, in $p$-Wasserstein distance for all $p \geq 1$, almost surely as $n,d \to \infty$.
    \end{itemize}
\end{assumption}

\begin{assumption} \label{ass:gf}
    The functions $\ell_t(f;z),w_t(\theta),p_t(g;\theta)$ and their Jacobians $\De\ell,\De w,\De p$ are Lipschitz continuous in $t \in \reals_{\geq 0}$ and $\theta,f,g\in\reals^k$, i.e., there exists a universal constant $M > 0$ such that for all $t_1,t_2\in[0,T]$ and all $\theta_1,\theta_2,f_1,f_2,g_1,g_2\in\reals^k$,
    \begin{align}
        \norm{\ell_{t_1}(f_1;z) - \ell_{t_2}(f_2;z)}_2           & \leq M(\norm{f_1-f_2}_2 + \abs{t_1-t_2})\,,                              \\
        \norm{w_{t_1}(\theta_1) - w_{t_2}(\theta_2)}_2           & \leq M(\norm{\theta_1-\theta_2}_2 + \abs{t_1-t_2})\,,                    \\
        \norm{p_{t_1}(g_1;\theta_1) - p_{t_2}(g_2;\theta_2)}_2   & \leq M(\norm{g_1-g_2}_2 + \norm{\theta_1-\theta_2}_2 + \abs{t_1-t_2})\,, \\
        \norm{\De\ell_{t_1}(f_1;z) - \De\ell_{t_2}(f_2;z)}_2         & \leq M(\norm{f_1-f_2}_2 + \abs{t_1-t_2})\,,                              \\
        \norm{\De w_{t_1}(\theta_1) - \De w_{t_2}(\theta_2)}_2         & \leq M(\norm{\theta_1-\theta_2}_2 + \abs{t_1-t_2})\,,                    \\
        \norm{\De p_{t_1}(g_1;\theta_1) - \De p_{t_2}(g_2;\theta_2)}_2 & \leq M(\norm{g_1-g_2}_2 + \norm{\theta_1-\theta_2}_2 + \abs{t_1-t_2})\,.
    \end{align}
\end{assumption}

The following theorem establishes the existence and uniqueness of the solution of the DMFT equation $\mathfrak{S}$. We give a proof in \cref{app:proof_dmft_sol}.

\begin{theorem} \label{thm:dmft_sol}
    Under \cref{ass:gf}, for any $T>0$ there exists a tuple $(\theta,g,f,R_w,R_\ell,\Gamma,C_w,C_\ell)$ that solves the DMFT system $\mathfrak{S}$. The solution is unique among all such tuples with $(C_w,R_w)$ bounded in all compact sets in $\reals_{\geq 0}^2$.
    Further, there exist functions $\Phi_{R_w},\Phi_{R_\ell},\Phi_{C_w},\Phi_{C_\ell}\colon \reals_{\geq 0} \to \reals_{\geq 0}$ that satisfy
    \begin{align}
        \norm{R_w(t,s)} \leq \Phi_{R_w}(t-s)\,,  \; \norm{R_\ell(t,s)} \leq \Phi_{R_\ell}(t-s)\,, \; \norm{C_w(t,t)} \leq \Phi_{C_w}(t)\,, \; \norm{C_\ell(t,t)} \leq \Phi_{C_\ell}(t)\,,
    \end{align}
    and $\norm{\Gamma(t)} \leq M$, for all $t,s \geq 0$.
    Further, there exists $\lambda>0$ such that
    \begin{align}
        \lim_{t\to\infty} \napier^{-\lambda t} \max\{\Phi_{R_w}(t),\Phi_{R_\ell}(t),\Phi_{C_w}(t),\Phi_{C_\ell}(t)\} & = 0 \,.
    \end{align}
    Finally, the stochastic processes $(\theta(t),g(t),f(t))_{t \in [0,T]}$ have continuous sample paths.
\end{theorem}

The following theorem characterizes the empirical distribution of the flow variable $\btheta(t)$ in $\mathfrak{F}$ as the solution of the DMFT system $\mathfrak{S}$. We give a proof in \cref{app:proof_dmft_flow}.

\begin{theorem} \label{thm:dmft_flow}
    Under \cref{ass:main,ass:gf}, for $T > 0$, let $(\theta(t),g(t))_{t=0}^T$ be the unique stochastic processes that solve the DMFT equation $\mathfrak{S}$ in \cref{thm:dmft_sol}. Then, we have
    \begin{align}
        \frac{1}{d}\sum_{i=1}^d \dirac_{(\theta_i(t))_{t=0}^T} \xrightarrow{W_2} \sfP((\theta(t))_{t=0}^T)\,, \quad \frac{1}{n}\sum_{i=1}^n \dirac_{z_i,(f_i(t))_{t=0}^T} \xrightarrow{W_2} \sfP(z,(f(t))_{t=0}^T) \,,
    \end{align}
    almost surely as $n,d \to \infty$. Here, $\sfP$ denotes the law of the given random variables.
\end{theorem}

As a corollary of \cref{thm:dmft_flow}, we obtain a DMFT characterization of the gradient flow for truncated DLNs. We give a proof in \cref{app:proof_dmft_dln}.

\begin{corollary}\label{cor:dmft_dln}
    Under \cref{ass:main} (with $\bz=\bxi$ and $\btheta(0)=(\bu(0),\bv(0),\bw^*)$), for any $T > 0$, there exists a unique solution to the DMFT equation \eqref{eq:dmft_truncated_dln} in the sense of \cref{thm:dmft_sol}, and we have
    \begin{align}
        \frac{1}{d}\sum_{i=1}^d \dirac_{(u_i(t),v_i(t))_{t=0}^T} \xrightarrow{W_2} \sfP((u(t),v(t))_{t=0}^T)\,, \quad \frac{1}{n}\sum_{i=1}^n \dirac_{\xi_i,(f_i(t))_{t=0}^T} \xrightarrow{W_2} \sfP(\xi,(f(t))_{t=0}^T) \,,
    \end{align}
    almost surely as $n,d \to \infty$.
\end{corollary}

\subsection{Proof of Theorem~\ref{thm:dmft_sol}}
\label{app:proof_dmft_sol}

We follow the approach of the proof of \citet[Theorem 1]{celentano2021highdimensional}. The proof proceeds as follows.
\begin{enumerate}[label=\Roman*.]
    \item Define the functions $\Phi_{R_w},\Phi_{R_\ell},\Phi_{C_w},\Phi_{C_\ell}$.
    \item Define metric spaces $\calS$ and $\bar\calS$ for the triple $(C_\ell,R_\ell,\Gamma)$ and the pair $(C_w,R_w)$, respectively, and show that the stochastic processes $\theta,f,g$ are determined uniquely in these spaces.
    \item Define a mapping $\calT\colon\calS\to\calS$. We show that the solution $(C_\ell,R_\ell,\Gamma)$ of the DMFT system is a fixed point of $\calT$. We then show that the map $\calT$ is a contraction. By Banach's fixed point theorem, it follows that $\calT$ has a unique fixed point, establishing the existence and uniqueness of the solution.
\end{enumerate}

\paragraph{I. Definition of the Functions $\Phi$.}
Since the following quantities are bounded by assumptions, we take $M > 0$ large enough such that
\begin{align}
    \max\ab\{1 + \E\norm{\theta(0)}_2^2, \; \sup_{t \geq 0} \E \norm{\ell_t(0;z)}_2^2, \; \sup_{t \geq 0}\norm{w_t(0)}_2, \; \sup_{t \geq 0}\norm{p_t(0,0)}_2 \} & \leq M\,.
\end{align}

Consider the following system of integro-differential equations.
\begin{align}
    \diff{}{t}\Phi_{R_\theta}(t)         & = 2M^3 \Phi_{R_\theta}(t) + M^2 \int_0^t \Phi_{R_\ell}(t-s') \Phi_{R_\theta}(s'-s) \de s' \,,                                                                                                                                    \label{eq:phi_1} \\
    \Phi_{R_\ell}(t)                     & = M \cdot \ab\{ \Phi_{R_\theta}(t) + \int_0^t \Phi_{R_\theta}(t-s)\Phi_{R_\ell}(s) \de s \} \,,                                                                                                                                                   \\
    \diff{}{t} \sqrt{\Phi_{C_\theta}(t)} & = \sqrt{3} \cdot \sqrt{8M^6\Phi_{C_\theta}(t) + \frac{M^2k}{\delta} \Phi_{C_\ell}(t) + 2M^4\int_0^t (t-s+1)^2\Phi_{R_\ell}(t-s)^2 \Phi_{C_\theta}(t) \de s} \,, \label{eq:phi_3}                                                                  \\
    \Phi_{C_\ell}(t)                     & = 3 \cdot \ab\{M + kM^2 \Phi_{C_\theta}(t) + M^2 \int_0^t (t-s+1)^2 \Phi_{R_\theta}(t-s)^2 \Phi_{C_\ell}(s) \de s \} \,.
\end{align}
By \citet[Lemma 5.1]{celentano2021highdimensional}, there exists a nondecreasing solution to the above system, given $\Phi_{R_\theta}(0) > 0, \Phi_{C_\theta}(0) > 0$. In addition, there exists $\lambda> 0$ such that the following holds.
\begin{align}
    \lim_{t\to\infty} \napier^{-\lambda t} \max\{\Phi_{R_\theta}(t),\Phi_{R_\ell}(t),\Phi_{C_\theta}(t),\Phi_{C_\ell}(t)\} & = 0 \,.
\end{align}
Using this solution, we define
\begin{align}
    \Phi_{R_w}(t) \coloneqq M\Phi_{R_\theta}(t) \,, \quad \Phi_{C_w}(t) \coloneqq 2M^2\Phi_{C_\theta}(t) \,.
\end{align}

\paragraph{II. Definition of the Spaces $\calS,\bar\calS$.}
We define the following function spaces.

\begin{definition}[Function triplet spaces $\calS$ and $\calS_{\mathrm{cont}}$] \label{def:space-S}
    We define the space $\calS \coloneqq \calS(\Phi_{R_w},\Phi_{R_\ell},\Phi_{C_w},\Phi_{C_\ell},M_{\calS},T)$ of triples $(C_\ell,R_\ell,\Gamma)$ that satisfy the following.
    \begin{itemize}
        \item $C_\ell(t,s)$ is a covariance kernel and satisfies $\norm{C_\ell(t,t)} \leq \Phi_{C_\ell}(t)$ for $t \in [0,T]$ and
              \begin{align}
                  C_\ell(0,0) = \E[\ell_0(f(0);z) \ell_0(f(0);z)^\transpose] \,, \quad f(0) \sim \normal(0, \E[\theta(0)\theta(0)^\transpose]) \,.
              \end{align}
              Further, $C_\ell(t,s)$ is continuous for $s,t \in [0,T] \setminus P$ where $P$ is a finite set, and for any $s \leq t$ such that $C_\ell$ is continuous in $[s,t]^2$,
              \begin{align}
                  \norm{C_\ell(t,t) - 2 C_\ell(t,s) + C_\ell(s,s)} \leq M_\calS (t-s)^2 \,.
              \end{align}
        \item $R_\ell(t,s)$ is measurable, $R_\ell(t,s) = 0$ for $t \leq s$, and $\norm{R_\ell(t,s)} \leq \Phi_{R_\ell}(t-s)$ for $0 \leq s \leq t \leq T$.
        \item $\Gamma(t)$ is measurable, $\norm{\Gamma(t)} \leq M$ for $t \in [0,T]$, and
              \begin{align}
                  \Gamma(0) = \E[\nabla_f \ell_0(f(0);z)] \,, \quad f(0) \sim \normal(0, \E[\theta(0)\theta(0)^\transpose]) \,.
              \end{align}
    \end{itemize}
    We define the space $\calS_{\mathrm{cont}} \subset \calS$ of all $(C_\ell,R_\ell,\Gamma)$ such that $C_\ell$ is continuous (i.e., $P= \emptyset$) and for all $s, s' \leq t$,
    \begin{align}
        \norm{C_\ell(t,s) - C_\ell(t,s')} & \leq \sqrt{\Phi_{C_\ell}(T) M_\calS } \cdot \abs{s-s'} \,.
    \end{align}
\end{definition}

\begin{definition}[Function pair spaces $\bar \calS$ and $\bar \calS_{\mathrm{cont}}$]
    We define the space $\bar \calS \coloneqq \bar \calS(\Phi_{R_w},\Phi_{R_\ell},\Phi_{C_w},\Phi_{C_\ell},M_{\bar \calS},T)$ of pairs $(C_w,R_w)$ that satisfy the following.
    \begin{itemize}
        \item $C_w(t,s)$ is a covariance kernel and satisfies $\norm{C_w(t,t)} \leq \Phi_{C_w}(t)$ for $t \in [0,T]$ and
              \begin{align}
                  C_w(0,0) = \E[w_0(\theta(0)) w_0(\theta(0))^\transpose] \,.
              \end{align}
              Further, $C_w(t,s)$ is continuous for all $s, t \in [0,T] \setminus P$ where $P$ is a finite set, and for any $s \leq t$ such that $C_w(t,s)$ is continuous in
              $[s,t]^2$,
              \begin{align}
                  \norm{C_w(t,t) - 2C_w(t,s) + C_w(s,s)} \leq M_{\bar S} (t-s)^2 \,.
              \end{align}
        \item $R_w(t,s)$ is measurable, $R_w(t,s) = 0$ for $t < s$, and $\norm{R_w(t,s)} \leq \Phi_{R_w}(t-s)$ for $0 \leq s \leq t \leq T$.
    \end{itemize}
    We define the space $\bar \calS_{\mathrm{cont}} \subset \bar \calS$ of all $(C_w,R_w)$ such that $C_w$ is continuous (i.e., $P = \emptyset$) and for all $s, s' \leq t$,
    \begin{align}
        \norm{C_w(t,s) - C_w(t,s')} & \leq  \sqrt{\Phi_{C_w}(T) M_{\bar S} } \cdot \abs{s-s'} \,.
    \end{align}
\end{definition}

The constants $M_\calS,M_{\bar\calS}$ are chosen in the proof of \cref{lem:mapping}.

For given $(C_\ell,R_\ell,\Gamma)\in\calS$ and $(C_w,R_w)\in\bar\calS$, stochastic processes $\theta(t),f(t),\difsp{w_t(\theta(t))}{u_g(s)},\difsp{\ell_t(f(t);z)}{u_f(s)}$ are uniquely determined by \cref{eq:dmft_general,eq:dmft_general_aux} \citep[Lemma 5.4]{celentano2021highdimensional}.

\paragraph{III. Definition of the Map $\calT$.}
We define $\calT=\calT_{\bar\calS\to\calS} \circ \calT_{\calS\to\bar \calS}$, where $\calT_{\calS\to\bar \calS}\colon\calS\to\bar \calS,\calT_{\bar \calS\to\calS}\colon\bar \calS \to \calS$ are defined in the following.

We define $\calT_{\calS\to\bar \calS}\colon(C_\ell,R_\ell,\Gamma) \mapsto (\bar C_w,\bar R_w)$ as follows.
For a given $(C_\ell,R_\ell,\Gamma) \in \bar \calS$, take the unique stochastic processes $\theta(t),\difsp{w_t(\theta(t))}{u_g(s)}$ satisfying \cref{eq:gf_theta,eq:gf_g,eq:gf_dw,eq:gf_dtheta,eq:gf_dg}, and define
\begin{align}
    \bar C_w(t,s) = \E[w_t(\theta(t)) w_s(\theta(s))^\transpose] \,, \quad \bar R_w(t,s) = \E\ab[\diffp{w_t(\theta(t))}{u_g(s)}] \,.
\end{align}
Similarly, we define $\calT_{\bar \calS\to \calS}\colon(\bar C_w,\bar R_w) \mapsto (\bar C_\ell,\bar R_\ell,\bar \Gamma)$ as follows.
For a given $(\bar C_w,\bar R_w) \in \bar \calS$, take the unique stochastic processes $f(t),\difsp{\ell_t(f(t);z)}{u_f(s)}$ satisfying the equations \cref{eq:gf_f,eq:gf_dl,eq:gf_df}, and define
\begin{align}
    \bar C_\ell(t,s) = \E[\ell_t(f(t);z)\ell_s(f(s);z)^\transpose] \,, \quad \bar R_\ell(t,s) = \E\ab[\diffp{\ell_t(f(t);z)}{u_f(s)}] \,,  \quad \bar \Gamma(t) = \E\ab[\nabla_f\ell_t(f(t);z)] \,.
\end{align}

In the following lemma, we show that these mappings indeed map into $\bar \calS$ and $\calS$, respectively.

\begin{lemma} \label{lem:mapping}
    In addition to the assumptions for \cref{thm:dmft_sol}, assume $\Phi_{C_\theta}(0) > M$ and $\Phi_{R_\theta}(0) > M/\delta$. Then, $\calT_{\calS\to\bar \calS}$ maps $\calS$ to $\bar\calS_\mathrm{cont}\subset \bar \calS$, and $\calT_{\bar \calS \to \calS}$ maps $\bar\calS$ to $\calS_\mathrm{cont}\subset \calS$.
\end{lemma}

Once this lemma is established, it remains to show that $\calT$ is a contraction. The rest of the proof is essentially identical to \citet[Section 5.4]{celentano2021highdimensional} and is omitted.

\begin{proof}[Proof of \cref{lem:mapping}]

    The proof proceeds along the same lines as the proof of \citet[Lemma 5.5]{celentano2021highdimensional}. We need to modify the norm evaluations to account for additional processing with functions $w_t$ and $p_t$.

    Note that the map $\calT_{\bar \calS \to \calS}$ defined above is identical to the one defined in \citet[Section 5.4]{celentano2021highdimensional}, up to rescaling. Thus, it follows immediately that $\calT_{\bar \calS \to \calS}$ maps $\bar \calS$ into $\calS_\mathrm{cont}$

    It remains to show that $\calT_{\calS \to \bar \calS}$ maps $\calS$ into $\bar\calS_\mathrm{cont}$.
    By the Lipschitz continuity of $w$ and $p$, we have
    \begin{align}
        \norm{w_t(\theta(t))}      & \leq \norm{w_t(0)}_2 + M\norm{\theta(t)}_2 \leq M(1+\norm{\theta(t)}_2) \,,                                       \\
        \norm{p_t(g(t),\theta(t))} & \leq \norm{p_t(0,0)}_2 + M(\norm{g(t)}_2 + \norm{\theta(t)}_2) \leq M(1 + \norm{g(t)}_2 + \norm{\theta(t)}_2) \,.
    \end{align}

    We first show that $\norm{\bar C_w(t,t)} \leq \Phi_{C_w}(t)$.
    By $\mathfrak{S}$, we have
    \begin{align}
        \diff{}{t} \norm{\theta(t)}_2 & \leq \norm{p_t(g(t),\theta(t))} \leq M(1 + \norm{g(t)}_2 + \norm{\theta(t)}_2)\,,                                                                   \\
        \norm{g(t)}_2                 & \leq \frac{1}{\delta} \norm{u_g(t)}_2 + \norm{\Gamma(t)} \norm{w_t(\theta(t))}_2 + \int_0^t \norm{R_\ell(t,s)} \norm{w_s(\theta(s))}_2 \de s \notag \\
                                      & \leq \frac{1}{\delta} \norm{u_g(t)}_2 + M^2(1 + \norm{\theta(t)}_2) + M\int_0^t \Phi_{R_\ell}(t-s) (1+\norm{\theta(s)}_2) \de s \,.
    \end{align}
    Combining, we get
    \begin{align}
        \diff{}{t} \norm{\theta(t)}_2 & \leq 2M^3(1 + \norm{\theta(t)}_2) + \frac{M}{\delta} \norm{u_g(t)}_2 + M^2\int_0^t \Phi_{R_\ell}(t-s) (1+\norm{\theta(s)}_2) \de s \,.
    \end{align}
    Furthermore, we have
    \begin{align}
         & \diff{}{t} \sqrt{1 + \E\norm{\theta(t)}_2^2} = \frac{\E\ab[\norm{\theta(t)}_2 \diff{}{t}\norm{\theta(t)}_2]}{\sqrt{1 + \E\norm{\theta(t)}_2^2}} \stackrel{\text{(i)}}{\leq} \sqrt{\E\ab[\ab(\diff{}{t}\norm{\theta(t)}_2)^2]} \notag \\
         & \leq \sqrt{\E\ab[\ab(2M^3(1+\norm{\theta(t)}_2) + \frac{M}{\delta} \norm{u_g(t)}_2 + M^2\int_0^t \Phi_{R_\ell}(t-s) (1+\norm{\theta(s)}_2) \de s)^2]}                                                                         \notag \\
         & \stackrel{\text{(ii)}}{\leq} \sqrt{\E\ab[\ab(4M^6 (1 + \norm{\theta(t)}_2)^2 + \frac{M^2}{\delta^2} \norm{u_g(t)}_2^2 + M^4\int_0^t (t-s+1)^2\Phi_{R_\ell}(t-s)^2 (1+\norm{\theta(s)}_2)^2 \de s)]} \notag                           \\
         & \qquad \cdot \sqrt{\ab(1 + 1 + \int_0^t (t-s+1)^{-2} \de s)}                                                                                                                                                       \notag            \\
         & \stackrel{\text{(iii)}}{\leq} \sqrt{3} \cdot \sqrt{8M^6 (1 + \E\norm{\theta(t)}_2^2) + \frac{M^2k}{\delta} \Phi_{C_\ell}(t) + 2M^4\int_0^t (t-s+1)^2\Phi_{R_\ell}(t-s)^2 (1+\E\norm{\theta(s)}_2^2) \de s} \,,
    \end{align}
    where in (i) and (ii) we used the Cauchy-Schwarz inequality, and in (iii) we used the following inequality.
    \begin{align}
        \E\norm{u_g(t)}_2^2 = \tr(\E[u_g(t)u_g(t)^\transpose]) \leq k \norm{\E[u_g(t)u_g(t)^\transpose]} = k\delta \norm{C_\ell(t,t)} \leq k\delta \Phi_{C_\ell}(t) \,.
    \end{align}
    By $\Phi_{C_\theta}(0) > M > 1 + \E\norm{\theta(0)}_2^2$ and \cref{eq:phi_3}, $1 + \E\norm{\theta(t)}_2^2 < \Phi_{C_\theta}(t)$ holds for $t \geq 0$. Thus, we have
    \begin{multline}
        \norm{\bar C_w(t,t)} = \norm{\E[w_t(\theta(t))w_t(\theta(t))^\transpose]} \leq \E\norm{w_t(\theta(t))}_2^2 \leq 2M^2(1 + \E\norm{\theta(t)}_2^2) < 2M^2\Phi_{C_\theta}(t) = \Phi_{C_w}(t) \,.
    \end{multline}

    Next, we show that $\norm{\bar R_w(t,s)} \leq \Phi_{R_w}(t-s)$.
    By $\mathfrak{S}$, we have
    \begin{align}
        \diff{}{t} \norm*{\diffp{\theta(t)}{u_g(s)}} & \leq \norm{\nabla_g p_t(g(t),\theta(t))} \norm*{\diffp{g(t)}{u_g(s)}} + \norm{\nabla_\theta p_t(g(t),\theta(t))} \norm*{\diffp{\theta(t)}{u_g(s)}} \notag                                                          \\
                                                     & \leq M \norm*{\diffp{\theta(t)}{u_g(s)}} + M \norm*{\diffp{g(t)}{u_g(s)}} \,,                                                                                                                                      \\
        \norm*{\diffp{g(t)}{u_g(s)}}                 & \leq \norm{\Gamma(t)} \norm{\nabla_\theta w_t(\theta(t))} \norm*{\diffp{\theta(t)}{u_g(s)}} + \int_0^t \norm{R_\ell(t,u)} \norm{\nabla_\theta w_{s'}(\theta(s'))} \norm*{\diffp{\theta(s')}{u_g(s)}} \de s' \notag \\
                                                     & \leq M^2 \norm*{\diffp{\theta(t)}{u_g(s)}} + M\int_0^t \Phi_{R_\ell}(t-s') \norm*{\diffp{\theta(s')}{u_g(s)}} \de s' \,.
    \end{align}
    Thus, we get
    \begin{align}
        \diff{}{t} \E \norm*{\diffp{\theta(t)}{u_g(s)}} & \leq 2M^3 \E \norm*{\diffp{\theta(t)}{u_g(s)}} + M^2 \int_0^t \Phi_{R_\ell}(t-s') \E \norm*{\diffp{\theta(s')}{u_g(s)}} \de s' \,.
    \end{align}
    By $\norm*{\diffp{\theta(t)}{u_g(t)}} \leq M/\delta < \Phi_{R_\theta}(0)$ and \cref{eq:phi_1}, $\E \norm*{\diffp{\theta(t)}{u_g(s)}} < \Phi_{R_\theta}(t-s)$ holds for $t \geq s$. Thus, we have
    \begin{multline}
        \norm{\bar R_w(t,s)} = \norm*{\E\ab[\diffp{w_t(\theta(t))}{u_g(s)}]} \leq \E\ab[\norm*{\diffp{w_t(\theta(t))}{u_g(s)}}] \leq \E\ab[\norm{\nabla_\theta w_t(\theta(t))}\norm*{\diffp{\theta(t)}{u_g(s)}}] \\
        \leq M \E \norm*{\diffp{\theta(t)}{u_g(s)}} < M \Phi_{R_\theta}(t-s) = \Phi_{R_w}(t-s) \,.
    \end{multline}

    Next, we show that $\norm{\bar C_w(t,t)-2 \bar C_w(t,s)+\bar C_w(s,s)} \leq M_{\bar \calS}(t-s)^2$. We have
    \begin{align}
        \norm{\bar C_w(t,t)-2\bar C_w(t,s)+\bar C_w(s,s)} \leq \E\norm{w_t(\theta(t)) - w_s(\theta(s))}_2^2 \leq 2M((t-s)^2 + \E\norm{\theta(t) - \theta(s)}_2^2) \,,
    \end{align}
    and
    \begin{align}
        \E\norm{\theta(t) - \theta(s)}_2^2 & = \E\norm*{\int_s^t \diff{}{t'} \theta(t') \de t'}_2^2 \leq (t-s)^2 \sup_{0\leq t \leq T} \E\norm*{\diff{}{t} \theta(t)}_2^2   \notag                                                   \\
                                           & \leq (t-s)^2 \sup_{0\leq t \leq T} 3 \ab(8M^6\Phi_{C_\theta}(t) + \frac{M^2 k}{\delta} \Phi_{C_\ell}(t) + 2M^4 \int_0^t (t-s+1)^2 \Phi_{R_\ell}(t-s)^2 \Phi_{C_\theta}(s) \de s) \notag \\
                                           & \leq (t-s)^2 \cdot 3 \ab(8M^6\Phi_{C_\theta}(T) + \frac{M^2 k}{\delta} \Phi_{C_\ell}(T) + 2M^4 T (T+1)^2 \Phi_{R_\ell}(T)^2 \Phi_{C_\theta}(T))                               \notag    \\
                                           & \eqqcolon A (t-s)^2 \,.
    \end{align}
    Setting $M_{\bar \calS} \coloneqq 2M(1+A)$, we obtain $\norm{\bar C_w(t,t)-2\bar C_w(t,s)+\bar C_w(s,s)} \leq M_{\bar \calS}(t-s)^2$.

    Finally, by the Cauchy-Schwarz inequality, we have
    \begin{align}
        \norm{\bar C_w(t,s)-\bar C_w(t,s')} \leq \sqrt{\E\norm{w_t(\theta(t))}_2^2 \cdot \E\norm{w_s(\theta(s))-w_{s'}(\theta(s'))}_2^2} \leq \sqrt{\Phi_{C_w(T)} M_{\bar \calS}} \cdot  \abs{s-s'} \,.
    \end{align}
    Thus, we have $(\bar C_w,\bar R_w) \in \bar\calS_\mathrm{cont}$.

\end{proof}

\subsection{Proof of Theorem~\ref{thm:dmft_flow}}
\label{app:proof_dmft_flow}

Again, we follow the approach of the proof of \citet[Theorem 2]{celentano2021highdimensional}. The proof proceeds as follows.
\begin{enumerate}[label=\Roman*.]
    \item Discretize the flow $\mathfrak{F}$ with step size $\eta > 0$.
    \item Map the discretized flow $\mathfrak{F}^\eta$ to an approximate message passing (AMP) iteration. Show that the state evolution of the AMP iteration is equivalent to the discretized version of the DMFT equation, $\mathfrak{S}^\eta$.
    \item Show that as $\eta \to 0$, the solution of the discretized DMFT equation $\mathfrak{S}^\eta$ converges to the unique solution of the DMFT equation $\mathfrak{S}$.
\end{enumerate}

First, we define the discretized flow $\mathfrak{F}^\eta$ as follows. For $i=0,1,\dots$, set $t_i \coloneqq i\eta$ and define
\begin{align}
    \frac{\btheta^\eta(t_{i+1}) - \btheta^\eta(t_i)}{\eta} = p_{t_i}(\bg^\eta(t_i),\btheta^\eta(t_i)) \,, \quad \bg^\eta(t_i) = \frac{1}{\delta} \bX^\transpose \ell_{t_i}(\bff^\eta(t_i);\bz) \,, \quad \bff^\eta(t_i) = \bX w_{t_i}(\btheta^\eta(t_i)) \,,
\end{align}
with $\btheta^\eta(0) = \btheta(0)$. We extend this flow to continuous time by piecewise linear interpolation. We denote by

Next, we define the discretized DMFT equation $\mathfrak{S}^\eta$ as follows.
\begin{subequations}
    \begin{align}
        \frac{\theta^\eta(t_{i+1}) - \theta^\eta(t_i)}{\eta} & = p_{t_i}(g^\eta(t_i),\theta^\eta(t_i)) \,,                                                                                                                                                         \\
        g^\eta(t_i)                                          & = \frac{u_g^\eta(t_i)}{\delta} + \Gamma^\eta(t_i) w_{t_i}(\theta^\eta(t_i)) + \eta\sum_{j=0}^{i-1} R_\ell^\eta(t_i,t_j) w_{t_j}(\theta^\eta(t_j))\,, &  & u_g \sim \normal(0,\delta C_\ell^\eta)\,, \\
        f^\eta(t_i)                                          & = u_f^\eta(t_i) + \eta \sum_{j=0}^{i-1} R_w^\eta(t_i,t_j)\ell_{t_j}(f^\eta(t_j);z) \,,                                                               &  & u_f \sim \normal(0,C_w^\eta) \,,          \\
        R_w^\eta(t_i,t_j)                                    & = \eta^{-1} \E\ab[\diffp{w_{t_i}(\theta^\eta(t_i))}{u_g^\eta(t_j)}] \,,                                                                              &  & 0 \leq j < i \,,                          \\
        R_\ell^\eta(t_i,t_j)                                 & = \eta^{-1} \E\ab[\diffp{\ell_{t_i}(f^\eta(t_i);z)}{u_f^\eta(t_j)}] \,,                                                                              &  & 0 \leq j < i \,,                          \\
        \Gamma^\eta(t_i)                                     & = \E\ab[\nabla_f\ell_{t_i}(f^\eta(t_i);z)] \,,                                                                                                                                                      \\
        C_w^\eta(t_i,t_j)                                    & = \E[w_{t_i}(\theta^\eta(t_i)) w_{t_j}(\theta^\eta(t_j))^\transpose] \,,                                                                                                                            \\
        C_\ell^\eta(t_i,t_j)                                 & = \E[\ell_{t_i}(f^\eta(t_i);z)\ell_{t_j}(f^\eta(t_j);z)^\transpose] \,.
    \end{align}
\end{subequations}
We set $R_w^\eta(t_i,t_j) = R_\ell^\eta(t_i,t_j) = 0$ for $i \leq j$. The quantities $\partial w_{t_i}(\theta^\eta(t_i))/\partial u_g^\eta(t_j)$ and $\partial \ell_{t_i}(f^\eta(t_i);z)/\partial u_f^\eta(t_j)$ are stochastic processes defined as follows.
\begin{subequations}
    \begin{align}
        \diffp{w_{t_i}(\theta^\eta(t_i))}{u_g^\eta(t_j)}                                                         & = \nabla_\theta w_{t_i}(\theta^\eta(t_i)) \diffp{\theta^\eta(t_i)}{u_g^\eta(t_j)} \,,                                                                                                 \\
        \frac{1}{\eta}\ab(\diffp{\theta^\eta(t_{i+1})}{u_g^\eta(t_j)} - \diffp{\theta^\eta(t_i)}{u_g^\eta(t_j)}) & = \nabla_g p_{t_i}(g^\eta(t_i),\theta^\eta(t_i)) \diffp{g^\eta(t_i)}{u_g^\eta(t_j)} + \nabla_\theta p_{t_i}(g^\eta(t_i),\theta^\eta(t_i)) \diffp{\theta^\eta(t_i)}{u_g^\eta(t_j)} \,, \\
        \diffp{g^\eta(t_i)}{u_g^\eta(t_j)}                                                                       & = \Gamma^\eta(t_i) \diffp{w_{t_i}(\theta^\eta(t_i))}{u_g^\eta(t_j)} + \eta\sum_{j'=j+1}^{i-1} R_\ell^\eta(t_i,t_{j'}) \diffp{w_{t_{j'}}(\theta^\eta(t_{j'}))}{u_g^\eta(t_j)} \,,      \\
        \diffp{\ell_{t_i}(f^\eta(t_i);z)}{u_f^\eta(t_j)}                                                         & = \nabla_f \ell_{t_i}(f^\eta(t_i);z) \diffp{f^\eta(t_i)}{u_f^\eta(t_j)} \,,                                                                                                           \\
        \diffp{f^\eta(t_i)}{u_f^\eta(t_j)}                                                                       & = \eta R_w^\eta(t_i,t_j) \nabla_f \ell_j(f_j;z) + \eta \sum_{j'=j+1}^{i-1} R_w^\eta(t_i,t_{j'}) \diffp{\ell_{t_{j'}}(f^\eta(t_{j'});z)}{u_f^\eta(t_j)} \,,
    \end{align}
\end{subequations}
with the initial value
\begin{align}
    \diffp{\theta^\eta(t_{i+1})}{u_g^\eta(t_i)} = \frac{\eta}{\delta} \nabla_g p_{t_i}(g^\eta(t_i),\theta^\eta(t_i)) \,.
\end{align}
Similarly to $\mathfrak{F}^\eta$, we extend this to continuous time by piecewise linear interpolation.

The three parts of the proof correspond to the following three lemmas.

\begin{lemma} \label{lem:flow_disc}
    Under the assumptions of \cref{thm:dmft_flow}, for any $\tau_1,\dots,\tau_m \in [0,T]$, we have, almost surely,
    \begin{align}
        \lim_{\eta \to 0} \limsup_{n,d\to\infty} W_2\ab(\frac{1}{d} \sum_{i=1}^d \dirac_{\theta_i(\tau_1),\dots,\theta_i(\tau_m)},\frac{1}{d} \sum_{i=1}^d \dirac_{\theta^\eta_i(\tau_1),\dots,\theta^\eta_i(\tau_m)}) & = 0 \,, \\
        \lim_{\eta \to 0} \limsup_{n,d\to\infty} W_2\ab(\frac{1}{n} \sum_{i=1}^n \dirac_{f_i(\tau_1),\dots,f_i(\tau_m),z_i},\frac{1}{n} \sum_{i=1}^n \dirac_{f^\eta_i(\tau_1),\dots,f^\eta_i(\tau_m),z_i})             & = 0 \,.
    \end{align}
\end{lemma}

\begin{lemma} \label{lem:dmft_flow_disc}
    Under the assumptions of \cref{thm:dmft_flow}, let $(\theta^\eta(t),f^\eta(t))_{t=0}^T$ be the unique solution of the discretized DMFT equation $\mathfrak{S}^\eta$. For any $\tau_1,\dots,\tau_m \in [0,T]$, we have, almost surely,
    \begin{align}
        \limsup_{n,d\to\infty} W_2\ab(\frac{1}{d} \sum_{i=1}^d \dirac_{\theta^\eta_i(\tau_1),\dots,\theta^\eta_i(\tau_m)},\sfP(\theta^\eta(\tau_1),\dots,\theta^\eta(\tau_m))) & = 0 \,, \\
        \limsup_{n,d\to\infty} W_2\ab(\frac{1}{n} \sum_{i=1}^n \dirac_{f^\eta_i(\tau_1),\dots,f^\eta_i(\tau_m),z_i},\sfP(f^\eta(\tau_1),\dots,f^\eta(\tau_m),z))               & = 0 \,.
    \end{align}
\end{lemma}

\begin{lemma} \label{lem:dmft_disc}
    Under the assumptions of \cref{thm:dmft_flow}, let $(\theta^\eta(t),f^\eta(t))_{t=0}^T$ and $(\theta(t),f(t))_{t=0}^T$ be the unique solution of $\mathfrak{S}^\eta$ and $\mathfrak{S}$, respectively. For any $\tau_1,\dots,\tau_m \in [0,T]$, we have,
    \begin{align}
        \lim_{\eta \to 0} W_2\ab(\sfP(\theta^\eta(\tau_1),\dots,\theta^\eta(\tau_m)),\sfP(\theta(\tau_1),\dots,\theta(\tau_m))) & = 0 \,, \\
        \lim_{\eta \to 0} W_2\ab(\sfP(f^\eta(\tau_1),\dots,f^\eta(\tau_m),z),\sfP(f(\tau_1),\dots,f(\tau_m),z))                 & = 0 \,.
    \end{align}
\end{lemma}

Proofs of \cref{lem:flow_disc,lem:dmft_disc} are essentially identical to those of \citet[Lemmas 6.1 and 6.3]{celentano2021highdimensional} and thus omitted.

Here, we prove \cref{lem:dmft_flow_disc}. We prove a slightly stronger convergence result (almost sure 2-Wasserstein convergence) than \citet[Lemma 6.2]{celentano2021highdimensional} (weak convergence in probability) by utilizing a recent universality result \citep{wang2024universality}.

\begin{proof}[Proof of \cref{lem:dmft_flow_disc}]

    In the following, we omit the superscript $\eta$ for simplicity, as we only consider the discretized systems $\mathfrak{F}^\eta$ and $\mathfrak{S}^\eta$.

    \item
    \paragraph{Reduction to AMP.}
    We consider the following AMP iteration. For a sequence of Lipschitz functions $F_i\colon\reals^{k(i+1)+1}\to\reals^m$, $G_i\colon\reals^{k(i+1)}\to\reals^k\;(i=0,1,\dots)$, generate a sequence of matrices $\ba_{i+1}\in\reals^{d \times k},\bb_i \in\reals^{n\times k}\;(i\geq 0)$ as follows.
    \begin{align}
        \ba_{i+1} & = \bX^\transpose F_i(\bb_0,\dots,\bb_i;\bz) - \delta \sum_{j=0}^i G_j(\ba_1,\dots,\ba_j;\btheta(0)) \xi_{i,j}\,, \\
        \bb_i     & = \bX G_i(\ba_1,\dots,\ba_i;\btheta(0)) - \sum_{j=0}^{i-1} F_j(\bb_0,\dots,\bb_j;\bz) \zeta_{i,j}\,,
    \end{align}
    with the initial value $G_0(\btheta(0))=\btheta(0)$. Here, $F_i,G_i$ are applied row-wise. The matrices $\{\xi_{i,j}\}_{0\leq j\leq i},\{\zeta_{i,j}\}_{0\leq j\leq i-1} \subset\reals^{k\times k}$ are defined by
    \begin{align}
        \zeta_{i,j} & = \E\ab[\diffp{}{\bar u_g(j+1)} G_i(\bar u_g(1),\dots,\bar u_g(i);\theta(0))] \,, &  & 0 \leq j < i \,,    \\
        \xi_{i,j}   & = \E\ab[\diffp{}{\bar u_f(j)} F_i(\bar u_f(0),\dots,\bar u_f(i);z)] \,,           &  & 0 \leq j \leq i \,,
    \end{align}
    where $(\bar u_g(i+1),\bar u_f(i))_{i\geq 0}$ is a sequence of centered Gaussian vectors in $\reals^k$ with covariance
    \begin{align}
        \E[\bar u_f(i)\bar u_f(j)^\transpose]     & = \E[G_i(\bar u_g(1),\dots,\bar u_g(i);\theta(0))G_j(\bar u_g(1),\dots,\bar u_g(j);\theta(0))^\transpose]\,, &  & 0 \leq j \leq i\,, \\
        \E[\bar u_g(i+1)\bar u_g(j+1)^\transpose] & = \delta \E[F_i(\bar u_f(0),\dots,\bar u_f(i);z)F_j(\bar u_f(0),\dots,\bar u_f(j);z)^\transpose]\,,          &  & 0 \leq j \leq i\,.
    \end{align}

    This AMP iteration can be related to $\btheta(t_i)$ and $\bff(t_i)$ by taking $F_i$ and $G_i$ as follows.
    \begin{align}
        G_i(\ba_1,\dots,\ba_i;\btheta(0)) & = w_{t_i}(\btheta(t_i)) \,      \\
        F_i(\bb_0,\dots,\bb_i;\bz)        & = \ell_{t_i}(\bff(t_i);\bz) \,.
    \end{align}
    We can show by induction that the right-hand sides of the above equations are Lipschitz functions of $\ba_i,\btheta(0)$ and $\bb_i,\bz$, respectively, as below.
    \begin{align}
        \btheta(t_i) & = \btheta(t_{i-1}) + \eta p_{t_{i-1}}\ab(\frac{1}{\delta} \bX^\transpose \ell_{t_{i-1}}(\bff(t_{i-1});\bz),\btheta(t_{i-1}))                     \notag                        \\
                     & = \btheta(t_{i-1}) + \eta p_{t_{i-1}}\ab(\frac{1}{\delta} \bX^\transpose F_{i-1}(\bb_0,\dots,\bb_{i-1};\bz),\btheta(t_{i-1}))                     \notag                       \\
                     & = \btheta(t_{i-1}) + \eta p_{t_{i-1}}\ab(\frac{\ba_i}{\delta} + \sum_{j=0}^{i-1} G_j(\ba_1,\dots,\ba_j;\btheta(0)) \xi_{i-1,j},\btheta(t_{i-1}))\,,                            \\
        \bff(t_i)    & = \bX w_{t_i}(\btheta(t_i)) \notag                                                                                                                                             \\
                     & = \bX G_i(\ba_1,\dots,\ba_i;\btheta(0))                                                                                                                                 \notag \\
                     & = \bb_i + \sum_{j=0}^{i-1} F_j(\bb_0,\dots,\bb_j;\bz)\zeta_{i,j} \,.
    \end{align}

    By \citet[Theorem 2.21]{wang2024universality}, for all second-order pseudo-Lipschitz functions $\psi,\tilde\psi\colon\reals^{k(K+1)}\to\reals$, we have, almost surely,
    \begin{align}
        \lim_{n,d\to\infty} \frac{1}{d}\sum_{j=1}^d \psi((\ba_1)_j,\dots,(\ba_K)_j;\theta_j(0)) & = \E[\psi(\bar u_g(1),\dots,\bar u_g(K);\theta(0))]\,, \\
        \lim_{n,d\to\infty} \frac{1}{n}\sum_{j=1}^n \tilde\psi((\bb_0)_j,\dots,(\bb_K)_j;z_j)   & = \E[\tilde\psi(\bar u_f(0),\dots,\bar u_f(K);z)]\,.
    \end{align}
    Since $\btheta(i)$ is a Lipschitz function of $\ba_1,\dots,\ba_i,\btheta(0)$, we can take some Lipschitz function $h_\theta$ such that $\btheta(i) = h_\theta(\ba_1,\dots,\ba_i;\btheta(0))$. Then, we define $\bar\theta(i)\coloneqq h_\theta(\bar u_g(1),\dots,\bar u_g(i);\theta(0))$. Similarly, we define $\bar f(i) = h_f(\bar u_f(0),\dots,\bar u_f(i);z)$ where $h_f$ is such that $\bff(i) = h_f(\bb_0,\dots,\bb_i;\bz)$.
    Since a composition of Lipschitz functions is again Lipschitz, we have
    \begin{align}
        \lim_{n,d\to\infty} \frac{1}{d}\sum_{j=1}^d \psi(\theta_j(1),\dots,\theta_j(K);\theta(0)_j) & = \E[\psi(\bar \theta(1),\dots,\bar \theta(K);\theta(0))]\,, \\
        \lim_{n,d\to\infty} \frac{1}{n}\sum_{j=1}^n \tilde\psi(f_j(0),\dots,f_j(K);z_j)             & = \E[\tilde\psi(\bar f(0),\dots,\bar f(K);z)]\,.
    \end{align}
    It then follows that
    \begin{align}
        \lim_{n,d\to\infty} W_2\ab(\frac{1}{d}\sum_{j=1}^d \dirac_{\theta(0)_j,\dots,\theta_j(K)},\sfP(\bar \theta(0),\dots,\bar \theta(K))) & = 0 \,, \\
        \lim_{n,d\to\infty} W_2\ab(\frac{1}{n}\sum_{j=1}^n \dirac_{f_j(0),\dots,f_j(K),z_j},\sfP(\bar f(0),\dots,\bar f(K),z))               & = 0 \,.
    \end{align}

    \item
    \paragraph{Mapping the State Evolution to DMFT.}
    We map the state evolution variables $\bar \theta(i),\bar f(i)$ to the DMFT variables $\theta(t_i),f(t_i)$ for $\mathfrak{S}^\eta$.
    We show by induction on $r$ that the unique solution of $\mathfrak{S}^\eta$ is given as follows.
    \begin{align}
        (\theta(t_1),\dots,\theta(t_r)) & \eqdist (\bar \theta(1),\dots,\bar \theta(r)) \,,                                 \label{eq:dmft_se_theta}                                                       \\
        (f(t_0),\dots,f(t_r))           & \eqdist (\bar f(0),\dots,\bar f(r)) \,,                                 \label{eq:dmft_se_f}                                                                     \\
        R_w(t_i,t_j)                    & = \zeta_{i,j}/\eta \,,                                                                                     &  & 0 \leq j < i \leq r \,, \label{eq:dmft_se_Rw}    \\
        R_\ell(t_i,t_j)                 & = \xi_{i,j}/\eta\,,                                                                                        &  & 0 \leq j < i \leq r \,,    \label{eq:dmft_se_Rl} \\
        \Gamma(t_i)                     & = \xi_{i,i}\,,                                                                                             &  & 0 \leq i \leq r \,. \label{eq:dmft_se_Gamma}
    \end{align}

    For $r=0$, we have $\theta(0) \eqdist \bar \theta(0)$, and
    \begin{align}
        f(t_0)      & = u_f(t_0) \eqdist \bar u_f(0) = \bar f(0) \eqdist \normal(0,\E[\theta(0)\theta(0)^\transpose])\,,                                            \\
        \Gamma(t_0) & = \E[\nabla_f \ell_{t_0}(f(t_0);z)]  = \E[\nabla_f \ell_{t_0}(\bar f(0);z)] = \E\ab[\diffp{}{\bar u_f(0)} F_0(\bar u_f(0);z)] = \xi_{0,0} \,.
    \end{align}
    Next, we assume that the hypothesis holds for $r$, and show that equations \eqref{eq:dmft_se_theta}--\eqref{eq:dmft_se_Gamma} hold for $r+1$.

    First, we check \eqref{eq:dmft_se_theta} and \eqref{eq:dmft_se_Rw}.
    For $0 \leq j \leq i \leq r$, we have
    \begin{align}
        \E[\bar u_g(i+1)\bar u_g(j+1)^\transpose] & = \delta \E[F_i(\bar u_f(0),\dots,\bar u_f(i);z)F_j(\bar u_f(1),\dots,\bar u_f(j);z)^\transpose] \notag             \\
                                                  & = \delta \E[\ell_{t_i}(\bar f(i);z)\ell_{t_j}(\bar f(j);z)]                                              \notag     \\
                                                  & = \delta \E[\ell_{t_i}(f(t_i);z)\ell_{t_j}(f(t_j);z)]                                                        \notag \\
                                                  & = \delta C_\ell(t_i,t_j) \,,
    \end{align}
    and therefore $(u_g(t_0),\dots,u_g(t_r))\eqdist(\bar u_g(1),\dots,\bar u_g(r+1))$ holds.

    By $\mathfrak{S}^\eta$,
    \begin{align}
        \theta(t_{r+1}) & = \theta(r) + \eta p_{t_r}(g(t_r),\theta(t_r))\,,                                                                                                          \\
        g(t_r)          & = \frac{u_g(t_r)}{\delta} + \Gamma(t_r) w_{t_r}(\theta(t_r)) + \eta \sum_{j=0}^{r-1} R_\ell(t_r,t_j) w_{t_j}(\theta(t_j))                           \notag \\
                        & = \frac{u_g(t_r)}{\delta} + \xi_{r,r} w_{t_r}(\theta(t_r)) + \sum_{j=0}^{r-1} \xi_{r,j} w_{t_j}(\theta(t_j)) \notag                                        \\
                        & = \frac{u_g(t_r)}{\delta} + \sum_{j=0}^{r} \xi_{r,j} w_{t_j}(\theta(j)) \,.
    \end{align}
    By the state evolution,
    \begin{align}
        \bar \theta(r+1) & = \bar \theta(r) + \eta p_{t_r}(\bar g(r),\bar \theta(r))\,,                                                       \\
        \bar g(r)        & \coloneqq \frac{\bar u_g(r+1)}{\delta} + \sum_{j=0}^r G_j(\bar u_g(1),\dots,\bar u_g(j);\theta(0))\xi_{r,j} \notag \\
                         & = \frac{\bar u_g(r+1)}{\delta} + \sum_{j=0}^r \xi_{r,j} w_{t_j}(\bar \theta(j)) \,.
    \end{align}
    Comparing these two equations, we get $\theta(t_{r+1}) \eqdist \bar \theta(r+1)$.
    A similar argument shows that, for $0\leq j \leq r$, $\diffp{}{u_g(t_j)} w_{t_{r+1}}(\theta(t_{r+1})) \eqdist \diffp{}{\bar u_g(j+1)} w_{t_{r+1}}(\bar \theta(r+1))$.
    Thus, it follows that $R_w(t_{r+1},t_j)=\zeta_{r+1,j}/\eta$.

    Next, we check \eqref{eq:dmft_se_f}, \eqref{eq:dmft_se_Rl}, and \eqref{eq:dmft_se_Gamma}.
    For $0 \leq j \leq i \leq r+1$, we have
    \begin{align}
        \E[\bar u_f(i)\bar u_f(j)^\transpose] & = \E[G_i(\bar u_g(1),\dots,\bar u_g(i);\theta(0))G_j(\bar u_g(1),\dots,\bar u_g(j);\theta(0))^\transpose] \notag             \\
                                              & = \E[w_{t_i}(\bar \theta(i))w_{t_j}(\bar \theta(j))]                                                             \notag      \\
                                              & = \E[w_{t_i}(\theta(t_i))w_{t_j}(\theta(t_j))]                                                                        \notag \\
                                              & = C_w(t_i,t_j)
    \end{align}
    and therefore $(u_f(t_0),\dots,u_f(t_{r+1}))\eqdist (\bar u_f(0),\dots,\bar u_f(r+1))$ holds.

    By $\mathfrak{S}^\eta$,
    \begin{align}
        f(t_{r+1}) & = u_f(t_{r+1}) +\eta \sum_{j=0}^r R_w(t_{r+1},t_j)\ell_{t_j}(f(t_j);z) \notag \\
                   & = u_f(t_{r+1}) + \sum_{j=0}^r \zeta_{r+1,j}\ell_{t_j}(f(t_j);z)\,,
    \end{align}
    By the state evolution,
    \begin{align}
        \bar f(r+1) & = \bar u_f(r+1) + \sum_{j=0}^r F_j(\bar u_f(0),\dots,\bar u_f(j);z) \zeta_{r+1,j} \notag \\
                    & = \bar u_f(r+1) + \sum_{j=0}^r \zeta_{r+1,j} \ell_{t_j}(\bar f(j);z) \,.
    \end{align}
    Comparing these two equations, we get $f(t_{r+1}) \eqdist \bar f(r+1)$.
    A similar argument shows that, for $0\leq j \leq r+1$, $\diffp{}{u_f(t_j)} \ell_{t_{r+1}}(f(t_{r+1});z) \eqdist \diffp{}{\bar u_f(j)} \ell_{t_{r+1}}(\bar f(r+1);z)$.
    Thus, it follows that $R_\ell(t_{r+1},t_j)=\xi_{r+1,j}/\eta,\Gamma(t_{r+1})=\xi_{r+1,r+1}$.

\end{proof}

\subsection{Proof of Corollary~\ref{cor:dmft_dln}}
\label{app:proof_dmft_dln}

The gradient flow for regression using truncated DLNs can be mapped to our general flow $\mathfrak{F}$ by setting $k=3$ and taking $\btheta(t), w_t, p_t, \ell_t$ as follows.
\begin{align}
    \btheta(t)           & = (\bu(t),\bv(t),\bw^*) \,,                           \\
    w_t(u,v,w^*)         & = \ab(\frac{1}{L}(\eta_M(u^L)-\eta_M(v^L)),w^*,0) \,, \\
    p_t(g,\_,\_;u,v,w^*) & = \ab(p_u(g,u,v),p_v(g,u,v),0) \,,                    \\
    \ell_t(f,f^*,\_;\xi) & = (f-f^*-\xi,0,0)  \,,
\end{align}
where the underscore $\_$ indicates unused entries.

It is easy to check \cref{ass:main,ass:gf} for our specific choices of functions. Applying \cref{thm:dmft_sol,thm:dmft_flow}, we obtain the following DMFT equation.
\begin{align}
    \diff{}{t}\pmat{u(t)                                            \\ v(t) \\ w^*} & = \pmat{p_u(g(t),u(t),v(t)) \\ p_v(g(t),u(t),v(t)) \\ 0} \,, \\
    \pmat{g(t)                                                      \\ \_}        & = \frac{1}{\delta} \pmat{z_g(t)\\ z_{g*}(t)} + \Gamma(t) \pmat{w(t) \\ w^*} + \int_0^t R_\ell(t,s) \pmat{w(s) \\ w^*} \de s \,, && \pmat{z_g \\ z_{g*}} \sim \GP(0,\delta C_\ell)\,,             \\
    \pmat{f(t)                                                      \\ f_*} & = \pmat{z_f(t) \\ z_{f*}(t)} + \int_0^t R_w(t,s) \pmat{f(t)-f_*-\xi\\ 0} \de s \,,  && \pmat{z_f \\ z_{f*}} \sim \GP(0,C_w) \,, \\
    w(t) & \coloneqq \frac{1}{L}(\eta_M(u(t)^L)-\eta_M(v(t)^L)) \,,
\end{align}
where
\begin{gather}
    \begin{gathered}
        R_w(t,s) = \pmat{R_w(t,s) & 0 \\ 0 & 0} \,, \quad R_\ell(t,s) = \pmat{R_f(t,s) & R_{f*}(t,s) \\ 0 & 0} \,, \quad \Gamma(t) = \pmat{1 & -1 \\ 0 & 0}  \,, \\
        C_w(t,s) = \pmat{C_w(t,s) & m(t) \\ m(s) & \rho^2} \,, \; C_\ell(t,s) = \pmat{C_f(t,s) & 0 \\ 0 & 0} \,,
    \end{gathered} \\
    \begin{gathered}
        R_w(t,s) = \E\ab[\diffp{w(t)}{z_g(s)}] \,, \quad R_f(t,s) = \E\ab[\diffp{(f(t)-f^*-\xi)}{z_f(s)}] \,, \quad R_{f*}(t,s) = \E\ab[\diffp{(f(t)-f^*-\xi)}{z_{f*}(s)}] \,, \\
        C_w(t,s) = \E[w(t)w(s)] \,, \quad m(t) = \E[w(t) w^*] \,, \quad C_f(t,s) = \E[(f(t)-f^*-\xi)(f(s)-f^*-\xi)] \,,
    \end{gathered}
\end{gather}
where $\rho^2 \coloneqq \E[w^{*2}]$ and we omitted unused rows and columns.

We have $R_{f*}(t,s) = -R_f(t,s)$. Resetting $f(t) - f^*$, $z_f(t) - z_{f*}(t)$ as $f(t)$, $z_f(t)$, respectively, and unfolding the expressions, and flipping the signs of $R_w$ and $R_f$, we have the following.
\begin{subequations}
    \begin{align}
        \diff{}{t} u(t) & = p_u(g(t),u(t),v(t)) \,,                                                                                                                                   \\
        \diff{}{t} v(t) & = p_v(g(t),u(t),v(t)) \,,                                                                                                                                   \\
        g(t)            & = \frac{z_g(t)}{\delta} + w(t) - w^* - \int_0^t R_f(t,s) (w(s) - w^*) \de s \,, &  & z_g \sim \GP(0,\delta C_f) \,, \quad w^* \sim P_* \,,                  \\
        f(t)            & = z_f(t) - \int_0^t R_w(t,s) (f(s) - \xi) \de s \,,                             &  & z_f \sim \GP(0,C_w) \,, \quad \xi \sim P_\xi \,, \label{eq:f_rigorous} \\
        R_w(t,s)        & = -\E\ab[\diffp{w(t)}{z_g(s)}] \,,                                                                                                                          \\
        R_f(t,s)        & = -\E\ab[\diffp{f(t)}{z_f(s)}] \,,                                                                                                                          \\
        C_w(t,s)        & = \E[(w(t) - w^*) (w(s) - w^*)] \,,                                                                                                                         \\
        C_f(t,s)        & = \E[(f(t)-\xi)(f(s)-\xi)] \,.
    \end{align}
\end{subequations}

This concludes the proof.

\section{DETAILS OF NUMERICAL EXPERIMENTS}
\label{app:numerical}

All experiments were conducted using Python on a single Apple MacBook Pro with an 8-core CPU and 16 GB of RAM.

\subsection{Experiments with Non-Gaussian Data}

\begin{figure*}[t]
    \begin{center}
        \begin{subfigure}{0.49\textwidth}
            \includegraphics[width=\textwidth]{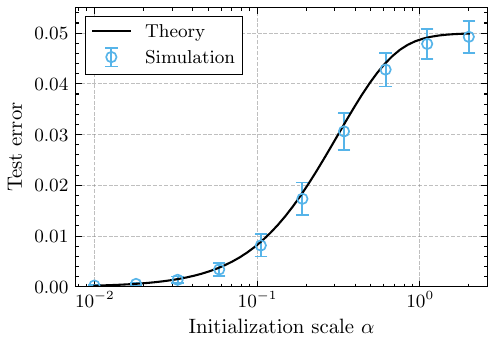}
            \caption{Fixed points.}
        \end{subfigure}
        \begin{subfigure}{0.49\textwidth}
            \includegraphics[width=\textwidth]{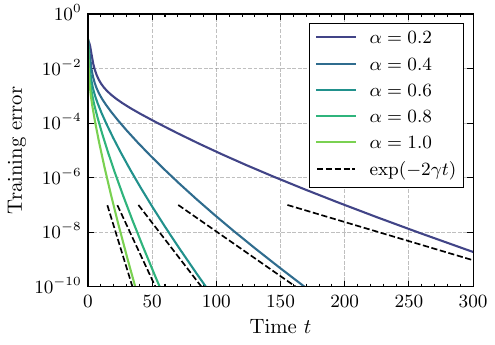}
            \caption{Convergence rates of the training error.}
        \end{subfigure}
        \caption{Fixed points and convergence rates of the loss on binary data with $\delta = 0.5$. Theoretical predictions are shown in black and agree well with the numerical simulation, thus highlighting the universality of our result with respect to the data distribution.}
        \label{fig:experiment_binary}
    \end{center}
\end{figure*}

To illustrate the universality of our theoretical result, we conducted experiments with binary features $x_{\mu i} \sim \Unif(\{\pm1/\sqrt{d}\})$. Other experimental setups are the same as the experiments on Gaussian data described in \cref{sec:numerical}. Fixed points and convergence rates are shown in \cref{fig:experiment_binary}, showing good agreement with the theoretical prediction.

\subsection{Experiments with Real-World Data}

\begin{figure}[t]
    \begin{center}
        \includegraphics[width=0.5\textwidth]{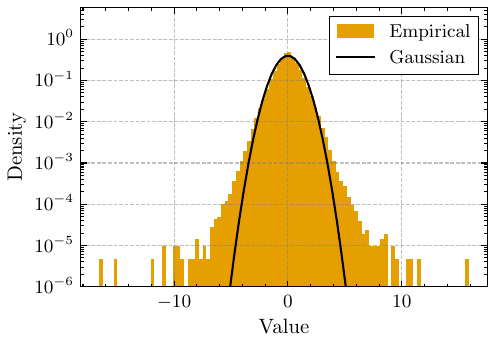}
        \caption{The empirical distribution of the whitened gene expression data. It has heavier tails than the standard Gaussian distribution shown in black.}
        \label{fig:empirical_gene}
    \end{center}
\end{figure}

We use a gene expression dataset \citep{fiorini2016gene} (a subset of \citet{ellrott2013tcga}) consisting of 801 instances and 20531 features. We whiten the entire dataset using principal component analysis and randomly sample subsets of samples and features to obtain $n = 100$ times  $d = 200$ data matrix. Columns are normalized to have mean zero and variance $1/d$. The empirical distribution of whitened features is shown in \cref{fig:empirical_gene}.

The results on the convergence rate quantitatively deviate from the theoretical prediction as illustrated in \cref{fig:convergencerate_real}. This is most likely due to a finite sample size effect. We have whitened the dataset using $n=801$ samples, instead of ideal $n=\infty$ samples. Since the matrix is whitened using finite samples, the spectrum of the empirical covariance matrix for $n=100$ samples concentrates closer to one, and the minimum eigenvalue of the empirical covariance matrix becomes larger. Since the convergence rate is closely related to the minimum eigenvalue of the sample covariance, as suggested by the $\alpha \to \infty$ analysis of the convergence rate in \cref{app:convergence_rate}, the finite sample size effect slightly accelerates the dynamics, resulting in the deviation from the theoretical prediction in \cref{fig:convergencerate_real}.

\end{document}